\documentclass[nohyperref]{article}

\usepackage{microtype}
\usepackage{graphicx}
\usepackage{subfigure}
\usepackage{booktabs} 

\usepackage{hyperref}

\usepackage[accepted]{icml2022}

\usepackage{amsmath}
\usepackage{amssymb}
\usepackage{mathtools}
\usepackage{amsthm}
\usepackage{bm}

\usepackage{algorithm} 
\usepackage{algorithmic}
\usepackage{tabularx}
\usepackage{bbm}
\usepackage{bbding}
\usepackage{multirow}

\usepackage[capitalize,noabbrev]{cleveref}

\theoremstyle{plain}
\newtheorem{theorem}{Theorem}
\newtheorem{proposition}{Proposition}
\newtheorem{lemma}{Lemma}

\theoremstyle{definition}
\newtheorem{definition}{Definition}
\newtheorem{assumption}{Assumption}
\newtheorem{remark}{Remark}

\DeclareMathOperator{\sign}{\mathrm{sign}}
\DeclareMathOperator*{\argmin}{\mathrm{arg\,min}}
\DeclareMathOperator*{\argmax}{\mathrm{arg\,max}}

\newcommand{\bE}{\mathbb{E}}
\newcommand{\bR}{\mathbb{R}}
\newcommand{\bP}{\mathbb{P}}
\newcommand{\bQ}{\mathbb{Q}}
\newcommand{\bF}{\mathbb{F}}

\newcommand{\bW}{\mathbb{W}}
\newcommand{\bJ}{\mathbb{J}}

\newcommand{\cF}{\mathcal{F}}
\newcommand{\cX}{\mathcal{X}}

\newcommand{\cH}{\mathcal{H}}
\newcommand{\cN}{\mathcal{N}}
\newcommand{\cD}{\mathcal{D}}
\newcommand{\cO}{\mathcal{O}}
\newcommand{\cJ}{\mathcal{J}}

\newcommand{\MMD}{\mathrm{MMD}}

\newcommand{\G}{\mathrm{G}}
\newcommand{\D}{\mathrm{D}}
\newcommand{\RoD}{\mathrm{RoD}}
\newcommand{\rS}{\mathrm{S}}
\newcommand{\rL}{\mathrm{L}}
\newcommand{\ME}{\mathrm{ME}}
\newcommand{\SCF}{\mathrm{SCF}}

\newcommand{\R}{\mathrm{R}}
\newcommand{\tr}{\mathrm{tr}}
\newcommand{\te}{\mathrm{te}}
\newcommand{\epsball}{\mathcal{B}_\epsilon}

\newcommand{\bmx}{x}
\newcommand{\bmy}{y}
\newcommand{\bigromanone}{\mathrm{\uppercase\expandafter{\romannumeral1}}}
\newcommand{\two}{\mathrm{\uppercase\expandafter{\romannumeral2}}}

\allowdisplaybreaks[4]

\usepackage[textsize=tiny]{todonotes}

\icmltitlerunning{Adversarial Attack and Defense for Non-Parametric Two-Sample Tests}

\begin{document}

\twocolumn[
\icmltitle{Adversarial Attack and Defense for Non-Parametric Two-Sample Tests}



\icmlsetsymbol{equal}{*}

\begin{icmlauthorlist}
\icmlauthor{Xilie Xu}{equal,nus}
\icmlauthor{Jingfeng Zhang}{equal,riken}
\icmlauthor{Feng Liu}{unimelb}
\icmlauthor{Masashi Sugiyama}{riken,UTokyo}
\icmlauthor{Mohan Kankanhalli}{nus}
\end{icmlauthorlist}

\icmlaffiliation{nus}{School of Computing, National University of Singapore}
\icmlaffiliation{riken}{RIKEN Center for Advanced Intelligence Project (AIP)}
\icmlaffiliation{UTokyo}{Graduate School of Frontier Sciences, The University of Tokyo}
\icmlaffiliation{unimelb}{School of Mathematics and Statistics, The University of Melbourne}

\icmlcorrespondingauthor{Jingfeng Zhang}{jingfeng.zhang@riken.jp}

\icmlkeywords{Machine Learning, ICML}

\vskip 0.3in
]



\printAffiliationsAndNotice{\icmlEqualContribution} 

\begin{abstract}
Non-parametric two-sample tests (TSTs) that judge whether two sets of samples are drawn from the same distribution, have been widely used in the analysis of critical data. People tend to employ TSTs as trusted basic tools and rarely have any doubt about their reliability. This paper systematically uncovers the failure mode of non-parametric TSTs through adversarial attacks and then proposes corresponding defense strategies. 
First, we theoretically show that an adversary can upper-bound the distributional shift which guarantees
the \textit{attack's invisibility}.
Furthermore, we theoretically find that the adversary can also degrade the lower bound of a TST's \textit{test power}, which enables us to iteratively minimize the \textit{test criterion} in order to search for adversarial pairs. 
To enable TST-agnostic attacks, we propose an \textit{ensemble attack} (EA) framework that jointly minimizes the different types of test criteria. 
Second, 
to robustify TSTs, we propose a \textit{max-min optimization} that iteratively generates adversarial pairs to train the deep kernels. 
Extensive experiments on both simulated and real-world datasets validate the adversarial vulnerabilities of non-parametric TSTs and the effectiveness of our proposed defense.
Source code is available at \href{https://github.com/GodXuxilie/Robust-TST.git}{https://github.com/GodXuxilie/Robust-TST.git}.
\end{abstract}

\section{Introduction}
Non-parametric two-sample tests (TSTs) that judge whether two sets of samples drawn from the same distribution have been widely used to analyze critical data in physics~\cite{baldi2014searching}, neurophysiology~\cite{rasch2008predicting}, biology~\cite{borgwardt2006integrating}, etc.
Compared with traditional methods (such as the \textit{t}-test), non-parametric TSTs can relax the strong parametric assumption about the distributions being studied and are effective in complex domains~\cite{gretton2009fast, gretton2012kernel, chwialkowski2015fast, jitkrittum2016interpretable,sutherland2016generative,lopez2016revisiting,cheng2019classification,liu2020learning, liu2021meta}.
Notably, the use of deep kernels~\cite{liu2020learning} flexibly empowers the non-parametric TSTs to learn even more complex distributions. 

However, the adversarial robustness of non-parametric TSTs is rarely studied, despite its extensive studies for deep neural networks (DNNs). 
Studies of DNNs' adversarial robustness~\cite{Madry_adversarial_training}
have enabled significant advances in defending against adversarial attacks~\cite{szegedy}, which can help enhance the security in various domains such as computer vision~\cite{xie2017adversarial,mahmood2021robustness}, natural language processing~\cite{Zhu2020FreeLB:, yoo2021towards}, recommendation system~\cite{peng2020robust}, etc. 
We therefore undertake this pioneer study on adversarial robustness of non-parametric TSTs, which uncovers the failure mode of non-parametric TSTs through adversarial attacks and facilitate an effective strategy for making TSTs reliable in critical applications~\cite{baldi2014searching,rasch2008predicting,borgwardt2006integrating}.

First, we theoretically show the adversary could upper-bound the distributional shift and degrade the lower bound of a TST's test power (details in Section~\ref{sec:theoretical_analysis}). 
Given a benign pair $(S_{\bP}, S_{\bQ})$, in which $S_{\bP} =\{{x}_i\}_{i=1}^{m} \sim \bP^m$
and $S_{\bQ} = \{{y}_j \}_{j=1}^n \sim \bQ^n$ 
, an $\ell_{\infty}$-bounded adversary could generate the adversarial pair $(S_\bP, \tilde{S}_\bQ)$. We will show in Proposition~\ref{proposition:mmd_shift} that the maximum mean discrepancy (MMD)~\cite{gretton2012kernel} between the benign and adversarial pairs is upper-bounded, which guarantees imperceptible adversarial perturbations~\cite{szegedy}. 
Furthermore, we will show in Theorem~\ref{theory:tp_attack} that 
the adversary can degrade the lower bound of a TST's test power, which implies that a TST could wrongly determine $\bP = \bQ$ with a larger probability under adversarial attacks when $\bP \neq \bQ$ holds.

\begin{figure}[t]
    \centering
    \includegraphics[width=0.49\textwidth]{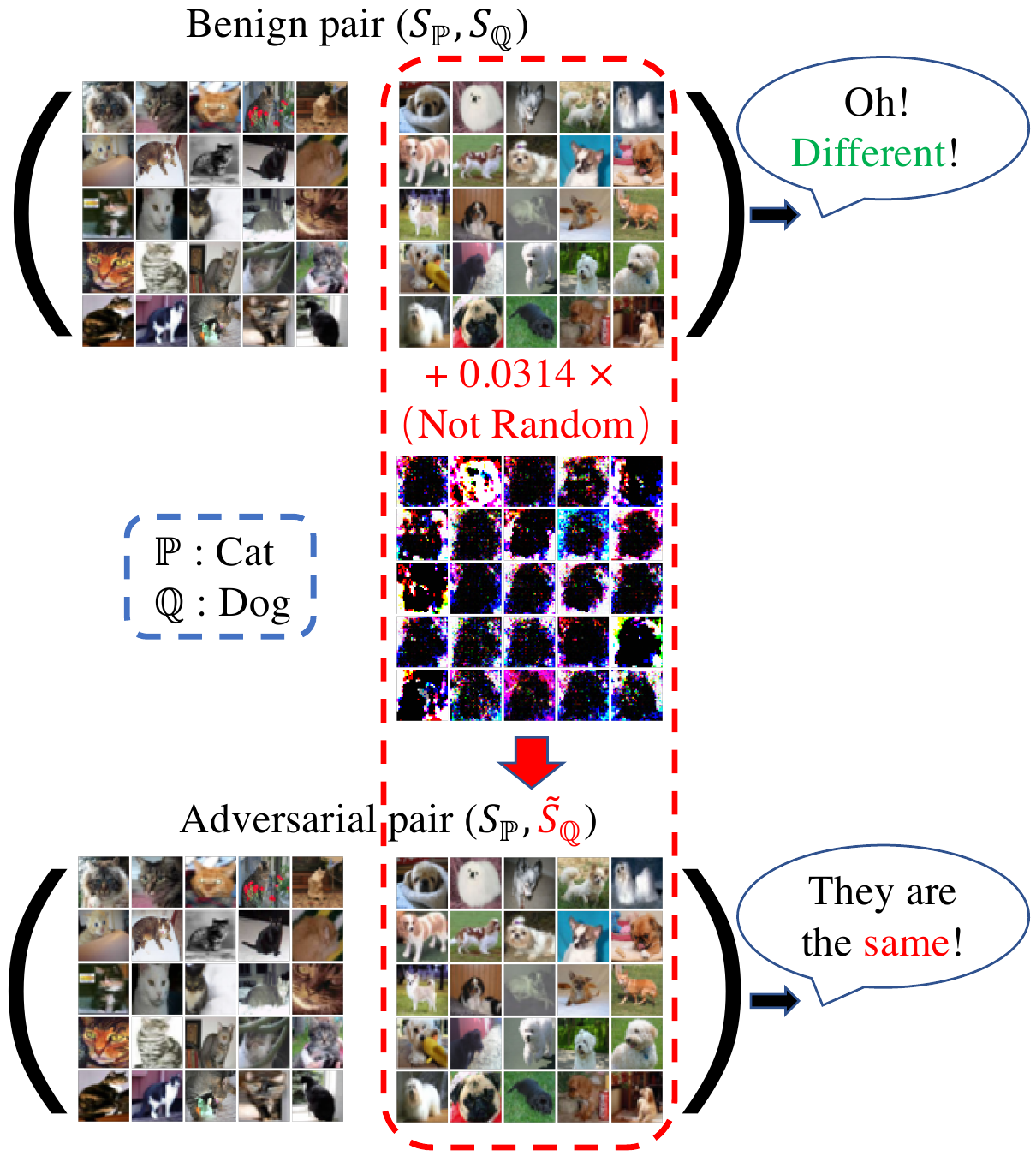}
    \vspace{-8mm}
    \caption{An example of adversarial pair $(S_{\mathbb{P}}, \tilde{S}_{\mathbb{Q}})$ generated by embedding an adversarial perturbation in the benign set $S_{\bQ}$ of the benign pair $(S_{\mathbb{P}}, S_{\mathbb{Q}})$. Experimental details are in Section~\ref{sec:attack_power}. } 
    \label{fig:corner_case_example}
    \vspace{-4mm}
\end{figure}

Then, we realize effective adversarial attacks against non-parametric TSTs (details in Section~\ref{sec:realization}). 
We formulate an attack as a constraint optimization problem that minimizes a TST's test criterion~\cite{liu2020learning} within the $\ell_{\infty}$-bound of size $\epsilon$ on $S_\bQ$. We utilize projected gradient descent (PGD)~\cite{Madry_adversarial_training} to efficiently search the adversarial set $\tilde{S}_\bQ$ and incorporate automatic schedule of the step size~\cite{croce2020reliable} to improve the optimization convergence.
Moreover, we extend the attack beyond a specific TST to a generic TST-agnostic attack, namely, ensemble attack (EA). EA jointly minimizes a weighted sum of different test criteria, which can simultaneously fool various TSTs. 
For example, Figure~\ref{fig:corner_case_example} shows non-parametric TSTs can correctly differentiate the benign pair of ``cats'' and ``dogs'' (top) coming from the different distributions, but wrongly judge adversarial pairs (bottom) as belonging to the same distribution. 

Second, to robustify the non-parametric TSTs, we study the corresponding defense approaches (details in Section~\ref{sec:defense}). 
A straightforward defense seems to use an ensemble of TSTs. We find an ensemble of TSTs is sometimes effective against a specific attack targeting a certain type of TSTs but almost always fails under EA (see experiments in Section~\ref{sec:attack_power}).
Therefore, to effectively defend against adversarial attacks, we propose to adversarially learn the robust kernels. The defense is formulated as a \textit{max-min} optimization that is similar in flavor to the adversarial training's \textit{min-max} formulation~\cite{Madry_adversarial_training}.
For its realization, we iteratively generate adversarial pairs by minimizing the test criterion in the \textit{inner minimization} and update kernel parameters by maximizing the test criterion on the adversarial pairs in the \textit{outer maximization}. 
We realize our defense using deep kernels that have achieved the state-of-the-art (SOTA) performance in non-parametric TSTs~\cite{liu2020learning}.

Lastly, we empirically justify the proposed attacks and defenses (in Section~\ref{sec:exp}).
We evaluate the test power of many existing non-parametric TSTs (non-robust) and the robust-kernel TST (robust) under the EA on simulated and real-world datasets, including complex synthetic distributions, high-energy physics data, and challenging images. 
Comprehensive experimental results validate that the existing non-parametric TSTs lack adversarial robustness; we can significantly improve the adversarial robustness of non-parametric TSTs through adversarially learning the deep kernels. 

\section{Non-Parametric Two-Sample Tests}
In this section, we provide the preliminaries of non-parametric TSTs and provide discussions with the related studies in Appendix~\ref{sec:related_work}.

\subsection{Problem Formulation}
Let $\cX \subset \bR^{d}$ and $\bP$, $\bQ$ be Borel probability measures on $\cX$. A non-parametric TST $\cJ(S_\bP, S_\bQ): \cX^m \times \cX^n \mapsto \{0,1\}$ is used to distinguish between the null hypothesis $\mathcal{H}_0: \bP = \bQ$ and the alternative hypothesis $\cH_1: \bP \neq \bQ$, where $S_\bP$ and $S_\bQ$ are independent identically distributed (IID) samples of size $m$ and $n$ drawn from $\bP$ and $\bQ$, respectively.
A non-parametric TST constructs a mean embedding based on a kernel parameterized with $\theta$ for each distribution, and utilizes the differences in these embeddings as the \emph{test statistic} for the hypothesis test. The judgement is made by comparing the test statistic $\mathcal{D}(S_{\bP}, S_{\bQ})$ with a particular threshold $r$: if the threshold is exceeded, then the test rejects $\cH_0$. The \emph{test power} (TP) of a non-parametric TST $\cJ$ is measured by the probability of correctly rejecting $\cH_0$ when the alternative hypothesis is true,
i.e., $\mathrm{TP}(\cJ) = \bE_{S_\bP \sim \bP^m, S_\bQ \sim \bQ^n}[ \mathbbm{1} (\cJ(S_{\bP}, S_{\bQ}) = 1)]$ for a paritular $\bP \neq \bQ$. A non-parametric TST optimizes its learnable parameters $\theta$ via maximizing its \emph{test criterion}, thus approximately maximizing its test power.

\subsection{Test Statistics}

Here, we introduce a typical test statistic, maximum mean discrepancy (MMD)~\cite{gretton2012kernel}, and leave other test statistics in Appendix~\ref{appendix:tst_intro}, such as tests based on Gaussian kernel mean embeddings at specific positions~\cite{chwialkowski2015fast, jitkrittum2016interpretable} and classifier two-sample tests (C2ST)~\cite{lopez2016revisiting,cheng2019classification}. 

\begin{definition}[\citet{gretton2012kernel}]
    Let $k: \cX \times \cX \to \bR$ be a kernel of a reproducing kernel Hilbert space $\cH_k$, with feature maps $k(\cdot,x) \in \cH_k$. Let $X \sim \bP$ and $Y \sim \bQ$, and define the kernel mean embeddings $\mu_{\bP} = \bE[k(\cdot, X)]$ and $\mu_{\bQ} = \bE[k(\cdot, Y)]$. Under mild integrability conditions,
    \begin{align}
    \label{eq:mmd}
        \MMD(\bP,\bQ; \cH_k) & = \sup_{f \in \cH, \|f\|_{\cH_k} \leq 1} |\bE[f(X)] - \bE[f(Y)]| \nonumber \\
        & = \|\mu_{\bP} - \mu_{\bQ}\|_{\cH_k}.
    \end{align}
\end{definition}
For characteristic kernels, $\MMD(\bP,\bQ; \cH_k) = 0$ if and only if $\bP = \bQ$. 
Assuming $n=m$, we can estimate $\MMD$ (Eq.~\eqref{eq:mmd}) using the $U$-statistic estimator, which is unbiased for ${\MMD}^2$
and has nearly minimal variance among all unbiased estimators~\cite{gretton2012kernel}:
\begin{align}
\label{eq:mmd_u}
    &{\widehat{\MMD}}^{2}(S_{\bP}, S_{\bQ};k) = \frac{1}{n(n-1)}\sum_{i \neq j}H_{ij},\\
    H_{ij}  & = k(\bmx_i,\bmx_j) + k(\bmy_i, \bmy_j) - k(\bmx_i, \bmy_j) - k(\bmx_j, \bmy_i), \nonumber
\end{align}
where $\bmx_i,\bmx_j \in S_{\bP}$ and $\bmy_i, \bmy_j \in S_{\bQ}$.

In this paper, we investigate six types of non-parametric TSTs as follows since \citet{liu2020learning, liu2021meta} have shown they are powerful on complex data.

\begin{itemize}
\setlength{\itemsep}{0pt}
\setlength{\parsep}{0pt}
\setlength{\parskip}{0pt}
    \item $\cD^{(\G)}(\cdot,\cdot;k^{(\G)}) = {\widehat{\MMD}}^{2}(\cdot, \cdot;k^{(\G)})$ for tests based on MMD with Gaussian kernels (MMD-G)~\cite{sutherland2016generative} with the learnable lengthscale $\sigma_{\phi}$, in which $k^{(\G)}(\bmx,\bmy) = \exp(-\frac{1}{2\sigma_{\phi}}\|\bmx-\bmy\|^2)$.
    \item $\cD^{(\D)}(\cdot,\cdot;k^{(\D)}) = {\widehat{\MMD}}^{2}(\cdot, \cdot;k^{(\D)})$ for tests based on MMD with deep kernels (MMD-D)~\cite{liu2020learning}. Note that $k^{(\D)}(\bmx,\bmy) = [(1-\gamma)\exp(-\frac{1}{2\sigma_{\phi}} \| \phi(\bmx)-\phi(\bmy) \|^2) + \gamma] \exp(-\frac{1}{2\sigma_{q}}\|\bmx-\bmy\|^2)$ where $\gamma, \sigma_{\phi}, \sigma_q$ are the learnable parameters and $\phi(\cdot)$ is a parameterized deep network to extract the features.
    \item $\cD^{(\rS)}(\cdot,\cdot)$ (Eq.~\eqref{eq:ts_c2sts}) for C2ST based on Sign (C2ST-S)~\cite{lopez2016revisiting}. A classifier $f: \cX \rightarrow \bR$ that outputs the classification probabilities is utilized by C2ST.
    \citet{liu2020learning} pointed out that the test statistic of C2ST-S is equivalent to MMD with kernel $k^{(\rS)}$, i.e., $\cD^{(\rS)}(\cdot,\cdot) = {\widehat{\MMD}}^{2}(\cdot, \cdot;k^{(\rS)})$ where $k^{(\rS)}(\bmx,\bmy) = \frac{1}{4}\mathbbm{1}(f(\bmx)>0)\mathbbm{1}(f(\bmy)>0)$.
    \item $\cD^{(\rL)}(\cdot,\cdot)$ (Eq.~\eqref{eq:ts_s2stl}) for C2ST-L~\cite{cheng2019classification} that utilizes the discriminator's measure of confidence. Its test statistic is also equivalent to MMD with kernel $k^{(\rL)}$~\cite{liu2020learning}, i.e., $\cD^{(\rL)}(\cdot,\cdot) = {\widehat{\MMD}}^{2}(\cdot, \cdot;k^{(\rL)})$ where $k^{(\rL)}(\bmx,\bmy) = f(\bmx)f(\bmy)$.
    \item $D^{(\ME)}(\cdot,\cdot)$ (Eq.~\eqref{eq:ts_ME}) for 
    tests based on differences in Gaussian kernel mean embeddings at specific locations~\cite{chwialkowski2015fast,jitkrittum2016interpretable}, namely Mean Embedding (ME).
    \item $D^{(\SCF)}(\cdot,\cdot)$ (Eq.~\eqref{eq:ts_SCF}) for 
    tests based on Gaussian kernel mean embeddings at a set of optimized frequency~\cite{chwialkowski2015fast,jitkrittum2016interpretable}, namely Smooth Characteristic Functions (SCF).
\end{itemize}

\subsection{Test Criterion}

In this subsection, we introduce the test criteria
for non-parametric TSTs based on MMD~\cite{sutherland2016generative,liu2020learning, lopez2016revisiting, cheng2019classification}.

\begin{theorem}[Asymptotics of MMD under $\cH_1$~\cite{serfling2009approximation}]
\label{theory:asymptotics}
    Under the alternative, $\cH_1 : \bP \neq \bQ$, the standard central limit theorem holds:
    \begin{align}
        \sqrt{n}({\widehat{\MMD}}^{2} - {\MMD}^2) \rightarrow \cN(0, \sigma_{\cH_1}^2), \nonumber\\
        \sigma_{\cH_1}^2 = 4(\bE[H_{12}H_{13}] - \bE[H_{12}]^2), \nonumber
    \end{align}
    where $H_{12}, H_{13}$ refer to $H_{ij}$ in Eq.~\eqref{eq:mmd_u}.
\end{theorem}

Guided by the asymptotics of MMD (Theorem~\ref{theory:asymptotics}), the test power is estimated as follows:
\begin{align}
    \label{eq:tp_kernel}
    \Pr(n{\widehat{\MMD}}^{2} > r) \rightarrow \Phi(\frac{\sqrt{n}{\MMD}^{2} }{\sigma_{\cH_1}} - \frac{r}{\sqrt{n}\sigma_{\cH_1}}),
\end{align}
where $\Phi$ is the cumulative distribution function (CDF) of standard normal distribution and $r$ is the rejection threshold approximately found via permutation testing~\cite{dwass1957modified,fernandez2008test}. This general method is usually considered best to estimate the null hypothesis: under $\cH_0$, samples from $\bP$ and $\bQ$ are interchangeable, and repeatedly re-computing the test statistic with samples randomly shuffled between $S_\bP$ and $S_\bQ$ estimates its null distribution.

For reasonably large $n$, the test power is dominated by the first term of Eq.~\eqref{eq:tp_kernel}, and thus the TST yields the most powerful test by approximately maximizing the test criterion~\cite{liu2020learning}
\begin{align}
    \cF(\bP,\bQ;k) = {\MMD}^{2}(\bP,\bQ;k)/\sigma_{\cH_1}(\bP,\bQ;k).
\end{align}
Further, $\cF(\bP,\bQ;k)$ can be empirically estimated with  
\begin{align}
\label{eq:test_criterion}
    \hat{\cF}(S_{\bP}, S_{\bQ};k) = \frac{{\widehat{\MMD}}^{2}(S_{\bP}, S_{\bQ};k)}{\hat{\sigma}_{\cH_1, \lambda}(S_{\bP}, S_{\bQ};k)},
\end{align}
where $\hat{\sigma}_{\cH_1, \lambda}^2$ is a regularized estimator of $\sigma_{\cH_1}^2$:
\begin{align}
    \hat{\sigma}_{\cH_1, \lambda}^2 = \frac{4}{n^3}\sum_{i=1}^{n} \bigg(\sum_{j=1}^{n}H_{ij} \bigg)^2 - \frac{4}{n^4}\bigg(\sum_{i=1}^{n}\sum_{j=1}^{n}H_{ij} \bigg)^2 + \lambda, \nonumber
\end{align}
where $\lambda$ is a positive constant. The test criterion of the MMD test (e.g., MMD-G, MMD-D, C2ST-S and C2ST-L) is calculated based on its corresponding kernel. We let $\sigma_{\theta}^2$ denote $\sigma^2_{\cH_1}(\bP,\bQ;k_{\theta})$, and analogously  $\hat{\sigma}_{\theta}^2$ denote $\hat{\sigma}^2_{\cH_1, \lambda}(\bP,\bQ;k_{\theta})$, for simplicity.

In addition, ~\citet{chwialkowski2015fast} and~\citet{jitkrittum2016interpretable} analyzed that the test power of ME tests, and SCF tests can be approximately maximized by maximizing the corresponding test criterion as well, i.e.,  $\hat{\cF}^{(\ME)}(S_{\bP},S_{\bQ})$ and $\hat{\cF}^{(\SCF)}(S_{\bP},S_{\bQ})$ (details in Appendix~\ref{appendix:tst_intro}).

To avoid notation clutter, we simply let $\theta$ represent all of the learnable parameters
in a non-parametric TST. 
The optimized parameters of a TST are obtained as follows.
\begin{align}
    \hat{\theta} \approx \argmax_{\theta} \hat{\cF}(S_{\bP}^{\tr}, S_{\bQ}^{\tr};k_{\theta}),
\end{align}
where $(S_{\bP}^{\tr}, S_{\bQ}^{\tr})$ is the training pair.
Then, we conduct a hypothesis test based on $\cD(S_{\bP}^{\te}, S_{\bQ}^{\te};k_{\hat{\theta}})$, where $(S_{\bP}^{\te}, S_{\bQ}^{\te})$ is the test pair.

\section{Adversarial Attacks Against Non-Parametric TSTs}
In this section, we first show the possible existence of adversarial attacks against a non-parametric TST. Then, we propose a method to generate adversarial test pairs that can fool a TST. To enable TST-agnostic attacks, we propose a unified attack framework, i.e., ensemble attack.

\subsection{Theoretical Analysis}
\label{sec:theoretical_analysis}
This section theoretically shows that there could exist adversarial attacks that can invisibly undermine a TST. 
We first lay out the needed assumptions on kernel functions. 

\begin{assumption}
\label{assump:param}
    The possible kernel parameterized with $\theta \in \bR^{\kappa}$ lies in Banach space.
    The set of possible kernel parameters $\Theta$ is bounded by $\R_{\Theta}$, i.e., $\Theta \subseteq \{ \theta \mid \|\theta\| \leq \R_{\Theta} \}$.
    We let $\bar{\Theta}_s = \{ \theta \in \Theta \mid \sigma_{\theta}^2 \geq s^2 > 0 \}$ in which $s$ is a positive constant.
\end{assumption}

\begin{assumption}
\label{assump:bound}
    The kernel function $k_{\theta}$ is uniformly bounded, i.e., $\sup_{\theta \in \Theta} \sup_{x \in \cX} k_{\theta}(\bmx,\bmx) \leq \nu$. We treat $\nu$ as a constant.
\end{assumption}

\begin{assumption}
\label{assump:lipschitz}
    The kernel function $k_{\theta}(\bmx,\bmy)$ satisfies the Lipschitz conditions as follows.
    \begin{gather}
        |k_{\theta}(\bmx,\bmy) - k_{\theta'}(\bmx,\bmy)| \le L_1 \|\theta - \theta'\|; \nonumber \\
        |k_{\theta}(\bmx,\bmy) - k_{\theta}(\bmx',\bmy')| \le L_2 (\|\bmx-\bmx'\| + \|\bmy-\bmy'\|) \nonumber, 
    \end{gather}
    where $L_1$ and $L_2$ are positive constants. 
\end{assumption}

We consider a potential risk that causes a malfunction of a non-parametric TST: an adversarial attacker that aims to deteriorate the TST's test power, can craft an adversarial pair $(S_\bP, \tilde{S}_\bQ)$ as the input to the TST during the testing procedure, in which the two sets $\tilde{S}_\bQ$ and $S_\bQ$ are nearly indistinguishable. We provide a detailed description of the attacker against non-parametric TSTs in Appendix~\ref{appendix:attacker}.

We define the $\epsilon$-ball centered at $x \in \cX$ as follows:
\begin{align}
\epsball[x] = \{\tilde{x} \in \cX \mid \| x - \tilde{x} \|_\infty \le \epsilon \}. \nonumber
\end{align}
Further, an $\ell_\infty$-bound of size $\epsilon$ on the set $S_\bQ$ is defined as
\begin{align}
\epsball[S_\bQ] = \{& \tilde{S}_\bQ = \{ \tilde{x}_i \in \cX \}_{i=1}^n \mid 
\nonumber \\ 
& \tilde{x}_i \in \epsball[x_i], \quad \forall x_i \in S_\bQ, \tilde{x}_i \in \tilde{S}_\bQ \}. \nonumber
\end{align}
Without loss of generality, we assume that the adversarial perturbation is $\ell_{\infty}$-bounded of size $\epsilon$, i.e., $\tilde{S}_{\bQ} \in \epsball[S_{\bQ}]$.  We leave exploring the effects of other constraints that can bound the ``human imperception'' as the future work, such as Wasserstein-distance constraints~\cite{wong2019wasserstein}.

Under $\ell_\infty$-bounded attacks, we conduct our theoretical analysis of distributional shift in the test pairs as follows. 

\begin{proposition}
\label{proposition:mmd_shift}
    Under Assumptions~\ref{assump:param} to~\ref{assump:lipschitz}, we use $n_{\tr}$ samples to train a kernel $k_{\theta}$ parameterized with $\theta$ and $n_{\te}$ samples to run a test of significance level $\alpha$.
    Given the adversarial budget $\epsilon \geq 0$, the benign pair $(S_{\bP}, S_{\bQ})$ and the corresponding adversarial pair $(S_{\bP}, \tilde{S}_{\bQ})$ where $\tilde{S}_{\bQ} \in \epsball[S_{\bQ}]$,  with the probability at least $1 - \delta$, we have
    \begin{align}
    &\sup_{\theta}|{\widehat{\MMD}}^{2}(S_{\bP}, \tilde{S}_{\bQ};k_{\theta}) - {\widehat{\MMD}}^{2}(S_{\bP}, S_{\bQ};k_{\theta}) | \nonumber   \\ 
    & \leq \frac{8 L_2 \epsilon \sqrt{d}}{\sqrt{n_{\te}}}  \sqrt{2 \log \frac{2}{\delta} + 2\kappa\log (4\R_{\Theta} \sqrt{n_{\te}})} + \frac{8 L_1}{\sqrt{n_{\te}}}. \nonumber
    \end{align}
\end{proposition}
The proof is in Appendix~\ref{appendix:proof_mmd_shift}.

\remark{Proposition~\ref{proposition:mmd_shift} shows that $\epsilon$ can control the upper bound of distributional shift measured by MMD between samples in the test pair.
In other words, a small $\epsilon$ can ensure the difference between ${\widehat{\MMD}}^{2}(S_{\bP}, \tilde{S}_{\bQ}; k_{\theta})$ and ${\widehat{\MMD}}^{2}(S_{\bP}, S_{\bQ}; k_{\theta})$ is numerically small.
Therefore, an $\ell_{\infty}$-bounded adversary can make the adversarial perturbation imperceptible, thus guaranteeing the attack's \textit{invisibility}. 
}

Next, we provide a lemma that theoretically analyzes the adversary's influence on the estimated test criterion.
\begin{lemma}
\label{lemma:adv_benign_tp}
  In the setup of Proposition~\ref{proposition:mmd_shift}, with probability at least $1 - \delta$, we have
 \begin{align}
     &\sup_{\theta \in \bar{\Theta}_s} |\hat{\cF}(S_{\bP}, \tilde{S}_\bQ;k_{\theta}) - \hat{\cF}(S_{\bP}, S_\bQ;k_{\theta})| \nonumber \\
     & = \cO \bigg(\frac{\epsilon L_2 \sqrt{d \big(\log \frac{1}{\delta} + \kappa \log (\R_{\Theta} \sqrt{n_\te})\big) } + L_1 }{s \sqrt{n_\te}} \bigg). \nonumber
 \end{align}
\end{lemma}
The proof is in Appendix~\ref{appendix:proof_adv_benign_tp}.
\remark{Lemma~\ref{lemma:adv_benign_tp} shows that, when $\epsilon > 0$, a TST needs a larger number of test samples to facilitate the estimated test criterion on the adversarial test pair to converge to the estimated test criterion on the benign test pair.
In other words, the estimated test criterion in adversarial settings ($\epsilon > 0$) could be lower than the estimated test criterion in benign settings for a particular $n_\te$.}

Since the test criterion dominates the test power, Lemma~\ref{lemma:adv_benign_tp} motivates us to further theoretically analyze the adversary's effects on the lower bound of a TST's test power as follows.

\begin{theorem}
\label{theory:tp_attack}
    In the setup of Proposition~\ref{proposition:mmd_shift}, given $\hat{\theta}_{n_{\tr}} = \arg\max_{\theta \in \bar{\Theta}_s} \hat{\cF}(k_{\theta})$, $r^{(n_{\te})}$ denoting the rejection threshold, $\cF^{*} = \sup_{\theta \in \bar{\Theta}_s} \cF(k_\theta)$, and constants $ C_1, C_2, C_3$ depending on $\nu, L_1 , \lambda, s, \R_{\Theta}$ and $\kappa$, with probability at least $1-\delta$, the test under adversarial attack has power
    \begin{align}
    &\Pr(n_{\te} \widehat{\MMD}^2(S_\bP, \tilde{S}_\bQ;k_{\hat{\theta}_{n_\tr}}) > r^{(n_{\te})} ) \geq \Phi \bigg[ \! \sqrt{n_{\te}} \bigg( \cF^* \! -  \nonumber \\
    & \frac{C_1}{\sqrt{n_{\tr}}} \! \sqrt{\log \! \frac{\sqrt{n_{\tr}}}{\delta}} \! - \! \frac{C_2 L_2 \epsilon \sqrt{d}}{\sqrt{n_{\te}}} \sqrt{\log \! \frac{\sqrt{n_{\te}}}{\delta}} \bigg)\! - \! C_3 \! \sqrt{\log \frac{1}{\alpha}} \bigg]. \nonumber
    \end{align}
\end{theorem}

The proof is in Appendix~\ref{appendix:proof_tp_attack}.

\remark{Theorem~\ref{theory:tp_attack} indicates that the lower bound of test power can become lower with the increase of the adversarial budget $\epsilon$, the dimensionality of data $d$ and Lipschitz constant $L_2$ of the kernel function, which implies that the test power of a TST could be further degraded in the adversarial setting. In other words, a non-parametric TST could wrongly accept $\cH_0$ with a larger probability in the adversarial setting when $\bP \neq \bQ$ holds. Therefore, with the $\epsilon > 0$ being constrained within a reasonable range, there could exist an adversarial attack that can invisibly fool a non-parametric TST.}

\subsection{Generation of Adversarial Pairs}
\label{sec:realization}

\paragraph{Formulation.}
Motivated by Theorem~\ref{theory:asymptotics}, a TST could output a wrong judgement on an adversarial pair $(S_\bP, \tilde{S}_\bQ)$ with a larger probability when the test criterion $\hat{\cF}(S_{\bP},\tilde{S}_{\bQ})$ becomes smaller. Therefore, to generate an adversarial pair against a non-parametric TST, we update $\tilde{S}_\bQ$ via minimizing the test criterion $\hat{\cF}(S_{\bP},\tilde{S}_{\bQ})$.
We formulate adversarial attacks against a non-parametric TST $\cJ$ in the following: 
\begin{align}
\label{eq:attack}
    \tilde{S}_{\bQ} = \argmin_{\tilde{S}_{\bQ} \in \epsball[S_{\bQ}]} \hat{\cF}^{(\cJ)}(S_{\bP},\tilde{S}_{\bQ}),
\end{align}
where $\hat{\cF}^{(\cJ)}(\cdot,\cdot)$ is the test criterion, $\tilde{S}_\bQ$ is constrained in an $\epsilon$-ball centered at $S_\bQ$.

\paragraph{Realization.}
We utilize PGD~\cite{Madry_adversarial_training} to approximately 
solve the minimization problem of
Eq.~\eqref{eq:attack}. Given a starting point $S_{\bQ}^{(0)}$, step size $\rho > 0$, iteration number $t \in \mathbb{N}$,
and the size of adversarial budget $\epsilon \geq 0$, PGD works as follows:
\begin{align}
    S_\bQ^{(t+1)} \! = \! \{  \Pi&_{ \! \epsball[x_i^{(0)}]} \big( \! x_i^{(t)} \!- \! \rho \sign (\nabla_{x_i^{(t)}} \hat{\cF}(S_{\bP},S_{\bQ}^{(t)})\!)\! \big)\!\}_{i=1}^{n} , \nonumber
\end{align}
where $x_i^{(0)} \in S_\bQ^{(0)}$,  $x_i^{(t)} \in S_\bQ^{(t)}$, $\Pi_{\epsball[{x^{(0)}}]}(\cdot)$ is the projection function that projects the adversarial data back into the $\epsilon$-ball centered at ${x^{(0)}}$, and $\hat{\cF}(\cdot, \cdot)$ is a differentiable function. 

Further, we introduce a strategy that automatically schedules the step size $\rho$, which can improve the convergence of PGD~\cite{croce2020reliable}. We start with step size $\rho^{(0)} = \epsilon$ at iteration $0$ and identify whether it is necessary to halve the current step size at checkpoints $c_0, c_1, \ldots, c_n$. We set two conditions: 
\begin{enumerate}
\setlength{\itemsep}{0pt}
\setlength{\parsep}{0pt}
\setlength{\parskip}{0pt}
    \item $\sum_{i=c_{j-1}}^{c_j -1} \mathbbm{1}_{\hat{\cF}(S_{\bP}, S_{\bQ}^{(i+1)}) < \hat{\cF}(S_{\bP}, S_{\bQ}^{(i)})} < 0.75 \cdot (c_j - c_{j-1})$;
    \item $\rho^{(c_{j-1})} \equiv \rho^{(c_j)}$ and $\hat{\cF}_{\mathrm{min}}^{(c_{j-1})} \equiv \hat{\cF}_{\mathrm{min}}^{(c_{j})}$,
\end{enumerate}
where $\hat{\cF}_{\mathrm{min}}^{(t)}$ is the lowest value of the test criterion found in the first
$t$ iterations. If one of the conditions is triggered, then the step size at iteration $t=c_j$ is halved and $\rho^{(t)} = \rho^{(c_{j})}/2$ for every $t \in \{c_j+1, \ldots, c_{j+1}\}$. If at a checkpoint $c$, the step size gets halved, then we set $S_\bQ^{(c+1)}$ to the current $\tilde{S}_\bQ$.

\begin{algorithm}[t!]
   \caption{Ensemble Attack (EA)}
   \label{alg:era}
\begin{algorithmic}[1]
   \STATE {\bfseries Input:} benign pair $(S_{\bP}, S_{\bQ})$, maximum PGD step $T$, adversarial budget $\epsilon$, test criterion function set $\hat{\bF}$, weight set $\bW$, checkpoint $\mathbb{C} = \{c_0, \ldots, c_n\}$
   \STATE {\bfseries Output:} adversarial pair $(S_{\bP},\tilde{S}_{\bQ})$ 
   \STATE $S_{\bQ}^{(0)} \gets {S_{\bQ}}$ and $\rho \gets {\epsilon}$
   \STATE $S_{\bQ}^{(1)} \! \leftarrow \! \{  \Pi_{ \! \epsball[x_i^{(0)}]} \big( \! x_i^{(0)} \!- \! \rho \sign (\nabla_{x_i^{(0)}} \ell(S_{\bP},S_{\bQ}^{(0)})\!)\! \big)\!\}_{i=1}^{n}$
   \STATE $\ell_{\mathrm{min}} \gets {\min\{\ell(S_{\bP},S_{\bQ}^{(0)}), \ell(S_{\bP},S_{\bQ}^{(1)})} \}$
   \STATE $\tilde{S}_{\bQ} \gets S_{\bQ}^{(0)}$  \textbf{if} $\ell_{\mathrm{min}} \equiv \ell(S_{\bP},S_{\bQ}^{(0)})$ \textbf{else} $\tilde{S}_{\bQ} \gets S_{\bQ}^{(1)}$
   \FOR{$t=1$ \textbf{to} $T-1$}
   \STATE $S_{\bQ}^{(t+1)} \! \leftarrow \! \{ \! \Pi_{ \! \epsball[x_i^{(0)}]} \! \big( \! x_i^{(t)} \!- \! \rho \! \sign (\nabla_{x_i^{(t)}} \! \ell(S_{\bP}, \! S_{\bQ}^{(t)} \!)\!)\! \big)\!\}_{i=1}^{n} $
   \IF {$\ell_{\mathrm{min}} > \ell(S_{\bP},S_{\bQ}^{(t+1)})$}
   \STATE $\tilde{S}_{\bQ} \gets S_{\bQ}^{(t+1)}$ and $\ell_{\mathrm{min}} \gets \ell(S_{\bP},S_{\bQ}^{(t+1)})$
   \ENDIF
   \IF{$t \in \mathbb{C}$}
   \IF {Condition 1 \textbf{or} Condition 2}
   \STATE $\rho \gets \rho/2$ and $S_{\bQ}^{(t+1)} \gets \tilde{S}_{\bQ}$
   \ENDIF
   \ENDIF
   \ENDFOR
\end{algorithmic}
\end{algorithm}

\subsection{TST-Agnostic Ensemble Attack}

In practice, different TSTs have different formulations of the test criteria.
To provide a generic TST-agnostic attack framework, we propose the ensemble attack (EA) that finds the adversarial set $\tilde{S}_{\bQ}$ as follows.
\begin{align}
    \tilde{S}_{\bQ} = \argmin_{\tilde{S}_{\bQ} \in \epsball[S_{\bQ}]} \sum_{w^{(\cJ_i)} \in \bW, \hat{\cF}^{(\cJ_i)} \in \hat{\bF}} w^{(\cJ_i)}\hat{\cF}^{(\cJ_i)}(S_{\bP},\tilde{S}_{\bQ}) \nonumber, 
\end{align}
where $\bJ = \{\cJ_1,\cJ_2, \ldots, \cJ_n\}$ a set of non-parametric TSTs, $\hat{\bF} = \{\hat{\cF}^{(\cJ_1)}, \hat{\cF}^{(\cJ_2)},\ldots,\hat{\cF}^{(\cJ_n)}\}$ is a set composed of the test criterion for each TST $\cJ_i \in \bJ$, $\bW = \{w^{(\cJ_1)},w^{(\cJ_2)},\ldots,w^{(\cJ_n)}\}$ is a weight set, and $\sum_{w^{(\cJ_i)} \in \bW} w^{(\cJ_i)} = 1$. For notational simplicity, we let
\begin{align}
    \ell(S_{\bP},{\tilde{S}_{\bQ}}) = \sum_{w^{(\cJ_i)} \in \bW, \hat{\cF}^{(\cJ_i)} \in \hat{\bF}} w^{(\cJ_i)}\hat{\cF}^{(\cJ_i)}(S_{\bP},\tilde{S}_{\bQ}). \nonumber
\end{align}
We utilize ``PGD with a dynamic schedule of step size $\rho$'' (see above) to realize EA. We summarize the realization of EA in Algorithm~\ref{alg:era}. 
Note that an adversarial attack against a TST $\cJ$ is the special case of EA when we set $w^{(\cJ)} = 1$.

\section{Defending Non-Parametric TSTs}
\label{sec:defense}

In this section, to counteract the threats incurred by adversarial attacks, we propose defensive strategies to enhance the test power of non-parametric TSTs under attacks.

\subsection{A Simple Ensemble as A Vanilla Defense}
In machine learning, ensemble methods leverage various learning algorithms together to obtain better performance than could be obtained from any of the individual learning algorithms alone~\cite{opitz1999popular,rokach2010ensemble}.
Therefore, a simple ensemble of different non-parametric TSTs could be a vanilla defense. 
Correspondingly, we let the test power of an ensemble of TSTs measure the probability of any non-parametric TST $\cJ_i \in \bJ$ correctly rejecting $\cH_0$ when $\cH_1$ is true, i.e., for a particular $\bP \neq \bQ$,
\begin{align}
    \mathrm{TP}(\bJ) = \bE_{S_{\bP}\sim \bP^m, S_{\bQ}\sim \bQ^n}[\vee_{\cJ_i \in \bJ}\mathbbm{1}(\cJ_i(S_{\bP}, S_{\bQ}) = 1)]. \nonumber
\end{align} 
However, this simple defense cannot effectively improve the test power of TSTs under EA. We empirically find that EA can significantly degrade the test power of an ensemble of different TSTs (see Table~\ref{tab:attack_power}).
Therefore, the ensemble of TSTs is no longer an effective defensive strategy.

\subsection{Adversarially Learning Kernels for TSTs}
To effectively enhance the robustness of non-parametric TSTs, we propose a general defense which employs adversarial learning~\cite{Madry_adversarial_training} to obtain robust kernels for non-parametric TSTs.
The learning objective of robust kernels is formulated as a max-min optimization:
\begin{align}
\label{eq:rob_obj}
    \hat{\theta} \approx \argmax_{\theta} \min_{\tilde{S}_{\bQ} \in \epsball[S_{\bQ}]} \hat{\cF}(S_{\bP}, \tilde{S}_{\bQ}; k_{\theta}).
\end{align}

Eq.~\eqref{eq:rob_obj} is equivalent to a minimax optimization problem by simply flipping its inner minimization term and its outer maximization term simultaneously. Then, Danskin’s theorem~\cite{danskin1966theory} can apply~\cite{Madry_adversarial_training}. Therefore, we can adversarially learn the deep kernels with one step minimizing the test criterion to find an adversarial pair and one step maximizing the test criterion on the adversarial pair w.r.t. the parameters $\theta$.

\paragraph{Robust deep kernels for TSTs (MMD-RoD).} 
Since MMD-D~\cite{liu2020learning} has been validated as a superior non-parametric TST, our defense is based on the deep kernels, i.e., $\hat{\cF}^{(\RoD)}(\cdot, \cdot; k_{\theta}^{(\RoD)}) = \hat{\cF}^{(\D)}(\cdot, \cdot; k_{\theta}^{(\RoD)})$ where $k_{\theta}^{(\RoD)} = k_{\theta}^{(\D)}$. We let $\theta$ denote all the learnable parameters ($\gamma, \sigma_{\phi}, \sigma_{q}$ and the parameters of the DNN $\phi(\cdot)$) for a robust deep kernel. We summarize the training procedure of adversarially learning deep kernels in Algorithm~\ref{alg:rob_kernel}. The testing procedure of MMD-RoD exactly follows MMD-D~\cite{liu2020learning} and is introduced in Appendix~\ref{appendix:testing}.

\begin{algorithm}[t]
  \caption{Adversarially Learning Deep Kernels}
  \label{alg:rob_kernel}
\begin{algorithmic}[1]
  \STATE {\bfseries Input:} benign pair $(S_{\bP}, S_{\bQ})$, maximum PGD step $T$, adversarial budget $\epsilon$, checkpoint $\mathbb{C} = \{c_0, \ldots, c_n\}$, deep kernel $k_{\theta}^{(\RoD)}$ parameterized by $\theta$, training epochs $E$, learning rate $\eta$
  \STATE {\bfseries Output:} parameters of robust deep kernel $\theta$
  \FOR{$e=1$ \textbf{to} $E$}
  \STATE $X \gets$ minibatch from  $S_{\bP}$; $Y \gets$ minibatch from $S_{\bQ}$
  \STATE Generate an adversarial pair $(X,\tilde{Y})$ by Algorithm~\ref{alg:era} with setting $\hat{\bF} = \{\hat{\cF}^{(\RoD)}(\cdot,\cdot;k_{\theta}^{(\RoD)})\}$
  \STATE $\theta \gets \theta + \eta \nabla_{\theta}\hat{\cF}^{(\RoD)}(X,\tilde{Y};k_{\theta}^{(\RoD)})$
  \ENDFOR
  \setlength{\belowdisplayskip}{-3pt}
\end{algorithmic}
\end{algorithm}

\begin{table*}[t]
\centering
\caption{We report the average test power of six typical non-parametric TSTs ($\alpha=0.05$) as well as Ensemble on five benchmark datasets in benign and adversarial settings, respectively. The lower the test power under attacks is, the more adversarially vulnerable is the TST. }
\label{tab:attack_power}
\resizebox{\columnwidth * 2}{!} { 
\begin{tabular}{ccc|c|ccccccc}
\hline
Datasets  & $\epsilon$ & $n_{\te}$ & EA & MMD-D & MMD-G & C2ST-S & C2ST-L & ME & SCF & Ensemble  \\ \hline
\multirow{2}{*}{Blob} & \multirow{2}{*}{0.05} & \multirow{2}{*}{100} & $\times$ & 1.000\scriptsize{$\pm$0.000} & 1.000\scriptsize{$\pm$0.000} & 1.000\scriptsize{$\pm$0.000} & 1.000\scriptsize{$\pm$0.000} & 0.992\scriptsize{$\pm$0.002} & 0.962\scriptsize{$\pm$0.001} & 1.000\scriptsize{$\pm$0.000}  \\
  & & & $\surd$ & \textbf{0.131}\scriptsize{$\pm$0.007} & \textbf{0.099}\scriptsize{$\pm$0.003} & \textbf{0.021}\scriptsize{$\pm$0.003} & \textbf{0.715}\scriptsize{$\pm$0.091} & \textbf{0.154}\scriptsize{$\pm$0.011} & \textbf{0.098}\scriptsize{$\pm$0.022} & \textbf{0.846}\scriptsize{$\pm$0.030} \\ \hline
\multirow{2}{*}{HDGM} & \multirow{2}{*}{0.05} & \multirow{2}{*}{3000} & $\times$ & 1.000\scriptsize{$\pm$0.000} & 1.000\scriptsize{$\pm$0.000} & 1.000\scriptsize{$\pm$0.000} & 1.000\scriptsize{$\pm$0.000} & 1.000\scriptsize{$\pm$0.002} & 0.942\scriptsize{$\pm$0.013} & 1.000\scriptsize{$\pm$0.000}  \\
& & & $\surd$ & \textbf{0.259}\scriptsize{$\pm$0.009} & \textbf{0.081}\scriptsize{$\pm$0.003} & \textbf{0.105}\scriptsize{$\pm$0.000} & \textbf{0.090}\scriptsize{$\pm$0.000} & \textbf{0.500}\scriptsize{$\pm$0.025} & \textbf{0.006}\scriptsize{$\pm$0.000} & \textbf{0.734}\scriptsize{$\pm$0.078}  \\ \hline
\multirow{2}{*}{Higgs} & \multirow{2}{*}{0.05} & \multirow{2}{*}{5000} & $\times$ & 1.000\scriptsize{$\pm$0.000} & 1.000\scriptsize{$\pm$0.000} & 0.970\scriptsize{$\pm$0.002} & 0.984\scriptsize{$\pm$0.003} & 0.830\scriptsize{$\pm$0.042} & 0.675\scriptsize{$\pm$0.071} & 1.000\scriptsize{$\pm$0.000}  \\
 & & & $\surd$ & \textbf{0.027}\scriptsize{$\pm$0.001} & \textbf{0.002}\scriptsize{$\pm$0.000} & \textbf{0.065}\scriptsize{$\pm$0.000} & \textbf{0.080}\scriptsize{$\pm$0.006} & \textbf{0.263}\scriptsize{$\pm$0.022} & \textbf{0.058}\scriptsize{$\pm$0.005} & \textbf{0.422}\scriptsize{$\pm$0.013} \\ \hline
\multirow{2}{*}{MNIST} & \multirow{2}{*}{0.05} & \multirow{2}{*}{500} & $\times$ & 1.000\scriptsize{$\pm$0.000} & 0.904\scriptsize{$\pm$0.000} & 1.000\scriptsize{$\pm$0.000} & 1.000\scriptsize{$\pm$0.000} & 1.000\scriptsize{$\pm$0.000} & 0.386\scriptsize{$\pm$0.005} & 1.000\scriptsize{$\pm$0.000}  \\
 & & & $\surd$ & \textbf{0.087}\scriptsize{$\pm$0.040} & \textbf{0.102}\scriptsize{$\pm$0.002} & \textbf{0.003}\scriptsize{$\pm$0.000} & \textbf{0.005}\scriptsize{$\pm$0.000} & \textbf{0.062}\scriptsize{$\pm$0.002} & \textbf{0.001}\scriptsize{$\pm$0.000} & \textbf{0.213}\scriptsize{$\pm$0.026}  \\ \hline
\multirow{2}{*}{CIFAR-10} & \multirow{2}{*}{0.0314} & \multirow{2}{*}{500} & $\times$ & 1.000\scriptsize{$\pm$0.000} & 1.000\scriptsize{$\pm$0.000} & 1.000\scriptsize{$\pm$0.000} & 1.000\scriptsize{$\pm$0.000} & 1.000\scriptsize{$\pm$0.000} & 0.033\scriptsize{$\pm$0.001} & 1.000\scriptsize{$\pm$0.000}  \\
 & & & $\surd$ & \textbf{0.187}\scriptsize{$\pm$0.001} & \textbf{0.279}\scriptsize{$\pm$0.004} & \textbf{0.107}\scriptsize{$\pm$0.017} & \textbf{0.119}\scriptsize{$\pm$0.021} & \textbf{0.079}\scriptsize{$\pm$0.000} & \textbf{0.000}\scriptsize{$\pm$0.000} & \textbf{0.429}\scriptsize{$\pm$0.005} \\ \hline
\end{tabular}
}
\vskip -0.05in
\end{table*}

\section{Experiments}
\label{sec:exp}
In this section, we empirically uncover the adversarial vulnerabilities of non-parametric TSTs and demonstrate the efficacy of our proposed MMD-RoD in enhancing adversarial robustness of non-parametric TSTs.

\subsection{Test Power Evaluated under Ensemble Attacks}
\label{sec:attack_power}

We conduct six typical non-parametric TSTs (MMD-D, MMD-G, C2ST-S, C2ST-L, ME and SCF) under EA on five benchmark datasets---Blob~\cite{gretton2012kernel,jitkrittum2016interpretable,sutherland2016generative}, high-dimensional Gaussian mixture (HDGM)~\cite{liu2020learning}, Higgs~\cite{chwialkowski2015fast}, MNIST~\cite{lecun1998gradient,radford2015unsupervised} and CIFAR-10~\cite{krizhevsky2009learning_cifar10}. 
$\bP$ and $\bQ$ of each dataset are illustrated in Appendix~\ref{appendix:dataset_stats}.
Note that $\bP \neq \bQ$ in each dataset. 
For Blob, HDGM and Higgs, we randomly sample a training pair ($S_\bP^{\tr}$, $S_\bQ^{\tr}$) for learning a kernel once for each non-parametric TST. For MNIST and CIFAR-10, we select a subset of the available data as training data $S_\bP^{\tr}$ and $S_\bQ^{\tr}$. The training settings (e.g., the structure of neural network and the optimizer) follow~\citet{liu2020learning} and are illustrated in detail in Appendix~\ref{appendix:exp_detail}. 

During the testing procedure, we randomly sample 100 new pairs ($S_\bP^{\te}$, $S_\bQ^{\te}$), disjoint from the training data, as the benign test pairs. We let $n_{\tr}$ and $n_{\te}$ be large enough to ensure TSTs can achieve a high test power in benign settings. EA is implemented on each benign test pair and generates the corresponding adversarial test pair as the input for TSTs. We illustrate experimental settings of permutation test in Appendix~\ref{appendix:testing}. Note that we utilize the wild bootstrap process~\cite{chwialkowski2014wild} (introduced in Appendix~\ref{appendix:testing}) to resample the value of MMD for MMD-D and MMD-G (as well as MMD-RoD) since adversarial data are probably not IID. Wild bootstrap process guarantees that we can get correct p-values in non-IID/IID scenarios.
We repeat the full process 10 times, and report the average test power (comparing $\bP$ to $\bQ$) of each non-parametric TST as well as an ensemble of these six typical TSTs (denoted as ``Ensemble'') in Table~\ref{tab:attack_power}. In addition, we confirm that these TSTs have reasonable Type $\bigromanone$ errors (comparing $\bP$ to $\bP$) in Appendix~\ref{appendix:type1}.

EA minimizes a weighted sum of test criteria of six typical TSTs, i.e., $\hat{\bF} = \{\hat{\cF}^{(\D)}, \hat{\cF}^{(\G)}, \hat{\cF}^{(\rS)}, \hat{\cF}^{(\rL)}, \hat{\cF}^{(\ME)}, \hat{\cF}^{(\SCF)} \}$. 
Weight set $\bW$ is manually set for each dataset and is summarized in Table~\ref{tab:attack_weight} (Appendix~\ref{appendix:attack_config}). For all datasets, $T = 50$. $\epsilon$ for each dataset is summarized in Table~\ref{tab:attack_power}. 

In Table~\ref{tab:attack_power}, we implement EA in the white-box setting where we can obtain the non-parametric TST's all information (e.g., the kernel parameters).
Table~\ref{tab:attack_power} demonstrates that the test power of each particular non-parametric TST and even Ensemble are significantly deteriorated among all datasets. 
It empirically validates that many existing non-parametric TSTs suffer from severe adversarial vulnerabilities. 

In addition, we surprisingly find that $\epsilon=0.05$ is large enough to significantly degrade the test power on MNIST. In contrast, conventional adversarial attacks that aim to fool DNNs on MNIST need a larger adversarial budget $\epsilon$ which is up to $0.3$~\cite{Madry_adversarial_training}. It seems that non-parametric TSTs are more adversarially vulnerable than classifiers. However, this claim could be inaccurate for two reasons. First, attack target is different. We target to fool non-parametric TSTs that belong to hypothesis tests, while previous works aim to attack DNN-based classifiers. Second, measurement is different. We cannot fairly compare the non-parametric TST's test power to the classifier's classification accuracy.

\subsection{Adversarial Robustness of MMD-RoD}
\label{sec:mmd-rod-tp}
For hyperparameters of adversarially learning kernels, we keep $\epsilon$ same as the dataset-corresponding adversarial budget in Table~\ref{tab:attack_power}, and set $T=1$ for all datasets. Other training settings such as the structure of the neural network and the optimizer as well as the testing procedure of MMD-RoD exactly follow MMD-D~\cite{liu2020learning}. We call an ensemble of six typical TSTs and MMD-RoD as ``Ensemble$^{+}$''. 
Here, EA is conducted based on the test criteria of TSTs in Ensemble$^{+}$. As for $\bW$, we let $w^{(\RoD)}$ and $w^{(\D)}$ in this section be half of $w^{(\D)}$ in Section~\ref{sec:attack_power}. Other attack settings (e.g., $n_{\te}, T, \epsilon$) for each dataset follow Section~\ref{sec:attack_power}. The Type $\bigromanone$ error of MMD-RoD is reported in Appendix~\ref{appendix:type1}.

\begin{table}[h!]
\vskip -0.1in
\caption{Test power of MMD-RoD and Ensemble$^{+}$.}
\label{tab:mmd-rod}
\resizebox{\columnwidth}{!} { 
\begin{tabular}{c|c|ccccc}
\hline
 & EA & Blob & HDGM & Higgs & MNIST & CIFAR-10 \\ \hline
\multirow{2}{*}{MMD-RoD} & $\times$ & \textbf{1.00}\scriptsize{$\pm$0.00} & 0.61\scriptsize{$\pm$0.07} & 0.53\scriptsize{$\pm$0.00} & \textbf{1.00}\scriptsize{$\pm$0.12} & \textbf{1.00}\scriptsize{$\pm$0.00} \\ 
 & $\surd$ & \textbf{0.19}\scriptsize{$\pm$0.06} & 0.00\scriptsize{$\pm$0.01} & 0.23\scriptsize{$\pm$0.02} & \textbf{0.98}\scriptsize{$\pm$0.00} & \textbf{0.91}\scriptsize{$\pm$0.00} \\ \hline
\multirow{2}{*}{Ensemble$^{+}$} & $\times$ & 1.00\scriptsize{$\pm$0.00} & 1.00\scriptsize{$\pm$0.00} & 1.00\scriptsize{$\pm$0.00} & 1.00\scriptsize{$\pm$0.00} & 1.00\scriptsize{$\pm$0.00} \\ 
& $\surd$ & \textbf{0.89}\scriptsize{$\pm$0.01} & 0.73\scriptsize{$\pm$0.08} & 0.54\scriptsize{$\pm$0.04} & \textbf{0.98}\scriptsize{$\pm$0.00} & \textbf{0.95}\scriptsize{$\pm$0.00} \\ 
\hline
\end{tabular}
}
\vskip -0.1in
\end{table}

Table~\ref{tab:mmd-rod} reports the test power of MMD-RoD and Ensemble$^{+}$ in benign and adversarial settings. Table~\ref{tab:mmd-rod} shows that the test power of MMD-RoD and Ensemble$^{+}$ under EA are significantly enhanced on most datasets such as MNIST and CIFAR-10, even without sacrificing test power in the benign setting. It validates robust deep kernels can improve adversarial robustness of non-parametric TSTs.

We surprisingly observe in Table~\ref{tab:mmd-rod} that benign test power of MMD-RoD on MNIST and CIFAR-10 remains high while the test power under attacks is significantly improved. This seems to conflict with the robustness-accuracy trade-off in conventional adversarial training~\cite{Zhang_trades}. The main reason could be that the metric is different, i.e., test power for non-parametric TSTs v.s. classification accuracy for classifiers. Due to this difference, the trade-off between benign test power and adversarial robustness may not hold in the case of non-parametric TSTs. In addition, there are published papers~\cite{yang2020closer, pang2022robustness} that claimed there should be no trade-off between benign accuracy and adversarial robustness.

\vskip -0.1in

MMD-RoD unexpectedly performs poorly on HDGM and Higgs, which has low test power in both benign and adversarial settings. 
The poor performance in the benign setting could be attributed to that the most adversarial training pairs can lead to the cross-over mixture problem~\cite{zhang2020fat}, thus making the learning extremely difficult and even fail. 
The reason for the poor robustness could be that the number of training data is small since enhancing adversarial robustness needs more training data~\cite{schmidt2018adversarial_more_data}. 
Therefore, we believe that utilizing the style of friendly adversarial training~\cite{zhang2020fat} for learning kernels along with sampling more training data can further enhance the performance of MMD-RoD. We leave further improving MMD-RoD as future work.

\begin{figure}[t!]
    \centering
    \subfigure[Blob]{
    \begin{minipage}[b]{0.214\textwidth}
    \includegraphics[width=\textwidth]{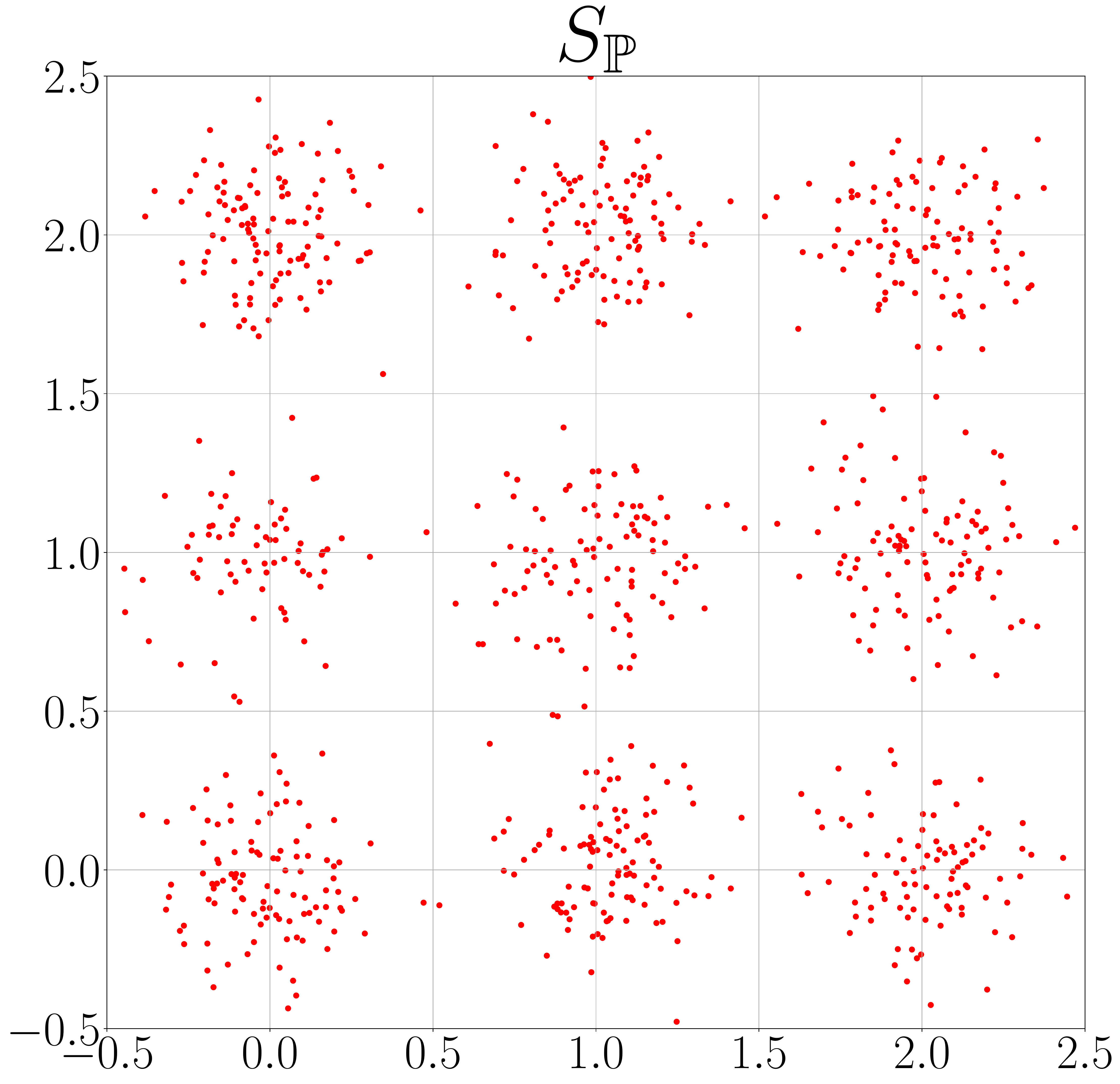} \\
    \includegraphics[width=\textwidth]{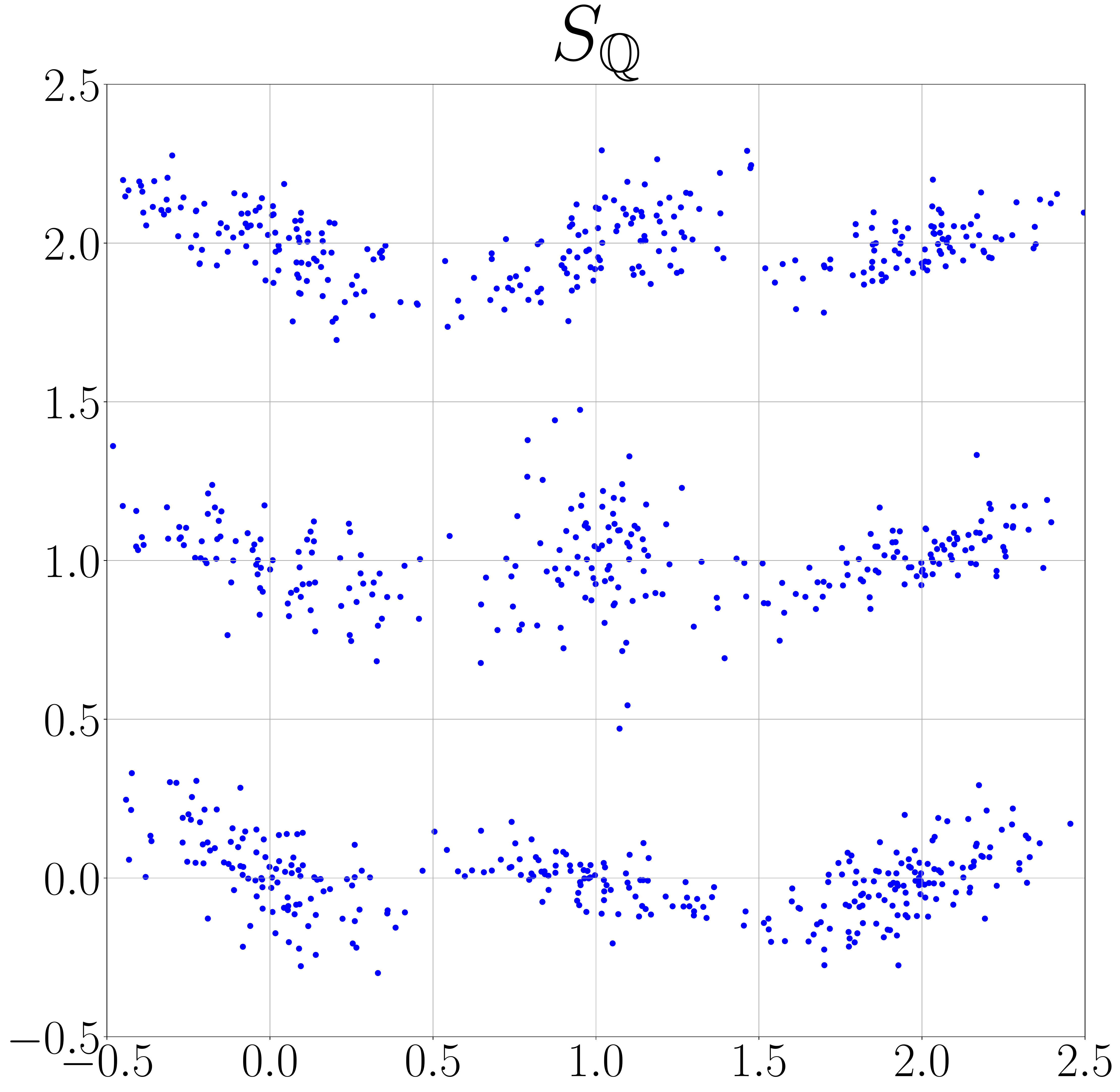} \\
    \includegraphics[width=\textwidth]{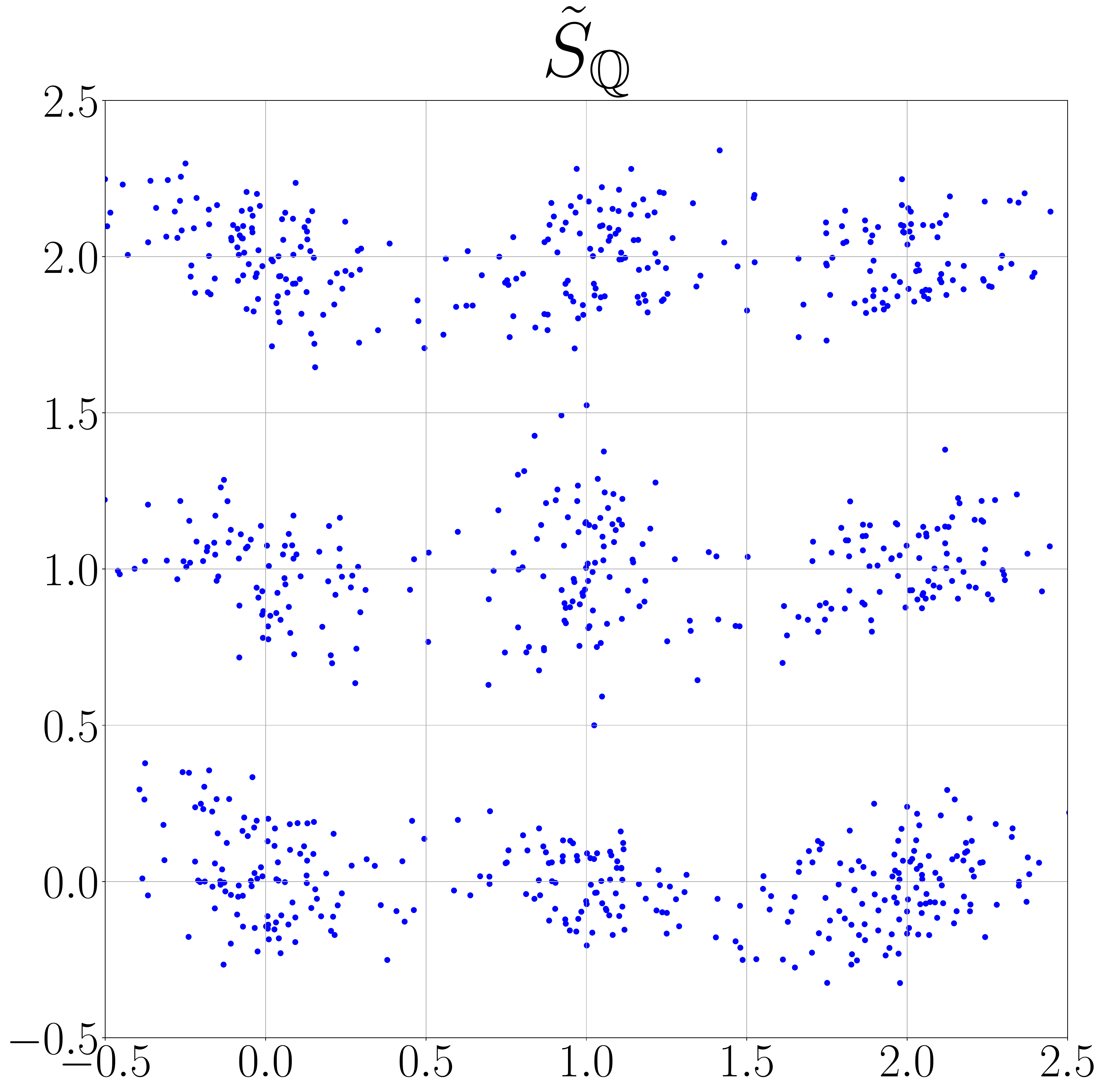}
    \end{minipage}
    }
    \hspace{+2mm}
    \subfigure[MNIST]{
    \begin{minipage}[b]{0.207\textwidth}
    \includegraphics[width=\textwidth]{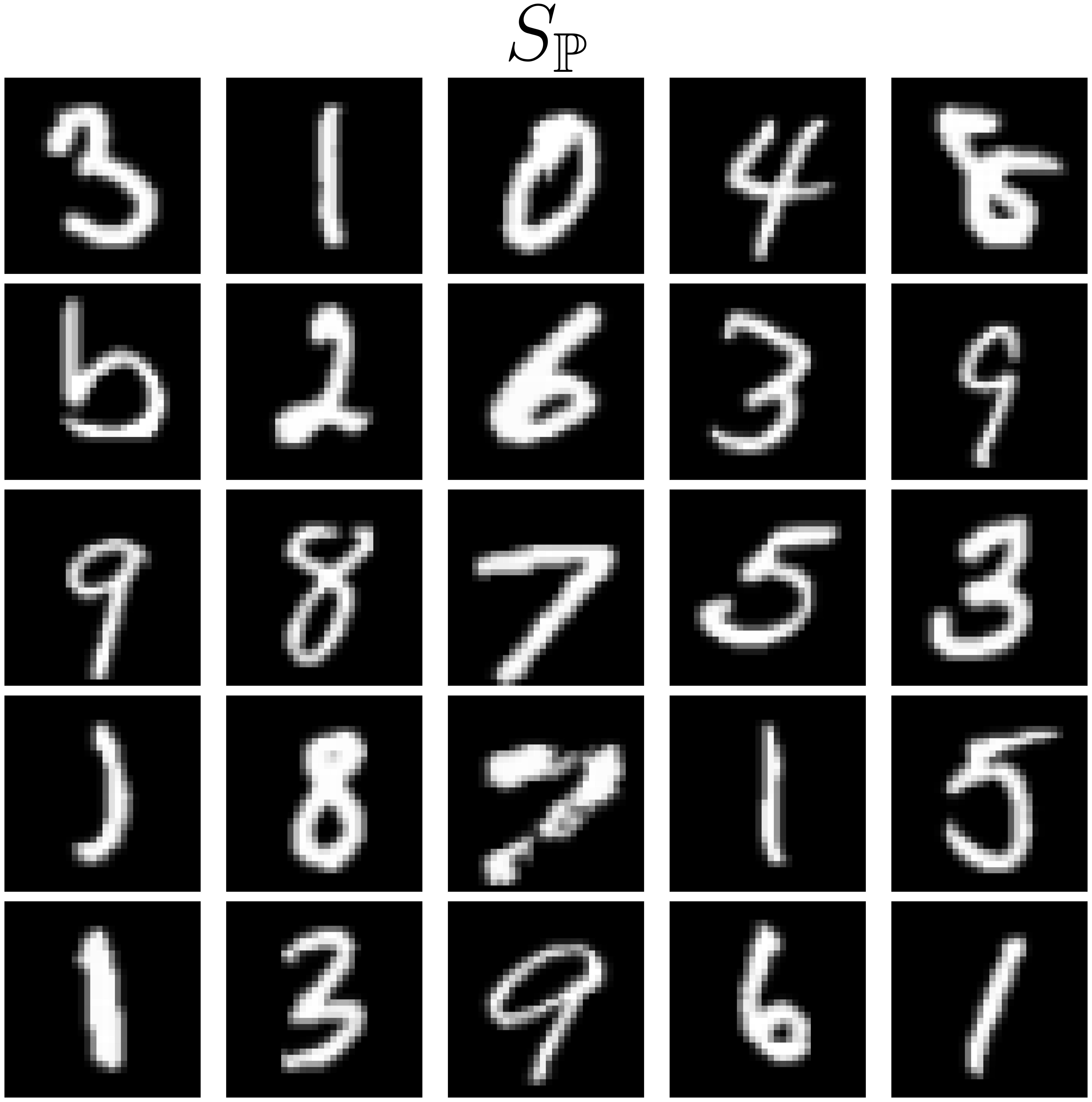} \\
    \includegraphics[width=\textwidth]{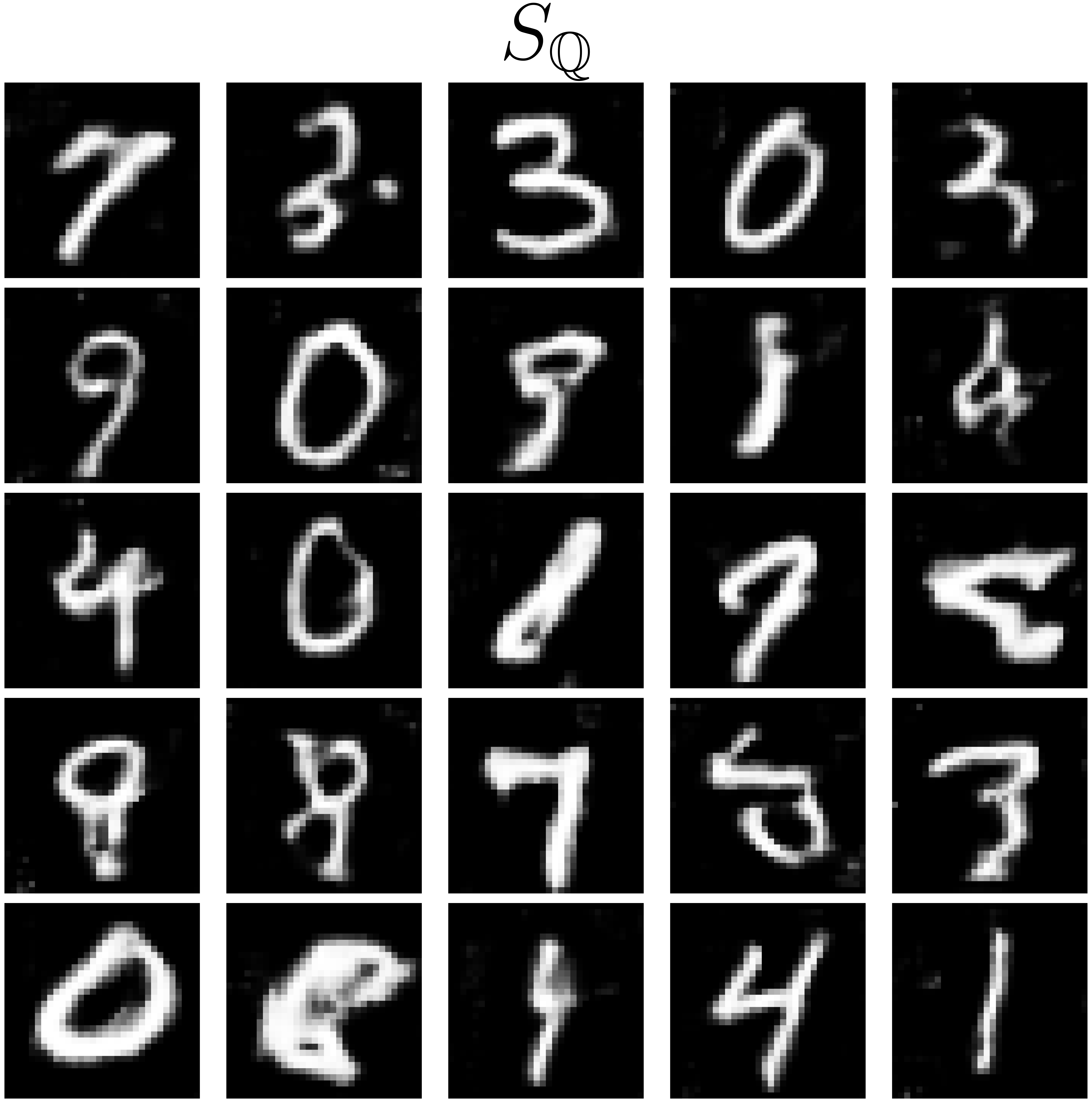} \\
    \includegraphics[width=\textwidth]{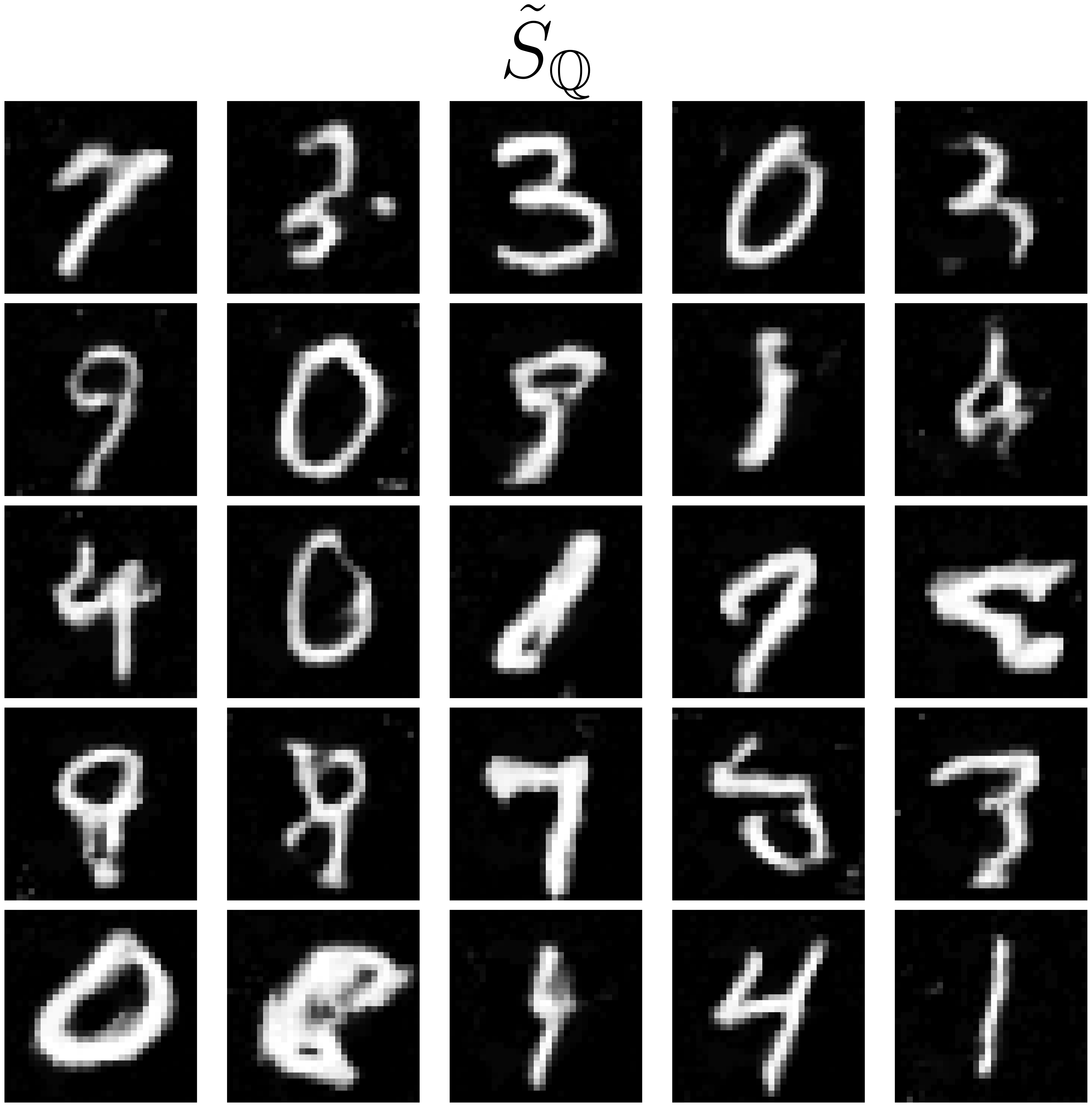}
    \end{minipage}
    }
    \vspace{-4mm}
    \caption{Visualization of adversarial test sets.}
    \label{fig:vis_adv_data}
    \vspace{-4mm}
\end{figure}

\subsection{Visualization of Adversarial Test Sets}
We visualize benign test set $S_{\bQ}$ (middle) and the corresponding adversarial test set $\tilde{S}_{\bQ}$ (bottom) on Blob and MNIST in Figure~\ref{fig:vis_adv_data} as well as CIFAR-10 in Figure~\ref{fig:corner_case_example}. The adversarial data are generated in the experiments illustrated in Section~\ref{sec:attack_power}. Note that the benign test pair we choose to visualize can be correctly judged as samples drawn from different distributions by each TST in Ensemble, and its corresponding adversarial test pair can successfully fool Ensemble. Due to limited space, we visualize only a part of samples from each set. Figure~\ref{fig:corner_case_example}-\ref{fig:vis_adv_data} verify that the differences between $S_{\bQ}$ and $\tilde{S}_{\bQ}$ is almost visually indistinguishable to humans, and meanwhile the distribution of $S_{\bP}$ is explicitly different from that of $\tilde{S}_{\bQ}$. Therefore, Figure~\ref{fig:corner_case_example}-\ref{fig:vis_adv_data} empirically validate that an $\ell_\infty$-bound can guarantee the invisibility of adversarial attacks.

\subsection{Ablation Studies on Important Hyperparameters}
\label{sec:ablation}
In this subsection, we conduct ablation studies on important hyperparameters, including $\epsilon$, $d$, $n_{\te}$ and $\bW$. 
Comprehensive results further validate non-parametric TSTs lack adversarial robustness.

\begin{figure}[t!]
    \centering
    \includegraphics[width=0.23\textwidth]{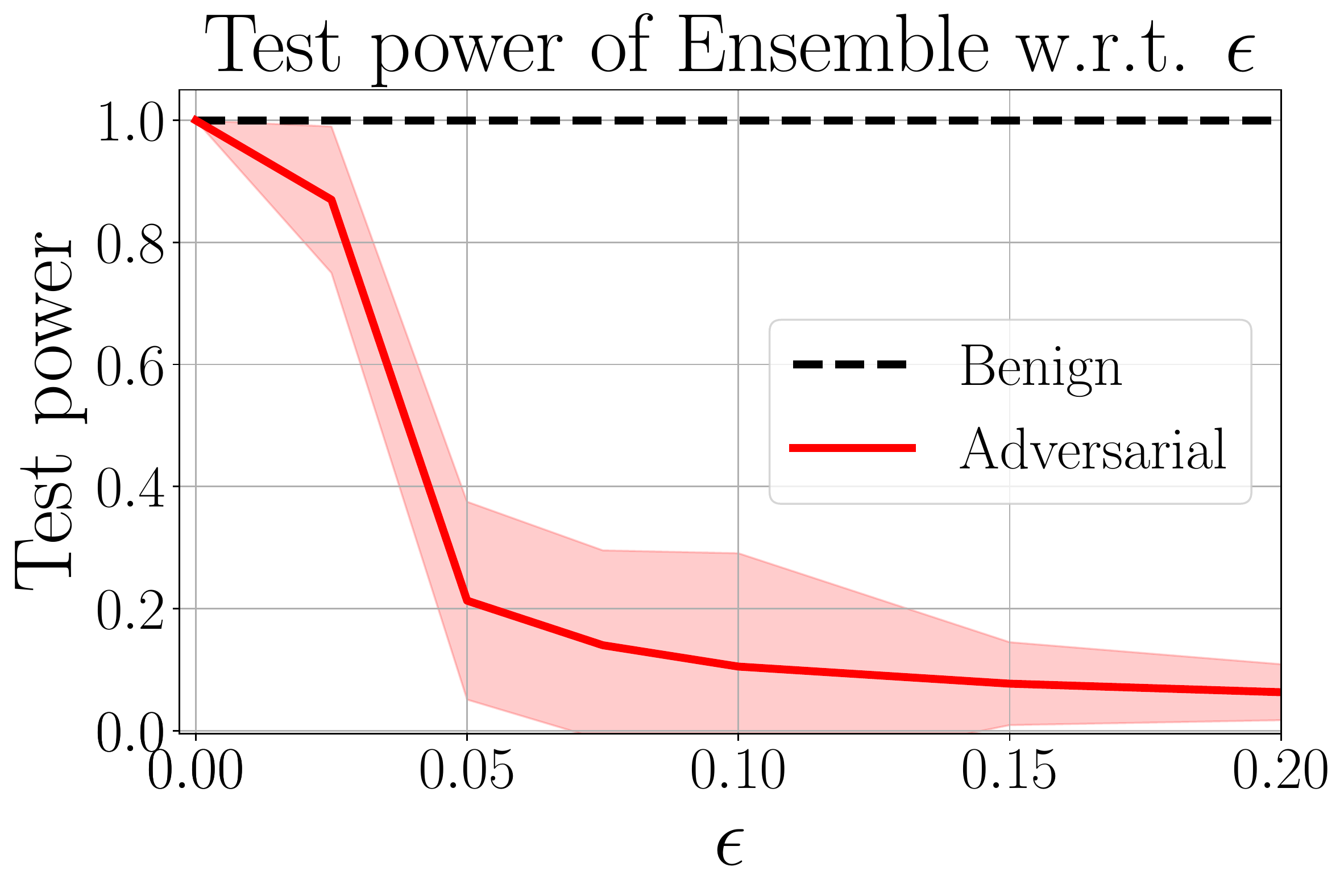}
    \includegraphics[width=0.23\textwidth]{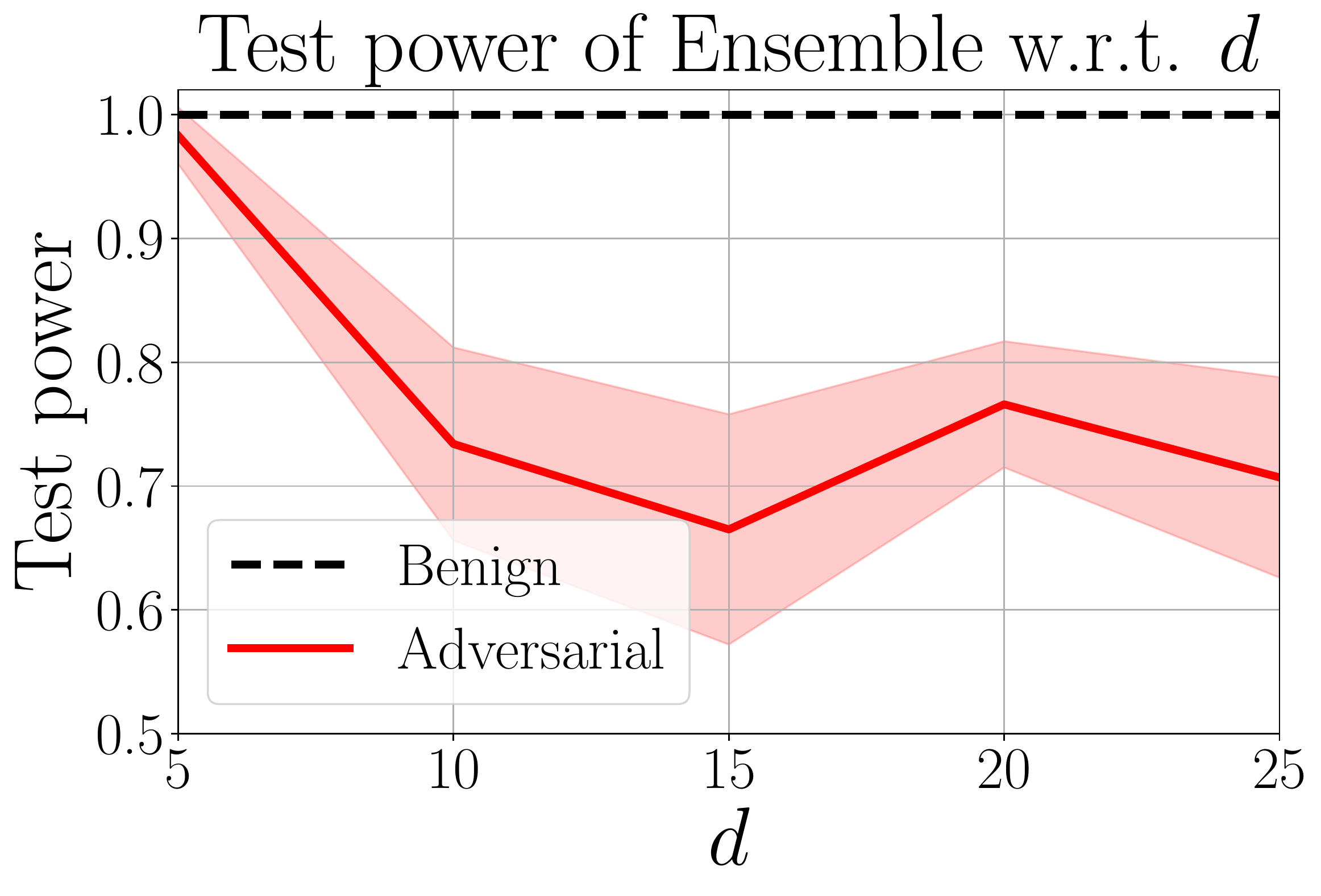}\\
    \includegraphics[width=0.23\textwidth]{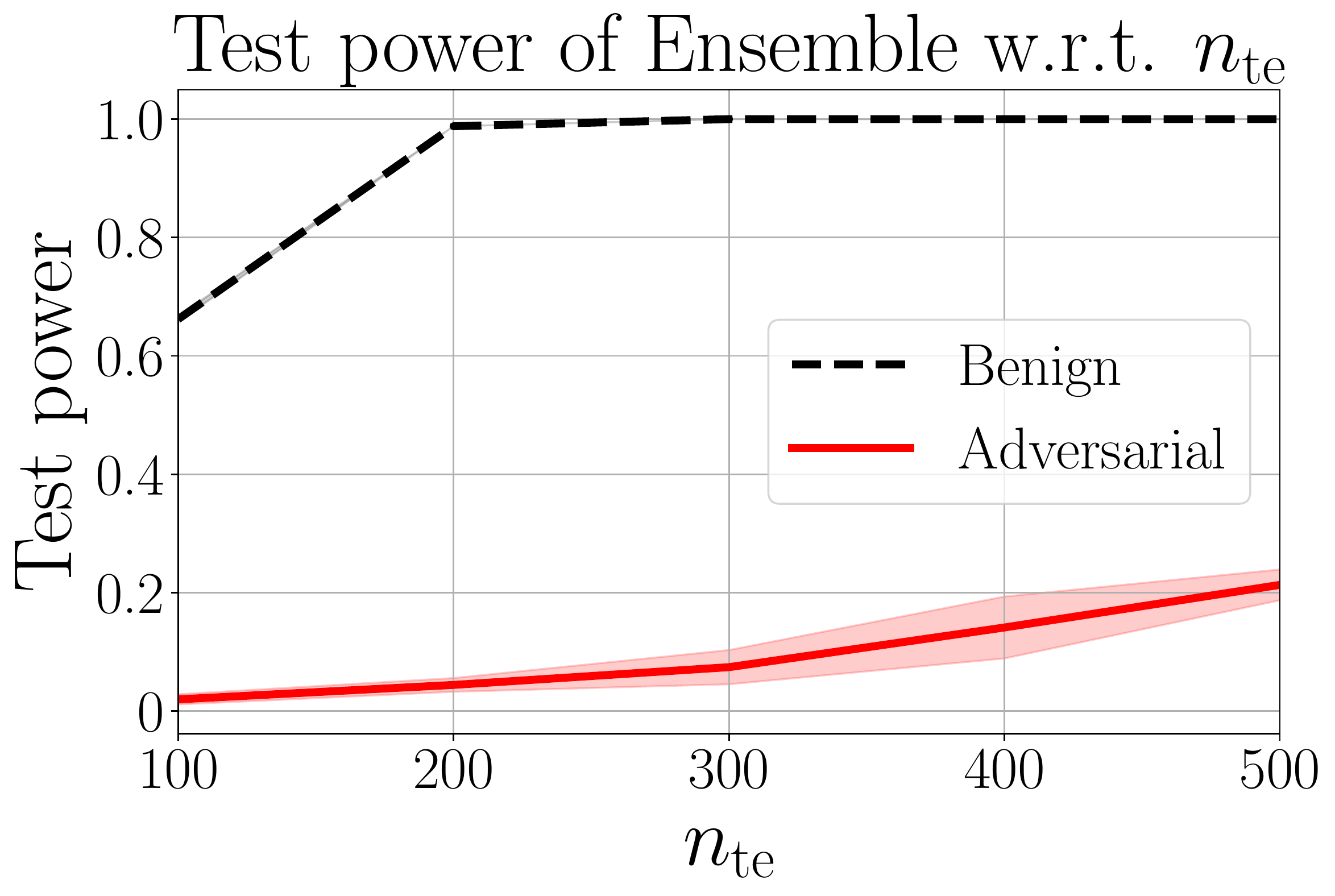}
    \includegraphics[width=0.23\textwidth]{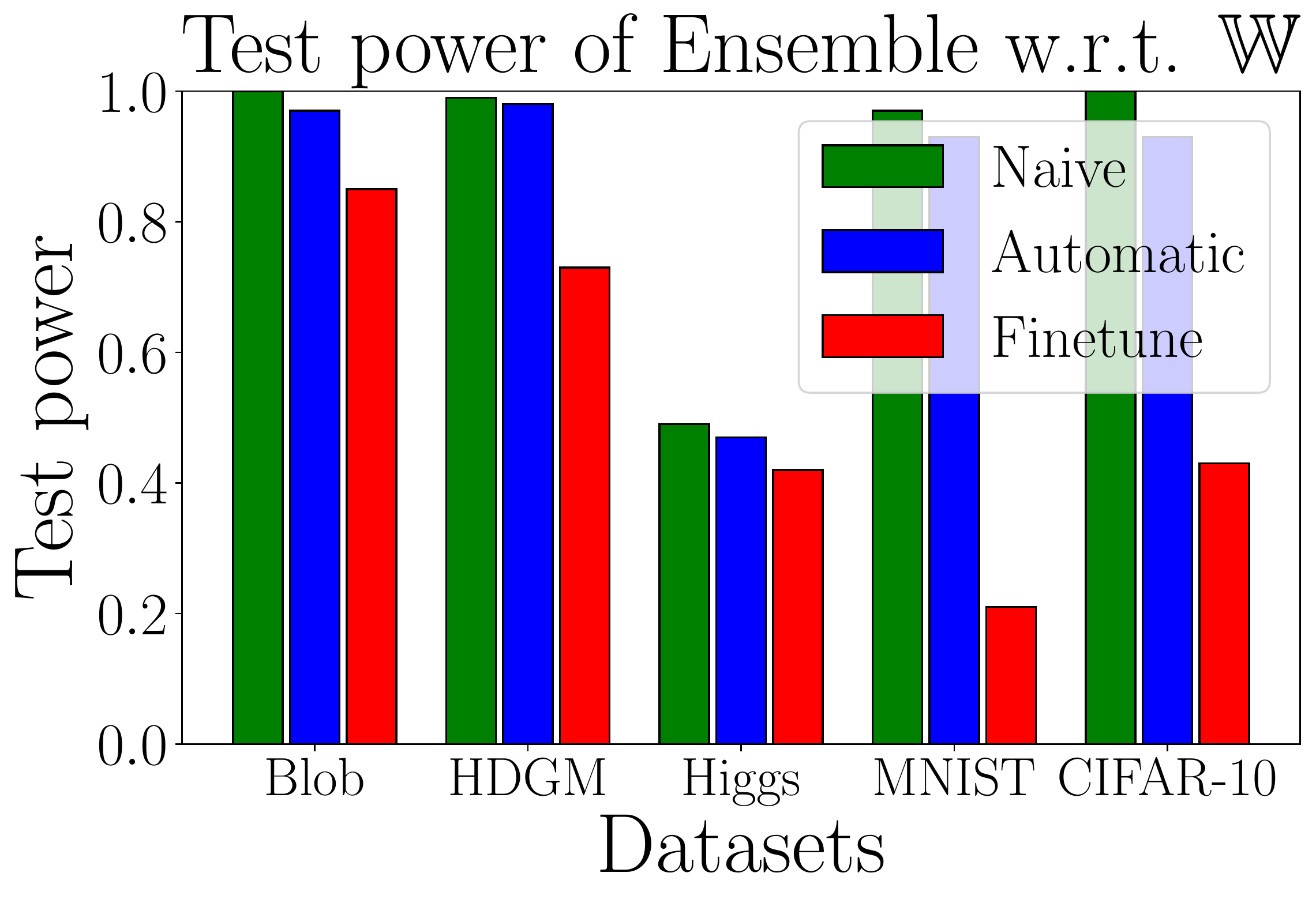}
    \vspace{-4mm}
    \caption{Ablation studies on important hyperparameters.}
    \label{fig:ablation}
    \vspace{-4mm}
\end{figure}

\vspace{-1.5mm}
\paragraph{Evaluation with different $\epsilon$.} We report the average test power of Ensemble under EA with $\epsilon \in \{0.00, 0.025, 0.05, 0.075, 0.1, 0.15, 0.2\}$ on MNIST.
Other settings keep same as Section~\ref{sec:attack_power}. 
The upper left panel of Figure~\ref{fig:ablation} shows that the test power of Ensemble under EA (red solid line) becomes lower as $\epsilon$ increases, and is significantly lower than the test power evaluated in the benign setting (black dash line) over different $\epsilon$, which is in line with the conclusion of Theorem~\ref{theory:tp_attack}.

\vspace{-1.5mm}
\paragraph{Evaluation with different $d$.} 
We evaluate the test power of Ensemble under EA on HDGM with different $d \in \{5,10,15,20,25\}$. The settings follow Section~\ref{sec:attack_power} except the dimensionality of Gaussian mixture. The upper right panel of Figure~\ref{fig:ablation} shows that the test power in the adversarial setting (red solid line) decreases as $d$ rises and remains lower than benign test power (black dash line). However, with larger $d$ (e.g., $d > 15$) the test power under EA does not keep degrading and even rises. We believe it is due to that the weight set for EA with larger $d$ is set inappropriately. We discuss the reasons in detail in Appendix~\ref{appendix:d_reason}.  

\vspace{-1.5mm}
\paragraph{Evaluation with different $n_{\te}$.} 
We evaluate the test power of Ensemble under EA on MNIST with different $n_{\te} \in \{100,200,300,400,500\}$, and use the same settings as Section~\ref{sec:attack_power}. The lower left panel of Figure~\ref{fig:ablation} shows that the test power in the adversarial setting (red solid line) increases as $n_{\te}$ becomes larger, but the test power under EA is always severely deteriorated compared to benign test power (black dash line), which reflects that non-parametric TSTs lack adversarial robustness.

\paragraph{Evaluation with different $\bW$.} We report the test power of Ensemble under EA with three weight strategies: 1) ``Naive'' (green pillar) denotes that we set $\bW=\{1/6,1/6,1/6,1/6,1/6,1/6\}$; 2) ``Automatic'' (blue pillar) denotes that we use the softmax of test criterion for each test as $\bW$ at each PGD iteration, i.e., $w^{(\cJ_i)} = \frac{\exp (\hat{\cF}^{(\cJ_i)})}{\sum_{j=1}^n \exp (\hat{\cF}^{(\cJ_j)}) }$; 3) ``Finetune'' (red pillar) denotes that we set manually-finetuned $\bW$ for each dataset. The finetuned weight set is summarized in Appendix~\ref{appendix:exp_detail}. Other settings follow Section~\ref{sec:attack_power}. The lower right panel of Figure~\ref{fig:ablation} shows that the test power of Ensemble under EA can be severely deteriorated with an appropriate weight strategy. 

\subsection{Transferability of Adversarial Attacks}
\label{sec:transferability}
Further, we empirically demonstrate that our proposed EA against non-parametric TSTs has transferability. 

\paragraph{Transferability between different types of non-parametric TSTs.} 
We report test power of non-parametric TSTs under the adversarial attack against a certain type of TSTs on MNIST in Figure~\ref{fig:MNIST_one_for_all} and the test power of non-parametric TSTs under EA against a TST ensemble composed by leaving one TST out of Ensemble on MNIST in Figure~\ref{fig:MNIST_leave_one_out}. The experimental details and results are in Appendix~\ref{appendix:extensive_exp}. Figure~\ref{fig:MNIST_one_for_all} shows that attacks against a certain type of TST sometimes can fool other types of TSTs. Figure~\ref{fig:MNIST_leave_one_out} demonstrates that attacks against an ensemble of TSTs sometimes can successfully fool TSTs that are not included in the attack ensemble. Therefore, Figure~\ref{fig:MNIST_transferability} validates our proposed EA has transferability between different types of non-parametric TSTs.

\paragraph{Transferability between target and surrogate non-parametric TSTs.} 
Here, we assume that the attacker cannot obtain the target non-parametric TST's kernel parameters and training data, and it only knows the target non-parametric TST's test criterion (including its kernel function). 
We generate adversarial pairs via EA based on an ensemble of surrogate non-parametric TSTs on MNIST (other attack configurations follow Section~\ref{sec:attack_power}) and then report the average test power of target tests on these adversarial pairs in Table~\ref{tab:black}. Surrogate tests are trained on the training data with different random seeds. Table~\ref{tab:black} shows that the test power of each target non-parametric TST and Ensemble are deteriorated under EA based on surrogate non-parametric TSTs, which further validates that existing non-parametric TSTs are adversarially vulnerable.

\begin{table}[h!]
\centering
\vskip -0.1in
\caption{Transferability between target and surrogate non-parametric TSTs.}
\label{tab:black}
\resizebox{\columnwidth}{!} { 
\begin{tabular}{ccccccc}
\hline
MMD-D & MMD-G & C2ST-S & C2ST-L & ME & SCF & Ensemble \\ \hline
0.564\scriptsize{$\pm$0.09} & 0.149\scriptsize{$\pm$0.00} & 0.418\scriptsize{$\pm$0.03} & 0.471\scriptsize{$\pm$0.04} & 0.064\scriptsize{$\pm$0.01} & 0.001\scriptsize{$\pm$0.00} & 0.751\scriptsize{$\pm$0.01}   \\ \hline
\end{tabular}
}
\vskip -0.1in
\end{table}

\paragraph{Transferability between different test sets drawn from $\bP$.} 
We replace the set $S_\bP$ with $S'_\bP$ where $S'_\bP$ is drawn from the distribution $\bP$ with different random seeds (i.e., $S_\bP \neq S'_\bP$). $\tilde{S}_\bQ$ is generated by EA on the benign test pair $(S_\bP, S_\bQ)$. We report the average test power of non-parametric TSTs on $(S'_\bP, \tilde{S}_\bQ)$ under EA on MNIST (details follow Section~\ref{sec:attack_power}) in Table~\ref{tab:transfer}. Table~\ref{tab:transfer} shows that EA still hurts the test power of non-parametric TSTs on $(S'_\bP, \tilde{S}_\bQ)$, and implies that EA has a good transferability property between different test sets drawn from $\bP$.

\begin{table}[h!]
\vskip -0.1in
\centering
\caption{Transferability between different test sets drawn form $\bP$.}
\label{tab:transfer}
\resizebox{\columnwidth}{!} { 
\begin{tabular}{ccccccc}
\hline
MMD-D & MMD-G & C2ST-S & C2ST-L & ME & SCF & Ensemble  \\ \hline
0.166\scriptsize{$\pm$0.05} & 0.201\scriptsize{$\pm$0.00} & 0.013\scriptsize{$\pm$0.00} & 0.018\scriptsize{$\pm$0.00} & 0.270\scriptsize{$\pm$0.03} & 0.017\scriptsize{$\pm$0.01} & 0.486\scriptsize{$\pm$0.04}  \\ \hline
\end{tabular}
}
\vskip -0.1in
\end{table}

\section{Conclusions}
This paper systematically studies adversarial robustness of non-parametric TSTs. We propose a generic ensemble attack framework which reveals non-parametric TSTs are adversarially vulnerable.To counteract these risks, we propose to adversarially learn kernels for non-parametric TSTs. We empirically show that SOTA non-parametric TSTs can fail catastrophically under adversarial attacks, and our proposed MMD-RoD can substantially enhance the adversarial robustness of non-parametric TSTs. 
We believe our work makes people aware of potential risks when they apply non-parametric TSTs to critical applications.

One of the limitations of our current work is that our proposed attack method is computationally heavy and user-dependent, in that it needs very large GPU memory when $n_{\te}$ is too large and the weight set needs to be manually finetuned. Future research includes (a) how to fool non-parametric TSTs by perturbing fewer samples, (b) how to adaptively adjust the weight set at each PGD iteration.

\section*{Acknowledgements}
Jingfeng Zhang was supported by JST, ACT-X Grant Number JPMJAX21AF.
Masashi Sugiyama was supported by JST AIP Acceleration Research Grant Number JPMJCR20U3 and the Institute for AI and Beyond, UTokyo.
Mohan Kankanhalli's research is supported by the National Research Foundation, Singapore under its Strategic Capability Research Centres Funding Initiative. Any opinions, findings and conclusions or recommendations expressed in this material are those of the author(s) and do not reflect the views of National Research Foundation, Singapore.

\clearpage
\nocite{langley00}

\bibliography{reference}

\begin{thebibliography}{127}
\providecommand{\natexlab}[1]{#1}
\providecommand{\url}[1]{\texttt{#1}}
\expandafter\ifx\csname urlstyle\endcsname\relax
  \providecommand{\doi}[1]{doi: #1}\else
  \providecommand{\doi}{doi: \begingroup \urlstyle{rm}\Url}\fi

\bibitem[Alaifari et~al.(2019)Alaifari, Alberti, and
  Gauksson]{Alaifari19_iclr_deformation}
Alaifari, R., Alberti, G.~S., and Gauksson, T.
\newblock Adef: an iterative algorithm to construct adversarial deformations.
\newblock In \emph{ICLR}, 2019.

\bibitem[Amsaleg et~al.(2017)Amsaleg, Bailey, Barbe, Erfani, Houle, Nguyen, and
  Radovanovi{\'c}]{amsaleg2017vulnerability}
Amsaleg, L., Bailey, J., Barbe, D., Erfani, S., Houle, M.~E., Nguyen, V., and
  Radovanovi{\'c}, M.
\newblock The vulnerability of learning to adversarial perturbation increases
  with intrinsic dimensionality.
\newblock In \emph{2017 IEEE Workshop on Information Forensics and Security
  (WIFS)}, pp.\  1--6. IEEE, 2017.

\bibitem[Andriushchenko et~al.(2020)Andriushchenko, Croce, Flammarion, and
  Hein]{andriushchenko2020square}
Andriushchenko, M., Croce, F., Flammarion, N., and Hein, M.
\newblock Square attack: a query-efficient black-box adversarial attack via
  random search.
\newblock In \emph{European Conference on Computer Vision}, pp.\  484--501.
  Springer, 2020.

\bibitem[Athalye et~al.(2018)Athalye, Carlini, and
  Wagner]{Athalye_ICML_18_Obfuscated_Gradients}
Athalye, A., Carlini, N., and Wagner, D.~A.
\newblock Obfuscated gradients give a false sense of security: Circumventing
  defenses to adversarial examples.
\newblock In \emph{ICML}, 2018.

\bibitem[Baldi et~al.(2014)Baldi, Sadowski, and Whiteson]{baldi2014searching}
Baldi, P., Sadowski, P., and Whiteson, D.
\newblock Searching for exotic particles in high-energy physics with deep
  learning.
\newblock \emph{Nature communications}, 5\penalty0 (1):\penalty0 1--9, 2014.

\bibitem[Bhattacharjee \& Chaudhuri(2020)Bhattacharjee and
  Chaudhuri]{bhattacharjee2020non}
Bhattacharjee, R. and Chaudhuri, K.
\newblock When are non-parametric methods robust?
\newblock In \emph{International Conference on Machine Learning}, pp.\
  832--841. PMLR, 2020.

\bibitem[Bhattacharjee \& Chaudhuri(2021)Bhattacharjee and
  Chaudhuri]{bhattacharjee2021consistent}
Bhattacharjee, R. and Chaudhuri, K.
\newblock Consistent non-parametric methods for maximizing robustness.
\newblock \emph{Advances in Neural Information Processing Systems}, 34, 2021.

\bibitem[Biggio \& Roli(2018)Biggio and Roli]{biggio2018wild}
Biggio, B. and Roli, F.
\newblock Wild patterns: Ten years after the rise of adversarial machine
  learning.
\newblock \emph{Pattern Recognition}, 84:\penalty0 317--331, 2018.

\bibitem[Bi{\'n}kowski et~al.(2018)Bi{\'n}kowski, Sutherland, Arbel, and
  Gretton]{binkowski2018demystifying}
Bi{\'n}kowski, M., Sutherland, D.~J., Arbel, M., and Gretton, A.
\newblock Demystifying mmd gans.
\newblock \emph{arXiv preprint arXiv:1801.01401}, 2018.

\bibitem[Borgwardt et~al.(2006)Borgwardt, Gretton, Rasch, Kriegel,
  Sch{\"o}lkopf, and Smola]{borgwardt2006integrating}
Borgwardt, K.~M., Gretton, A., Rasch, M.~J., Kriegel, H.-P., Sch{\"o}lkopf, B.,
  and Smola, A.~J.
\newblock Integrating structured biological data by kernel maximum mean
  discrepancy.
\newblock \emph{Bioinformatics}, 22\penalty0 (14):\penalty0 e49--e57, 2006.

\bibitem[Cai et~al.(2018)Cai, Liu, and Song]{Cai_CAT}
Cai, Q., Liu, C., and Song, D.
\newblock Curriculum adversarial training.
\newblock In \emph{IJCAI}, 2018.

\bibitem[Carlini \& Wagner(2017)Carlini and Wagner]{Carlini017_CW}
Carlini, N. and Wagner, D.~A.
\newblock Towards evaluating the robustness of neural networks.
\newblock In \emph{Symposium on Security and Privacy (SP)}, 2017.

\bibitem[Carmon et~al.(2019)Carmon, Raghunathan, Schmidt, Liang, and
  Duchi]{carmon2019unlabeled}
Carmon, Y., Raghunathan, A., Schmidt, L., Liang, P., and Duchi, J.~C.
\newblock Unlabeled data improves adversarial robustness.
\newblock In \emph{NeurIPS}, 2019.

\bibitem[Chen \& Friedman(2017)Chen and Friedman]{chen2017new}
Chen, H. and Friedman, J.~H.
\newblock A new graph-based two-sample test for multivariate and object data.
\newblock \emph{Journal of the American statistical association}, 112\penalty0
  (517):\penalty0 397--409, 2017.

\bibitem[Chen et~al.(2019)Chen, Zhang, Boning, and
  Hsieh]{DBLP:conf/icml/ChenZBH19}
Chen, H., Zhang, H., Boning, D.~S., and Hsieh, C.-J.
\newblock Robust decision trees against adversarial examples.
\newblock In \emph{ICML}, pp.\  1122--1131, 2019.

\bibitem[Chen et~al.(2020)Chen, Jordan, and
  Wainwright]{chen2020hopskipjumpattack}
Chen, J., Jordan, M.~I., and Wainwright, M.~J.
\newblock Hopskipjumpattack: A query-efficient decision-based attack.
\newblock In \emph{2020 ieee symposium on security and privacy (sp)}, pp.\
  1277--1294. IEEE, 2020.

\bibitem[Chen et~al.(2018)Chen, Sharma, Zhang, Yi, and Hsieh]{chen2018ead}
Chen, P.-Y., Sharma, Y., Zhang, H., Yi, J., and Hsieh, C.-J.
\newblock Ead: elastic-net attacks to deep neural networks via adversarial
  examples.
\newblock In \emph{Thirty-second AAAI conference on artificial intelligence},
  2018.

\bibitem[Chen et~al.(2021)Chen, Zhang, Liu, Chang, and Wang]{chen2021robust}
Chen, T., Zhang, Z., Liu, S., Chang, S., and Wang, Z.
\newblock Robust overfitting may be mitigated by properly learned smoothening.
\newblock In \emph{ICLR}, 2021.

\bibitem[Cheng et~al.(2019)Cheng, Le, Chen, Zhang, Yi, and
  Hsieh]{cheng2018queryefficient}
Cheng, M., Le, T., Chen, P.-Y., Zhang, H., Yi, J., and Hsieh, C.-J.
\newblock Query-efficient hard-label black-box attack: An optimization-based
  approach.
\newblock In \emph{International Conference on Learning Representations}, 2019.
\newblock URL \url{https://openreview.net/forum?id=rJlk6iRqKX}.

\bibitem[Cheng et~al.(2020)Cheng, Singh, Chen, Chen, Liu, and
  Hsieh]{Cheng2020Sign-OPT:}
Cheng, M., Singh, S., Chen, P.~H., Chen, P.-Y., Liu, S., and Hsieh, C.-J.
\newblock Sign-opt: A query-efficient hard-label adversarial attack.
\newblock In \emph{International Conference on Learning Representations}, 2020.
\newblock URL \url{https://openreview.net/forum?id=SklTQCNtvS}.

\bibitem[Cheng \& Cloninger(2019)Cheng and Cloninger]{cheng2019classification}
Cheng, X. and Cloninger, A.
\newblock Classification logit two-sample testing by neural networks.
\newblock \emph{arXiv preprint arXiv:1909.11298}, 2019.

\bibitem[Chwialkowski et~al.(2014)Chwialkowski, Sejdinovic, and
  Gretton]{chwialkowski2014wild}
Chwialkowski, K.~P., Sejdinovic, D., and Gretton, A.
\newblock A wild bootstrap for degenerate kernel tests.
\newblock In \emph{Advances in neural information processing systems}, pp.\
  3608--3616, 2014.

\bibitem[Chwialkowski et~al.(2015)Chwialkowski, Ramdas, Sejdinovic, and
  Gretton]{chwialkowski2015fast}
Chwialkowski, K.~P., Ramdas, A., Sejdinovic, D., and Gretton, A.
\newblock Fast two-sample testing with analytic representations of probability
  measures.
\newblock \emph{Advances in Neural Information Processing Systems},
  28:\penalty0 1981--1989, 2015.

\bibitem[Croce \& Hein(2020)Croce and Hein]{croce2020reliable}
Croce, F. and Hein, M.
\newblock Reliable evaluation of adversarial robustness with an ensemble of
  diverse parameter-free attacks.
\newblock In \emph{ICML}, 2020.

\bibitem[Croce et~al.(2020)Croce, Andriushchenko, Sehwag, Debenedetti,
  Flammarion, Chiang, Mittal, and Hein]{croce2020robustbench}
Croce, F., Andriushchenko, M., Sehwag, V., Debenedetti, E., Flammarion, N.,
  Chiang, M., Mittal, P., and Hein, M.
\newblock Robustbench: a standardized adversarial robustness benchmark.
\newblock \emph{arXiv preprint arXiv:2010.09670}, 2020.

\bibitem[Danskin(1966)]{danskin1966theory}
Danskin, J.~M.
\newblock The theory of max-min, with applications.
\newblock \emph{SIAM Journal on Applied Mathematics}, 14\penalty0 (4):\penalty0
  641--664, 1966.

\bibitem[Ding et~al.(2020)Ding, Sharma, Lui, and Huang]{ding2020mma}
Ding, G.~W., Sharma, Y., Lui, K. Y.~C., and Huang, R.
\newblock Mma training: Direct input space margin maximization through
  adversarial training.
\newblock In \emph{ICLR}, 2020.

\bibitem[Dong et~al.(2020)Dong, Li, Wang, and Xu]{dong2020adversarially}
Dong, M., Li, Y., Wang, Y., and Xu, C.
\newblock Adversarially robust neural architectures.
\newblock \emph{arXiv preprint arXiv:2009.00902}, 2020.

\bibitem[Dwass(1957)]{dwass1957modified}
Dwass, M.
\newblock Modified randomization tests for nonparametric hypotheses.
\newblock \emph{The Annals of Mathematical Statistics}, pp.\  181--187, 1957.

\bibitem[Erdemir et~al.(2021)Erdemir, Bickford, Melis, and
  Aydore]{erdemir2021adversarial}
Erdemir, E., Bickford, J., Melis, L., and Aydore, S.
\newblock Adversarial robustness with non-uniform perturbations.
\newblock In Beygelzimer, A., Dauphin, Y., Liang, P., and Vaughan, J.~W.
  (eds.), \emph{Advances in Neural Information Processing Systems}, 2021.
\newblock URL \url{https://openreview.net/forum?id=oi08QWKs84}.

\bibitem[Fang et~al.(2020{\natexlab{a}})Fang, Lu, Niu, and
  Sugiyama]{DBLP:conf/nips/FangL0S20}
Fang, T., Lu, N., Niu, G., and Sugiyama, M.
\newblock Rethinking importance weighting for deep learning under distribution
  shift.
\newblock In \emph{NeurIPS}, 2020{\natexlab{a}}.

\bibitem[Fang et~al.(2020{\natexlab{b}})Fang, Lu, Liu, Xuan, and
  Zhang]{fang2020open}
Fang, Z., Lu, J., Liu, F., Xuan, J., and Zhang, G.
\newblock Open set domain adaptation: Theoretical bound and algorithm.
\newblock \emph{IEEE transactions on neural networks and learning systems},
  2020{\natexlab{b}}.

\bibitem[Feinman et~al.(2017)Feinman, Curtin, Shintre, and
  Gardner]{feinman2017detecting}
Feinman, R., Curtin, R.~R., Shintre, S., and Gardner, A.~B.
\newblock Detecting adversarial samples from artifacts.
\newblock \emph{arXiv preprint arXiv:1703.00410}, 2017.

\bibitem[Fern{\'a}ndez et~al.(2008)Fern{\'a}ndez, Gamero, and
  Garcia]{fernandez2008test}
Fern{\'a}ndez, V.~A., Gamero, M.~J., and Garcia, J.~M.
\newblock A test for the two-sample problem based on empirical characteristic
  functions.
\newblock \emph{Computational statistics \& data analysis}, 52\penalty0
  (7):\penalty0 3730--3748, 2008.

\bibitem[Gao et~al.(2018)Gao, Xie, Xie, and Xu]{GaoX0X18}
Gao, R., Xie, L., Xie, Y., and Xu, H.
\newblock Robust hypothesis testing using wasserstein uncertainty sets.
\newblock In \emph{NeurIPS}, pp.\  7913--7923, 2018.

\bibitem[Gao et~al.(2021)Gao, Liu, Zhang, Han, Liu, Niu, and
  Sugiyama]{gao2021maximum}
Gao, R., Liu, F., Zhang, J., Han, B., Liu, T., Niu, G., and Sugiyama, M.
\newblock Maximum mean discrepancy test is aware of adversarial attacks.
\newblock In \emph{International Conference on Machine Learning}, pp.\
  3564--3575. PMLR, 2021.

\bibitem[Ghoshdastidar et~al.(2017)Ghoshdastidar, Gutzeit, Carpentier, and von
  Luxburg]{ghoshdastidar2017two}
Ghoshdastidar, D., Gutzeit, M., Carpentier, A., and von Luxburg, U.
\newblock Two-sample tests for large random graphs using network statistics.
\newblock In \emph{Conference on Learning Theory}, pp.\  954--977. PMLR, 2017.

\bibitem[Gong et~al.(2016)Gong, Zhang, Liu, Tao, Glymour, and
  Sch{\"o}lkopf]{gong2016domain}
Gong, M., Zhang, K., Liu, T., Tao, D., Glymour, C., and Sch{\"o}lkopf, B.
\newblock Domain adaptation with conditional transferable components.
\newblock In \emph{International conference on machine learning}, pp.\
  2839--2848. PMLR, 2016.

\bibitem[Goodfellow et~al.(2015)Goodfellow, Shlens, and
  Szegedy]{Goodfellow14_Adversarial_examples}
Goodfellow, I.~J., Shlens, J., and Szegedy, C.
\newblock Explaining and harnessing adversarial examples.
\newblock In \emph{ICLR}, 2015.

\bibitem[Gowal et~al.(2021)Gowal, Rebuffi, Wiles, Stimberg, Calian, and
  Mann]{gowal2021improving}
Gowal, S., Rebuffi, S.-A., Wiles, O., Stimberg, F., Calian, D.~A., and Mann,
  T.~A.
\newblock Improving robustness using generated data.
\newblock \emph{Advances in Neural Information Processing Systems}, 34, 2021.

\bibitem[Gretton et~al.(2009)Gretton, Fukumizu, Harchaoui, and
  Sriperumbudur]{gretton2009fast}
Gretton, A., Fukumizu, K., Harchaoui, Z., and Sriperumbudur, B.~K.
\newblock A fast, consistent kernel two-sample test.
\newblock In \emph{NIPS}, volume~23, pp.\  673--681, 2009.

\bibitem[Gretton et~al.(2012)Gretton, Borgwardt, Rasch, Sch{\"o}lkopf, and
  Smola]{gretton2012kernel}
Gretton, A., Borgwardt, K.~M., Rasch, M.~J., Sch{\"o}lkopf, B., and Smola, A.
\newblock A kernel two-sample test.
\newblock \emph{The Journal of Machine Learning Research}, 13\penalty0
  (1):\penalty0 723--773, 2012.

\bibitem[Grosse et~al.(2017)Grosse, Manoharan, Papernot, Backes, and
  McDaniel]{grosse2017statistical}
Grosse, K., Manoharan, P., Papernot, N., Backes, M., and McDaniel, P.
\newblock On the (statistical) detection of adversarial examples.
\newblock \emph{arXiv preprint arXiv:1702.06280}, 2017.

\bibitem[G{\"u}l \& Zoubir(2017)G{\"u}l and Zoubir]{gul2017minimax}
G{\"u}l, G. and Zoubir, A.~M.
\newblock Minimax robust hypothesis testing.
\newblock \emph{IEEE Transactions on Information Theory}, 63\penalty0
  (9):\penalty0 5572--5587, 2017.

\bibitem[Hein \& Andriushchenko(2017)Hein and
  Andriushchenko]{DBLP:conf/nips/HeinA17}
Hein, M. and Andriushchenko, M.
\newblock Formal guarantees on the robustness of a classifier against
  adversarial manipulation.
\newblock In \emph{NIPS}, pp.\  2263--2273, 2017.

\bibitem[Hendrycks et~al.(2021)Hendrycks, Zhao, Basart, Steinhardt, and
  Song]{hendrycks2021natural}
Hendrycks, D., Zhao, K., Basart, S., Steinhardt, J., and Song, D.
\newblock Natural adversarial examples.
\newblock In \emph{Proceedings of the IEEE/CVF Conference on Computer Vision
  and Pattern Recognition}, pp.\  15262--15271, 2021.

\bibitem[Huber(2004)]{huber2004robust}
Huber, P.~J.
\newblock \emph{Robust statistics}, volume 523.
\newblock John Wiley \& Sons, 2004.

\bibitem[Ilyas et~al.(2018)Ilyas, Engstrom, Athalye, and Lin]{ilyas2018black}
Ilyas, A., Engstrom, L., Athalye, A., and Lin, J.
\newblock Black-box adversarial attacks with limited queries and information.
\newblock In \emph{International Conference on Machine Learning}, pp.\
  2137--2146. PMLR, 2018.

\bibitem[Jitkrittum et~al.(2016)Jitkrittum, Szab{\'o}, Chwialkowski, and
  Gretton]{jitkrittum2016interpretable}
Jitkrittum, W., Szab{\'o}, Z., Chwialkowski, K.~P., and Gretton, A.
\newblock Interpretable distribution features with maximum testing power.
\newblock \emph{Advances in Neural Information Processing Systems},
  29:\penalty0 181--189, 2016.

\bibitem[Kanth~Nakka \& Salzmann(2021)Kanth~Nakka and
  Salzmann]{kanth2021learning}
Kanth~Nakka, K. and Salzmann, M.
\newblock Learning transferable adversarial perturbations.
\newblock In \emph{Thirty-Fifth Conference on Neural Information Processing
  Systems}, 2021.

\bibitem[Kim et~al.(2021)Kim, Lee, and Ro]{kim2021distilling}
Kim, J., Lee, B.-K., and Ro, Y.~M.
\newblock Distilling robust and non-robust features in adversarial examples by
  information bottleneck.
\newblock \emph{Advances in Neural Information Processing Systems}, 34, 2021.

\bibitem[Kingma \& Ba(2015)Kingma and Ba]{kingma2014adam}
Kingma, D.~P. and Ba, J.
\newblock Adam: A method for stochastic optimization.
\newblock In \emph{ICLR (Poster)}, 2015.
\newblock URL \url{http://arxiv.org/abs/1412.6980}.

\bibitem[Kirchler et~al.(2020)Kirchler, Khorasani, Kloft, and
  Lippert]{kirchler2020two}
Kirchler, M., Khorasani, S., Kloft, M., and Lippert, C.
\newblock Two-sample testing using deep learning.
\newblock In \emph{International Conference on Artificial Intelligence and
  Statistics}, pp.\  1387--1398. PMLR, 2020.

\bibitem[Krizhevsky(2009)]{krizhevsky2009learning_cifar10}
Krizhevsky, A.
\newblock Learning multiple layers of features from tiny images.
\newblock Technical report, 2009.

\bibitem[LeCun et~al.(1998)LeCun, Bottou, Bengio, and
  Haffner]{lecun1998gradient}
LeCun, Y., Bottou, L., Bengio, Y., and Haffner, P.
\newblock Gradient-based learning applied to document recognition.
\newblock \emph{Proceedings of the IEEE}, 86\penalty0 (11):\penalty0
  2278--2324, 1998.

\bibitem[Leucht \& Neumann(2013)Leucht and Neumann]{leucht2013dependent}
Leucht, A. and Neumann, M.~H.
\newblock Dependent wild bootstrap for degenerate u-and v-statistics.
\newblock \emph{Journal of Multivariate Analysis}, 117:\penalty0 257--280,
  2013.

\bibitem[Levy(2008)]{levy2008robust}
Levy, B.~C.
\newblock Robust hypothesis testing with a relative entropy tolerance.
\newblock \emph{IEEE Transactions on Information Theory}, 55\penalty0
  (1):\penalty0 413--421, 2008.

\bibitem[Li \& Wang(2018)Li and Wang]{li2018fully}
Li, S. and Wang, X.
\newblock Fully distributed sequential hypothesis testing: Algorithms and
  asymptotic analyses.
\newblock \emph{IEEE Transactions on Information Theory}, 64\penalty0
  (4):\penalty0 2742--2758, 2018.

\bibitem[Liu et~al.(2019)Liu, Lu, Han, Niu, Zhang, and
  Sugiyama]{liu2019butterfly}
Liu, F., Lu, J., Han, B., Niu, G., Zhang, G., and Sugiyama, M.
\newblock Butterfly: A panacea for all difficulties in wildly unsupervised
  domain adaptation.
\newblock In \emph{NeurIPS LTS Workshop}, 2019.

\bibitem[Liu et~al.(2020{\natexlab{a}})Liu, Xu, Lu, Zhang, Gretton, and
  Sutherland]{liu2020learning}
Liu, F., Xu, W., Lu, J., Zhang, G., Gretton, A., and Sutherland, D.~J.
\newblock Learning deep kernels for non-parametric two-sample tests.
\newblock In \emph{International Conference on Machine Learning}, pp.\
  6316--6326. PMLR, 2020{\natexlab{a}}.

\bibitem[Liu et~al.(2020{\natexlab{b}})Liu, Zhang, and Lu]{liu2020multi}
Liu, F., Zhang, G., and Lu, J.
\newblock Multi-source heterogeneous unsupervised domain adaptation via
  fuzzy-relation neural networks.
\newblock \emph{IEEE Transactions on Fuzzy Systems}, 2020{\natexlab{b}}.

\bibitem[Liu et~al.(2021)Liu, Xu, Lu, and Sutherland]{liu2021meta}
Liu, F., Xu, W., Lu, J., and Sutherland, D.~J.
\newblock Meta two-sample testing: Learning kernels for testing with limited
  data.
\newblock In Beygelzimer, A., Dauphin, Y., Liang, P., and Vaughan, J.~W.
  (eds.), \emph{Advances in Neural Information Processing Systems}, 2021.
\newblock URL \url{https://openreview.net/forum?id=EUlAerrk47Y}.

\bibitem[Lopez-Paz \& Oquab(2016)Lopez-Paz and Oquab]{lopez2016revisiting}
Lopez-Paz, D. and Oquab, M.
\newblock Revisiting classifier two-sample tests.
\newblock \emph{arXiv preprint arXiv:1610.06545}, 2016.

\bibitem[Madry et~al.(2018)Madry, Makelov, Schmidt, Tsipras, and
  Vladu]{Madry_adversarial_training}
Madry, A., Makelov, A., Schmidt, L., Tsipras, D., and Vladu, A.
\newblock Towards deep learning models resistant to adversarial attacks.
\newblock In \emph{ICLR}, 2018.

\bibitem[Mahmood et~al.(2021)Mahmood, Mahmood, and
  Van~Dijk]{mahmood2021robustness}
Mahmood, K., Mahmood, R., and Van~Dijk, M.
\newblock On the robustness of vision transformers to adversarial examples.
\newblock \emph{arXiv preprint arXiv:2104.02610}, 2021.

\bibitem[Metzen et~al.(2017)Metzen, Genewein, Fischer, and
  Bischoff]{metzen2017detecting}
Metzen, J.~H., Genewein, T., Fischer, V., and Bischoff, B.
\newblock On detecting adversarial perturbations.
\newblock \emph{arXiv preprint arXiv:1702.04267}, 2017.

\bibitem[Moosavi-Dezfooli et~al.(2016)Moosavi-Dezfooli, Fawzi, and
  Frossard]{moosavi2016deepfool}
Moosavi-Dezfooli, S.-M., Fawzi, A., and Frossard, P.
\newblock Deepfool: a simple and accurate method to fool deep neural networks.
\newblock In \emph{Proceedings of the IEEE conference on computer vision and
  pattern recognition}, pp.\  2574--2582, 2016.

\bibitem[Mopuri et~al.(2018)Mopuri, Ganeshan, and
  Babu]{mopuri2018generalizable}
Mopuri, K.~R., Ganeshan, A., and Babu, R.~V.
\newblock Generalizable data-free objective for crafting universal adversarial
  perturbations.
\newblock \emph{IEEE transactions on pattern analysis and machine
  intelligence}, 41\penalty0 (10):\penalty0 2452--2465, 2018.

\bibitem[Oneto et~al.(2020)Oneto, Donini, Luise, Ciliberto, Maurer, and
  Pontil]{oneto2020exploiting}
Oneto, L., Donini, M., Luise, G., Ciliberto, C., Maurer, A., and Pontil, M.
\newblock Exploiting mmd and sinkhorn divergences for fair and transferable
  representation learning.
\newblock In \emph{NeurIPS}, 2020.

\bibitem[Opitz \& Maclin(1999)Opitz and Maclin]{opitz1999popular}
Opitz, D. and Maclin, R.
\newblock Popular ensemble methods: An empirical study.
\newblock \emph{Journal of artificial intelligence research}, 11:\penalty0
  169--198, 1999.

\bibitem[Pang et~al.(2019)Pang, Xu, Du, Chen, and
  Zhu]{Pang_ICML_19_AT_Ensemble}
Pang, T., Xu, K., Du, C., Chen, N., and Zhu, J.
\newblock Improving adversarial robustness via promoting ensemble diversity.
\newblock In \emph{ICML}, 2019.

\bibitem[Pang et~al.(2021)Pang, Yang, Dong, Su, and Zhu]{pang2021bag}
Pang, T., Yang, X., Dong, Y., Su, H., and Zhu, J.
\newblock Bag of tricks for adversarial training.
\newblock \emph{ICLR}, 2021.

\bibitem[Pang et~al.(2022)Pang, Lin, Yang, Zhu, and Yan]{pang2022robustness}
Pang, T., Lin, M., Yang, X., Zhu, J., and Yan, S.
\newblock Robustness and accuracy could be reconcilable by (proper) definition.
\newblock \emph{arXiv preprint arXiv:2202.10103}, 2022.

\bibitem[Papernot et~al.(2016)Papernot, McDaniel, Sinha, and
  Wellman]{papernot2016towards}
Papernot, N., McDaniel, P., Sinha, A., and Wellman, M.
\newblock Towards the science of security and privacy in machine learning.
\newblock \emph{arXiv:1611.03814}, 2016.

\bibitem[Paszke et~al.(2019)Paszke, Gross, Massa, Lerer, Bradbury, Chanan,
  Killeen, Lin, Gimelshein, Antiga, et~al.]{paszke2019pytorch}
Paszke, A., Gross, S., Massa, F., Lerer, A., Bradbury, J., Chanan, G., Killeen,
  T., Lin, Z., Gimelshein, N., Antiga, L., et~al.
\newblock Pytorch: An imperative style, high-performance deep learning library.
\newblock \emph{Advances in neural information processing systems},
  32:\penalty0 8026--8037, 2019.

\bibitem[Peng \& Mine(2020)Peng and Mine]{peng2020robust}
Peng, S. and Mine, T.
\newblock A robust hierarchical graph convolutional network model for
  collaborative filtering.
\newblock \emph{arXiv preprint arXiv:2004.14734}, 2020.

\bibitem[Poggio \& Shelton(2002)Poggio and Shelton]{poggio2002mathematical}
Poggio, T. and Shelton, C.~R.
\newblock On the mathematical foundations of learning.
\newblock \emph{American Mathematical Society}, 39\penalty0 (1):\penalty0
  1--49, 2002.

\bibitem[Qin et~al.(2019)Qin, Martens, Gowal, Krishnan, Dvijotham, Fawzi, De,
  Stanforth, and Kohli]{local_linearilization}
Qin, C., Martens, J., Gowal, S., Krishnan, D., Dvijotham, K., Fawzi, A., De,
  S., Stanforth, R., and Kohli, P.
\newblock Adversarial robustness through local linearization.
\newblock In \emph{NeurIPS}, 2019.

\bibitem[Radford et~al.(2015)Radford, Metz, and
  Chintala]{radford2015unsupervised}
Radford, A., Metz, L., and Chintala, S.
\newblock Unsupervised representation learning with deep convolutional
  generative adversarial networks.
\newblock \emph{arXiv preprint arXiv:1511.06434}, 2015.

\bibitem[Rahmati et~al.(2020)Rahmati, Moosavi-Dezfooli, Frossard, and
  Dai]{rahmati2020geoda}
Rahmati, A., Moosavi-Dezfooli, S.-M., Frossard, P., and Dai, H.
\newblock Geoda: a geometric framework for black-box adversarial attacks.
\newblock In \emph{Proceedings of the IEEE/CVF Conference on Computer Vision
  and Pattern Recognition}, pp.\  8446--8455, 2020.

\bibitem[Rasch et~al.(2008)Rasch, Gretton, Murayama, Maass, and
  Logothetis]{rasch2008predicting}
Rasch, M., Gretton, A., Murayama, Y., Maass, W., and Logothetis, N.
\newblock Predicting spiking activity from local field potentials.
\newblock \emph{Journal of Neurophysiology}, 99:\penalty0 1461--1476, 2008.

\bibitem[Rebuffi et~al.(2021)Rebuffi, Gowal, Calian, Stimberg, Wiles, and
  Mann]{rebuffi2021data}
Rebuffi, S.-A., Gowal, S., Calian, D.~A., Stimberg, F., Wiles, O., and Mann,
  T.~A.
\newblock Data augmentation can improve robustness.
\newblock \emph{Advances in Neural Information Processing Systems}, 34, 2021.

\bibitem[Robey et~al.(2021)Robey, Chamon, Pappas, Hassani, and
  Ribeiro]{robey2021adversarial}
Robey, A., Chamon, L., Pappas, G., Hassani, H., and Ribeiro, A.
\newblock Adversarial robustness with semi-infinite constrained learning.
\newblock \emph{Advances in Neural Information Processing Systems}, 34, 2021.

\bibitem[Rokach(2010)]{rokach2010ensemble}
Rokach, L.
\newblock Ensemble-based classifiers.
\newblock \emph{Artificial intelligence review}, 33\penalty0 (1):\penalty0
  1--39, 2010.

\bibitem[Sarkar et~al.(2021)Sarkar, Sarkar, Gali, and
  Balasubramanian]{sarkar2021adversarial}
Sarkar, A., Sarkar, A., Gali, S., and Balasubramanian, V.~N.
\newblock Adversarial robustness without adversarial training: A teacher-guided
  curriculum learning approach.
\newblock In Beygelzimer, A., Dauphin, Y., Liang, P., and Vaughan, J.~W.
  (eds.), \emph{Advances in Neural Information Processing Systems}, 2021.
\newblock URL \url{https://openreview.net/forum?id=MqCzSKCQ1QB}.

\bibitem[Schmidt et~al.(2018)Schmidt, Santurkar, Tsipras, Talwar, and
  Madry]{schmidt2018adversarial_more_data}
Schmidt, L., Santurkar, S., Tsipras, D., Talwar, K., and Madry, A.
\newblock Adversarially robust generalization requires more data.
\newblock In \emph{NeurIPS}, 2018.

\bibitem[Sehwag et~al.(2020)Sehwag, Wang, Mittal, and Jana]{sehwag2020hydra}
Sehwag, V., Wang, S., Mittal, P., and Jana, S.
\newblock Hydra: Pruning adversarially robust neural networks.
\newblock \emph{NeurIPS}, 2020.

\bibitem[Serfling(2009)]{serfling2009approximation}
Serfling, R.~J.
\newblock \emph{Approximation theorems of mathematical statistics}, volume 162.
\newblock John Wiley \& Sons, 2009.

\bibitem[Shafahi et~al.(2019)Shafahi, Najibi, Ghiasi, Xu, Dickerson, Studer,
  Davis, Taylor, and Goldstein]{Ali_NIPS19_adversarial_training_for_free}
Shafahi, A., Najibi, M., Ghiasi, M.~A., Xu, Z., Dickerson, J., Studer, C.,
  Davis, L.~S., Taylor, G., and Goldstein, T.
\newblock Adversarial training for free!
\newblock In \emph{NeurIPS}, 2019.

\bibitem[Sitawarin \& Wagner(2019)Sitawarin and
  Wagner]{sitawarin2019robustness}
Sitawarin, C. and Wagner, D.
\newblock On the robustness of deep k-nearest neighbors.
\newblock In \emph{2019 IEEE Security and Privacy Workshops (SPW)}, pp.\  1--7.
  IEEE, 2019.

\bibitem[Song et~al.(2019)Song, He, Wang, and Hopcroft]{song2018improving}
Song, C., He, K., Wang, L., and Hopcroft, J.~E.
\newblock Improving the generalization of adversarial training with domain
  adaptation.
\newblock In \emph{International Conference on Learning Representations}, 2019.
\newblock URL \url{https://openreview.net/forum?id=SyfIfnC5Ym}.

\bibitem[Sriramanan et~al.(2020)Sriramanan, Addepalli, Baburaj, and
  Babu]{sriramanan2020guided}
Sriramanan, G., Addepalli, S., Baburaj, A., and Babu, R.~V.
\newblock Guided adversarial attack for evaluating and enhancing adversarial
  defenses.
\newblock \emph{arXiv preprint arXiv:2011.14969}, 2020.

\bibitem[Sriramanan et~al.(2021)Sriramanan, Addepalli, Baburaj,
  et~al.]{sriramanan2021towards}
Sriramanan, G., Addepalli, S., Baburaj, A., et~al.
\newblock Towards efficient and effective adversarial training.
\newblock \emph{Advances in Neural Information Processing Systems}, 34, 2021.

\bibitem[Stojanov et~al.(2019)Stojanov, Gong, Carbonell, and
  Zhang]{stojanov2019data}
Stojanov, P., Gong, M., Carbonell, J., and Zhang, K.
\newblock Data-driven approach to multiple-source domain adaptation.
\newblock In \emph{The 22nd International Conference on Artificial Intelligence
  and Statistics}, pp.\  3487--3496. PMLR, 2019.

\bibitem[Sugiyama et~al.(2011)Sugiyama, Suzuki, Itoh, Kanamori, and
  Kimura]{sugiyama2011least}
Sugiyama, M., Suzuki, T., Itoh, Y., Kanamori, T., and Kimura, M.
\newblock Least-squares two-sample test.
\newblock \emph{Neural networks}, 24\penalty0 (7):\penalty0 735--751, 2011.

\bibitem[Sutherland et~al.(2017)Sutherland, Tung, Strathmann, De, Ramdas,
  Smola, and Gretton]{sutherland2016generative}
Sutherland, D.~J., Tung, H.-Y., Strathmann, H., De, S., Ramdas, A., Smola,
  A.~J., and Gretton, A.
\newblock Generative models and model criticism via optimized maximum mean
  discrepancy.
\newblock In \emph{ICLR}, 2017.

\bibitem[Szegedy et~al.(2014)Szegedy, Zaremba, Sutskever, Bruna, Erhan,
  Goodfellow, and Fergus]{szegedy}
Szegedy, C., Zaremba, W., Sutskever, I., Bruna, J., Erhan, D., Goodfellow, I.,
  and Fergus, R.
\newblock Intriguing properties of neural networks.
\newblock In \emph{ICLR}, 2014.

\bibitem[Tram{\`{e}}r et~al.(2018)Tram{\`{e}}r, Kurakin, Papernot, Goodfellow,
  Boneh, and McDaniel]{Tramer_iclr_18}
Tram{\`{e}}r, F., Kurakin, A., Papernot, N., Goodfellow, I.~J., Boneh, D., and
  McDaniel, P.~D.
\newblock Ensemble adversarial training: Attacks and defenses.
\newblock In \emph{ICLR}, 2018.

\bibitem[Wang et~al.(2021)Wang, Liu, Han, Liu, Gong, Niu, Zhou, and
  Sugiyama]{wang2021probabilistic}
Wang, Q., Liu, F., Han, B., Liu, T., Gong, C., Niu, G., Zhou, M., and Sugiyama,
  M.
\newblock Probabilistic margins for instance reweighting in adversarial
  training.
\newblock In Beygelzimer, A., Dauphin, Y., Liang, P., and Vaughan, J.~W.
  (eds.), \emph{Advances in Neural Information Processing Systems}, 2021.
\newblock URL \url{https://openreview.net/forum?id=rg8gNkvs3u}.

\bibitem[Wang et~al.(2018)Wang, Jha, and Chaudhuri]{wang2018analyzing}
Wang, Y., Jha, S., and Chaudhuri, K.
\newblock Analyzing the robustness of nearest neighbors to adversarial
  examples.
\newblock In \emph{International Conference on Machine Learning}, pp.\
  5133--5142. PMLR, 2018.

\bibitem[Wang et~al.(2019)Wang, Ma, Bailey, Yi, Zhou, and
  Gu]{Wang_Xingjun_MA_FOSC_DAT}
Wang, Y., Ma, X., Bailey, J., Yi, J., Zhou, B., and Gu, Q.
\newblock On the convergence and robustness of adversarial training.
\newblock In \emph{ICML}, 2019.

\bibitem[Wang et~al.(2020)Wang, Zou, Yi, Bailey, Ma, and
  Gu]{wang2020improving_MART}
Wang, Y., Zou, D., Yi, J., Bailey, J., Ma, X., and Gu, Q.
\newblock Improving adversarial robustness requires revisiting misclassified
  examples.
\newblock In \emph{ICLR}, 2020.

\bibitem[Wong \& Kolter(2018)Wong and
  Kolter]{Eric_Wong_provable_defence_convex_polytope}
Wong, E. and Kolter, J.~Z.
\newblock Provable defenses against adversarial examples via the convex outer
  adversarial polytope.
\newblock In \emph{ICML}, 2018.

\bibitem[Wong et~al.(2019)Wong, Schmidt, and Kolter]{wong2019wasserstein}
Wong, E., Schmidt, F., and Kolter, Z.
\newblock Wasserstein adversarial examples via projected sinkhorn iterations.
\newblock In \emph{International Conference on Machine Learning}, pp.\
  6808--6817. PMLR, 2019.

\bibitem[Wong et~al.(2020)Wong, Rice, and Kolter]{wong2020fast_zico_kolter}
Wong, E., Rice, L., and Kolter, J.~Z.
\newblock Fast is better than free: Revisiting adversarial training.
\newblock In \emph{ICLR}, 2020.

\bibitem[Wu et~al.(2020{\natexlab{a}})Wu, Wang, Xia, Bailey, and
  Ma]{Wu2020Skip}
Wu, D., Wang, Y., Xia, S.-T., Bailey, J., and Ma, X.
\newblock Skip connections matter: On the transferability of adversarial
  examples generated with resnets.
\newblock In \emph{International Conference on Learning Representations},
  2020{\natexlab{a}}.

\bibitem[Wu et~al.(2020{\natexlab{b}})Wu, Xia, and Wang]{wu2020adversarial}
Wu, D., Xia, S.-T., and Wang, Y.
\newblock Adversarial weight perturbation helps robust generalization.
\newblock \emph{NeurIPS}, 33, 2020{\natexlab{b}}.

\bibitem[Xiao et~al.(2018)Xiao, Zhu, Li, He, Liu, and Song]{xiao2018spatially}
Xiao, C., Zhu, J.-Y., Li, B., He, W., Liu, M., and Song, D.
\newblock Spatially transformed adversarial examples.
\newblock In \emph{International Conference on Learning Representations}, 2018.
\newblock URL \url{https://openreview.net/forum?id=HyydRMZC-}.

\bibitem[Xie et~al.(2017)Xie, Wang, Zhang, Zhou, Xie, and
  Yuille]{xie2017adversarial}
Xie, C., Wang, J., Zhang, Z., Zhou, Y., Xie, L., and Yuille, A.
\newblock Adversarial examples for semantic segmentation and object detection.
\newblock In \emph{Proceedings of the IEEE International Conference on Computer
  Vision}, pp.\  1369--1378, 2017.

\bibitem[Xie et~al.(2021)Xie, Gao, and Xie]{xie2021robust}
Xie, L., Gao, R., and Xie, Y.
\newblock Robust hypothesis testing with wasserstein uncertainty sets.
\newblock \emph{arXiv preprint arXiv:2105.14348}, 2021.

\bibitem[Yan et~al.(2018)Yan, Guo, and Zhang]{YanGZ18}
Yan, Z., Guo, Y., and Zhang, C.
\newblock Deep defense: Training dnns with improved adversarial robustness.
\newblock In \emph{NeurIPS}, pp.\  417--426, 2018.

\bibitem[Yan et~al.(2020)Yan, Guo, Liang, and Zhang]{yan2020policy}
Yan, Z., Guo, Y., Liang, J., and Zhang, C.
\newblock Policy-driven attack: Learning to query for hard-label black-box
  adversarial examples.
\newblock In \emph{International Conference on Learning Representations}, 2020.

\bibitem[Yang et~al.(2019)Yang, Rashtchian, Wang, and
  Chaudhuri]{yang2019adversarial}
Yang, Y., Rashtchian, C., Wang, Y., and Chaudhuri, K.
\newblock Adversarial examples for non-parametric methods: Attacks, defenses
  and large sample limits.
\newblock \emph{arXiv preprint arXiv:1906.03310}, 2019.

\bibitem[Yang et~al.(2020{\natexlab{a}})Yang, Rashtchian, Zhang, Salakhutdinov,
  and Chaudhuri]{yang2020closer}
Yang, Y., Rashtchian, C., Zhang, H., Salakhutdinov, R.~R., and Chaudhuri, K.
\newblock A closer look at accuracy vs. robustness.
\newblock In \emph{NeurIPS}, 2020{\natexlab{a}}.

\bibitem[Yang et~al.(2020{\natexlab{b}})Yang, Rashtchian, Wang, and
  Chaudhuri]{yang2020robustness}
Yang, Y.-Y., Rashtchian, C., Wang, Y., and Chaudhuri, K.
\newblock Robustness for non-parametric classification: A generic attack and
  defense.
\newblock In \emph{International Conference on Artificial Intelligence and
  Statistics}, pp.\  941--951. PMLR, 2020{\natexlab{b}}.

\bibitem[Yao et~al.(2021)Yao, Bielik, TSANKOV, and Vechev]{yao2021automated}
Yao, C., Bielik, P., TSANKOV, P., and Vechev, M.
\newblock Automated discovery of adaptive attacks on adversarial defenses.
\newblock In Beygelzimer, A., Dauphin, Y., Liang, P., and Vaughan, J.~W.
  (eds.), \emph{Advances in Neural Information Processing Systems}, 2021.
\newblock URL \url{https://openreview.net/forum?id=nWz-Si-uTzt}.

\bibitem[Yoo \& Qi(2021)Yoo and Qi]{yoo2021towards}
Yoo, J.~Y. and Qi, Y.
\newblock Towards improving adversarial training of nlp models.
\newblock \emph{arXiv preprint arXiv:2109.00544}, 2021.

\bibitem[Yu et~al.(2021)Yu, Gao, and Xu]{yu2021lafeat}
Yu, Y., Gao, X., and Xu, C.-Z.
\newblock Lafeat: Piercing through adversarial defenses with latent features.
\newblock In \emph{Proceedings of the IEEE/CVF Conference on Computer Vision
  and Pattern Recognition}, pp.\  5735--5745, 2021.

\bibitem[Zhang et~al.(2019{\natexlab{a}})Zhang, Zhang, Lu, Zhu, and
  Dong]{Lu_yiping_NIPS19_yopo}
Zhang, D., Zhang, T., Lu, Y., Zhu, Z., and Dong, B.
\newblock You only propagate once: Accelerating adversarial training via
  maximal principle.
\newblock In \emph{NeurIPS}, 2019{\natexlab{a}}.

\bibitem[Zhang et~al.(2019{\natexlab{b}})Zhang, Yu, Jiao, Xing, Ghaoui, and
  Jordan]{Zhang_trades}
Zhang, H., Yu, Y., Jiao, J., Xing, E.~P., Ghaoui, L.~E., and Jordan, M.~I.
\newblock Theoretically principled trade-off between robustness and accuracy.
\newblock In \emph{ICML}, 2019{\natexlab{b}}.

\bibitem[Zhang et~al.(2020{\natexlab{a}})Zhang, Xu, Han, Niu, Cui, Sugiyama,
  and Kankanhalli]{zhang2020fat}
Zhang, J., Xu, X., Han, B., Niu, G., Cui, L., Sugiyama, M., and Kankanhalli, M.
\newblock Attacks which do not kill training make adversarial learning
  stronger.
\newblock In \emph{ICML}, 2020{\natexlab{a}}.

\bibitem[Zhang et~al.(2021)Zhang, Zhu, Niu, Han, Sugiyama, and
  Kankanhalli]{zhang2021geometryaware}
Zhang, J., Zhu, J., Niu, G., Han, B., Sugiyama, M., and Kankanhalli, M.
\newblock Geometry-aware instance-reweighted adversarial training.
\newblock In \emph{ICLR}, 2021.

\bibitem[Zhang et~al.(2020{\natexlab{b}})Zhang, Yamane, Lu, and
  Sugiyama]{zhang2020one}
Zhang, T., Yamane, I., Lu, N., and Sugiyama, M.
\newblock A one-step approach to covariate shift adaptation.
\newblock In \emph{Asian Conference on Machine Learning}, pp.\  65--80. PMLR,
  2020{\natexlab{b}}.

\bibitem[Zhang et~al.(2020{\natexlab{c}})Zhang, Liu, Fang, Yuan, Zhang, and
  Lu]{DBLP:conf/ijcai/ZhangLFY0020}
Zhang, Y., Liu, F., Fang, Z., Yuan, B., Zhang, G., and Lu, J.
\newblock Clarinet: A one-step approach towards budget-friendly unsupervised
  domain adaptation.
\newblock In \emph{IJCAI}, pp.\  2526--2532, 2020{\natexlab{c}}.
\newblock URL \url{https://doi.org/10.24963/ijcai.2020/350}.

\bibitem[Zheng et~al.(2019)Zheng, Chen, and Ren]{zheng2019distributionally}
Zheng, T., Chen, C., and Ren, K.
\newblock Distributionally adversarial attack.
\newblock In \emph{Proceedings of the AAAI Conference on Artificial
  Intelligence}, volume~33, pp.\  2253--2260, 2019.

\bibitem[Zhu et~al.(2020)Zhu, Cheng, Gan, Sun, Goldstein, and
  Liu]{Zhu2020FreeLB:}
Zhu, C., Cheng, Y., Gan, Z., Sun, S., Goldstein, T., and Liu, J.
\newblock Freelb: Enhanced adversarial training for natural language
  understanding.
\newblock In \emph{ICLR}, 2020.

\bibitem[Zou et~al.(2021)Zou, Frei, and Gu]{DBLP:conf/icml/ZouFG21}
Zou, D., Frei, S., and Gu, Q.
\newblock Provable robustness of adversarial training for learning halfspaces
  with noise.
\newblock In \emph{ICML}, pp.\  13002--13011, 2021.
\newblock URL \url{http://proceedings.mlr.press/v139/zou21a.html}.

\end{thebibliography}
\bibliographystyle{icml2022}

\clearpage
\newpage
\appendix
\onecolumn

\section{Notation Table}
\label{appendix:notation_table}

\begin{table}[h!]
\centering
\caption{A notation table in convenience for viewing.}
\label{tab:notation}
\begin{tabular}{cc }
\hline
Notation & Description  \\ \hline
$\cJ$ & The non-parametric TST \\
$\bJ$ & The set of non-parametric TSTs \\
$\cH_0$ & The null hypothesis \\
$\cH_1$ & The alternative hypothesis \\ 
$\alpha$ & The significance level \\
$r$ & The rejection threshold \\
$\mathrm{TP}$ & The measurement function for the test power \\
$\cD$ & The test statistic function \\
$\hat{\cF}$ & The test criterion function \\
$\hat{\bF}$ & The set of test criterion functions \\
$d$ & The dimensionality of data \\
$\cX$ & The data feature space $\subset \bR^{d}$ \\
$\bP$ & The Borel probability measure on $\cX$ \\
$\bQ$ & The Borel probability measure on $\cX$ \\
$\bP^m$ & The joint probability distribution $\bP^m = \overbrace{\bP \times \bP \times \ldots \times \bP}^m$\\
$\bQ^n$ & The joint probability distribution $\bQ^n = \overbrace{\bQ \times \bQ \times \ldots \times \bQ}^n$\\
$S_\bP$ & The set $S_{\bP} =\{x_i\}_{i=1}^{m} \sim \bP^m $ \\
$S_\bQ$ & The set $S_{\bQ} =\{ y_i \}_{i=1}^{n} \sim \bQ^n $ \\
$n_{\tr}$ & The number of training samples drawn from a particular distribution \\
$n_{\te}$ & The number of testing samples drawn from a particular distribution \\
$k$ & The kernel function \\
$\theta$ & The kernel parameter \\
$\kappa$ & The dimensionality of the kernel parameter \\
$\Theta$ & The set of kernel parameters \\
$\R_{\Theta}$ & A positive constant that bounds the kernel parameter $\theta \in \Theta$ \\
$s$ & A positive constant used in defining $\bar{\Theta}_s$ \\
$\bar{\Theta}_s$ & A set of kernel parameters $\bar{\Theta}_s = \{\theta \in \Theta \mid \sigma_{\theta}^2 \geq s^2 > 0 \}$ \\
$\nu$ & A constant that uniformly bounds the kernel function \\
$L_1$ & Lipschitz constant of the kernel function \\
$L_2$ & Lipschitz constant of the kernel function \\
$\lambda$ & A constant $\in (0,1)$ used in calculating $\hat{\sigma}_{\cH_1, \lambda}$ (in Eq.~\eqref{eq:test_criterion}) \\
$\tilde{S}_\bQ$ & The adversarial data corresponding to $S_\bQ$ \\
$\epsilon$ & The size of adversarial budget \\
$T$ & The maximum PGD step \\
$\rho$ & The step size \\
$w$ & The weight for the test criterion function \\
$\bW$ & The weight set \\
$\mathbb{C}$ & The checkpoint set \\
$E$ & The number of training epoch \\
$\eta$ & The learning rate of the optimizer \\
$f$ & A classifier that outputs classification probabilities \\
$\phi$ & A neural network \\
$\gamma$ & A learnable parameter in the deep kernel \\
$\sigma_{\phi}$ & A learnable parameter in the deep kernel \\
$G$ & The number of test locations \\
$\mathcal{V}$ & The set of test locations \\
\hline
\end{tabular}
\end{table}

\section{Theoretical Analysis}
\label{appendix:theory}

All the proofs are inspired by~\citet{liu2020learning}.

\subsection{Uniform Convergence Results}
\label{appendix:proof_mmd_shift}
These results, on the uniform convergence of $\widehat{\MMD}^2(S_\bP, \tilde{S}_\bQ;k_\theta)$ and $\hat{\sigma}^2_{\cH_1, \lambda}(S_{\bP}, \tilde{S}_{\bQ};k_{\theta})$, were used in the proof of Theorem~\ref{theory:tp_attack}.

\paragraph{Proposition 1 (Restated).}
Under Assumptions~\ref{assump:param} to~\ref{assump:lipschitz}, we use $n_{\tr}$ samples to train a kernel $k_{\theta}$ parameterized with $\theta$ and $n_{\te}$ samples to run a test of significance level $\alpha$. Given adversarial budget $\epsilon \geq 0$, the benign pair $(S_{\bP}, S_{\bQ})$ and the corresponding adversarial pair $(S_{\bP}, \tilde{S}_{\bQ})$ where $\tilde{S}_{\bQ} \in \epsball[S_{\bQ}]$,  with the probability at least $1 - \delta$, we have
\begin{align}
\sup_{\theta}|{\widehat{\MMD}}^{2}(S_{\bP}, \tilde{S}_{\bQ};k_{\theta}) - {\widehat{\MMD}}^{2}(S_{\bP}, S_{\bQ};k_{\theta}) | \leq \frac{8 L_2 \epsilon \sqrt{d}}{\sqrt{n_{\te}}}  \sqrt{2 \log \frac{2}{\delta} + 2\kappa\log (4\R_{\Theta} \sqrt{n_{\te}})} + \frac{8 L_1}{\sqrt{n_{\te}}}. \nonumber
\end{align}

\begin{proof}[Proof of Proposition~\ref{proposition:mmd_shift}]

We study the random error function
\begin{align}
    \Delta(\theta) = \widehat{\MMD}^2(S_\bP, \tilde{S}_\bQ;k_\theta) - \widehat{\MMD}^2(S_\bP, S_\bQ;k_\theta). \nonumber
\end{align}

First, we choose $P$ points $\{ \theta_i\}_{i=1}^{P}$ such that any $\theta_i \in \Theta$ and $\min_i\|\theta - \theta_i \| \leq q$; Assumption~\ref{assump:param} ensures this is possible with at most $P = (4\R_{\Theta}/q)^\kappa$ points~\cite{poggio2002mathematical}.

We define $\tilde{H}_{ij} =  k(\bmx_i,\bmx_j) + k(\tilde{\bmy}_i, \tilde{\bmy}_j) - k(\bmx_i, \tilde{\bmy}_j) - k(\bmx_j, \tilde{\bmy}_i)$ where $x_i, x_j \in S_\bP$ and $\tilde{y}_i, \tilde{y}_j \in \tilde{S}_{\bQ}$. Note that $\tilde{y}_i = y_i + \zeta_i$ for any $y_i \in S_\bQ$ and $\tilde{y}_i \in \tilde{S}_\bQ$ where $\zeta_i$ is an adversarial perturbation under an $\ell_\infty$-bound of size $\epsilon$. Correspondingly, ${\widehat{\MMD}}^{2}(S_{\bP}, \tilde{S}_{\bQ};k_{\theta}) = \frac{1}{n(n-1)}\sum_{i \neq j}\tilde{H}_{ij}$. Via Assumption~\ref{assump:lipschitz} we know that $| \tilde{H}_{ij} - H_{ij} | \leq 4 L_2 \epsilon \sqrt{d}$.

Because $\tilde{S}_\bQ \in \epsball[S_\bQ]$, it holds that $ |\tilde{H}_{ij} - H_{ij}| \rightarrow 0$ when $\epsilon \rightarrow 0$. Therefore, we have $\bE \Delta \rightarrow 0$. Recall that $\widehat{\MMD}^2(S_\bP, S_\bQ;k_\theta) = \frac{1}{n(n-1)}\sum_{i \neq j}H_{ij}$. 
If we replace $(x_1, y_1)$ with $({x_1}', {y_1}')$, we can obtain ${{\widehat{\MMD'}^2}} (S_\bP, S_\bQ;k_\theta) = \frac{1}{n(n-1)}\sum_{i \neq j}F_{ij}$ and ${{\widehat{\MMD'}^2}}(S_{\bP}, \tilde{S}_{\bQ};k_{\theta}) = \frac{1}{n(n-1)}\sum_{i \neq j}\tilde{F}_{ij}$, where $F$ (or $\tilde{F}$) agrees with $H$ (or $\tilde{H}$) except when $i$ or $j$ is 1. Then, we have
\begin{align}
    & |\widehat{\MMD}^2(S_\bP, \tilde{S}_\bQ;k_\theta) - \widehat{\MMD}^2(S_\bP, S_\bQ;k_\theta) - ({{\widehat{\MMD'}}^{2}} (S_\bP, \tilde{S}_\bQ;k_\theta) - {{\widehat{\MMD'}}^{2}} (S_\bP, S_\bQ;k_\theta))| \nonumber \\
    & \leq  \frac{1}{n(n-1)} | \sum_{i>1} ( \tilde{H}_{i1} - H_{i1} - (\tilde{F}_{i1} - F_{i1}) ) + \sum_{j>1} ( \tilde{H}_{1j} - H_{1j} - (\tilde{F}_{1j} - F_{1j}) ) |\nonumber \\
    & \leq  \frac{1}{n(n-1)} \bigg( \sum_{i>1}| \tilde{H}_{i1} - H_{i1} | + \sum_{i>1}|(\tilde{F}_{i1} - F_{i1})| + \sum_{j>1}|\tilde{H}_{1j} - H_{1j}| + \sum_{j>1}|\tilde{F}_{1j} - F_{1j}| \bigg) \nonumber \\
    & \leq \frac{16L_2 \epsilon \sqrt{d}}{n}. \nonumber
\end{align}
Using McDiarmid's inequality for each $\Delta(\theta_i)$ and a union bound, we then obtain that with probability at least $1-\delta$,
\begin{align}
    \max_{i \in \{1,2,\ldots, P \}}\Delta(\theta) \leq \frac{16L_2 \epsilon \sqrt{d}}{ \sqrt{2n}}\sqrt{\log \frac{2P}{\delta}} \leq \frac{8L_2 \epsilon \sqrt{d}}{\sqrt{n}}\sqrt{2\log \frac{2}{\delta} + 2 \kappa \log \frac{4\R_{\Theta}}{q}}. \nonumber
\end{align}

Via Assumption~\ref{assump:lipschitz}, for any two $\theta, \theta' \in \Theta$, we also have
\begin{gather}
    | \widehat{\MMD}^2(S_\bP, \tilde{S}_\bQ;k_\theta) - \widehat{\MMD}^2(S_\bP, \tilde{S}_\bQ;k_{\theta'}) | \leq \frac{1}{n(n-1)}\sum_{i \neq j} | \tilde{H}_{ij}^{(\theta)} - \tilde{H}_{ij}^{(\theta')} | \nonumber \\
    \leq \frac{1}{n(n-1)}\sum_{i \neq j} 4L_1 \| \theta - \theta'\| = 4L_1 \| \theta - \theta'\| \leq 4L_1q. \nonumber
\end{gather}
Similarly, $| \widehat{\MMD}^2(S_\bP, S_\bQ;k_\theta) - \widehat{\MMD}^2(S_\bP, S_\bQ;k_{\theta'}) | \leq 4L_1 q$.

Combining these two results, we know that with probability at least $1 - \delta$,
\begin{align}
    \sup_{\theta} |\Delta(\theta)| \leq  \max_{i \in \{1,2,\ldots, P \}}\Delta(\theta) + 8L_1q \leq 
    \frac{8L_2 \epsilon \sqrt{d}}{\sqrt{n}} \sqrt{2\log \frac{2}{\delta} + 2 \kappa \log \frac{4\R_{\Theta}}{q}} + 8L_1q. \nonumber
\end{align}

Since the adversary perturbs benign test pairs, we let $n = n_\te$ and $q = \frac{1}{\sqrt{n_\te}}$, thus yielding the desired results.
\end{proof}

\begin{proposition}
\label{proposition:sigma_shift}
Under Assumptions~\ref{assump:param} to~\ref{assump:lipschitz}, we use $n_{\tr}$ samples to train a kernel $k_{\theta}$ parameterized with $\theta$ and $n_{\te}$ samples to run a test of significance level $\alpha$. Given adversarial budget $\epsilon \geq 0$, the benign pair $(S_{\bP}, S_{\bQ})$ and the corresponding adversarial pair $(S_{\bP}, \tilde{S}_{\bQ})$ where $\tilde{S}_{\bQ} \in \epsball[S_{\bQ}]$,  with the probability at least $1 - \delta$, we have
    \begin{align}
    \sup_{\theta}|\hat{\sigma}^2_{\cH_1, \lambda}(S_{\bP}, \tilde{S}_{\bQ};k_{\theta}) - \hat{\sigma}^2_{\cH_1, \lambda}(S_{\bP}, S_{\bQ};k_{\theta}) | \leq \frac{1024 \nu L_2 \epsilon \sqrt{d}}{\sqrt{n_{\te}}}\sqrt{2\log \frac{2}{\delta} + 2 \kappa \log (4\R_{\Theta} \sqrt{n_\te}) } + \frac{ 512L_1 \nu}{ \sqrt{n_\te}}. \nonumber
    \end{align}
\end{proposition}

\begin{proof}[Proof of Propositon~\ref{proposition:sigma_shift}]
We study the random error function
\begin{align}
    \Delta(\theta) = \hat{\sigma}^2_{\cH_1, \lambda}(S_{\bP}, \tilde{S}_{\bQ};k_{\theta}) - \hat{\sigma}^2_{\cH_1, \lambda}(S_{\bP}, S_{\bQ};k_{\theta}). \nonumber
\end{align}

Note that $\hat{\sigma}^2_{\cH_1, \lambda}(S_{\bP}, S_{\bQ};k_{\theta}) = \frac{4}{n^3}\sum_{i=1}^{n} \bigg(\sum_{j=1}^{n}H_{ij} \bigg)^2 - \frac{4}{n^4}\bigg(\sum_{i=1}^{n}\sum_{j=1}^{n}H_{ij} \bigg)^2 + \lambda$, and $\hat{\sigma}^2_{\cH_1, \lambda}(S_{\bP}, \tilde{S}_{\bQ};k_{\theta}) = \frac{4}{n^3}\sum_{i=1}^{n} \bigg(\sum_{j=1}^{n}\tilde{H}_{ij} \bigg)^2 - \frac{4}{n^4}\bigg(\sum_{i=1}^{n}\sum_{j=1}^{n}\tilde{H}_{ij} \bigg)^2 + \lambda$.

Because $\tilde{S}_\bQ \in \epsball[S_\bQ]$, it holds that $ |\tilde{H}_{ij} - H_{ij}| \rightarrow 0$ when $\epsilon \rightarrow 0$. Therefore, we have $\bE \Delta \rightarrow 0$. 

If we replace $(x_1, y_1)$ with $({x_1}', {y_1}')$, we can obtain $\hat{\sigma'}^2_{\cH_1, \lambda}(S_{\bP}, S_{\bQ};k_{\theta}) = \frac{4}{n^3}\sum_{i=1}^{n} \bigg(\sum_{j=1}^{n}F_{ij} \bigg)^2 - \frac{4}{n^4}\bigg(\sum_{i=1}^{n}\sum_{j=1}^{n}F_{ij} \bigg)^2 + \lambda$ and $\hat{\sigma'}^2_{\cH_1, \lambda}(S_{\bP}, \tilde{S}_{\bQ};k_{\theta}) = \frac{4}{n^3}\sum_{i=1}^{n} \bigg(\sum_{j=1}^{n}\tilde{F}_{ij} \bigg)^2 - \frac{4}{n^4}\bigg(\sum_{i=1}^{n}\sum_{j=1}^{n}\tilde{F}_{ij} \bigg)^2 + \lambda$, where $F$ (or $\tilde{F}$) agrees with $H$ (or $\tilde{H}$) except when $i$ or $j$ is 1. Via Assumption~\ref{assump:bound}, we have $|H_{ij}| \leq 4 \nu$. Then, we have
\begin{align}
    & |\hat{\sigma}^2_{\cH_1, \lambda}(S_{\bP}, \tilde{S}_{\bQ};k_{\theta}) - \hat{\sigma}^2_{\cH_1, \lambda}(S_{\bP}, S_{\bQ};k_{\theta})- (\hat{\sigma'}^2_{\cH_1, \lambda}(S_{\bP}, \tilde{S}_{\bQ};k_{\theta}) - \hat{\sigma'}^2_{\cH_1, \lambda}(S_{\bP}, S_{\bQ};k_{\theta}))| \nonumber \\
    & \leq  \frac{4}{n^3} \bigg |\sum_{i=1}^n \bigg[ \bigg(\sum_{j=1}^{n}\tilde{H}_{ij} \bigg)^2 - \bigg(\sum_{j=1}^{n}H_{ij} \bigg)^2 - \bigg(\sum_{j=1}^{n}\tilde{F}_{ij} \bigg)^2 + \bigg(\sum_{j=1}^{n}F_{ij} \bigg)^2 \bigg] \bigg|  \nonumber \\
    & \quad + \frac{4}{n^4} \bigg|\bigg[\bigg(\sum_{i=1}^{n}\sum_{j=1}^{n}\tilde{H}_{ij} \bigg)^2 - \bigg(\sum_{i=1}^{n}\sum_{j=1}^{n}H_{ij} \bigg)^2 - \bigg(\sum_{i=1}^{n}\sum_{j=1}^{n}\tilde{F}_{ij} \bigg)^2 + \bigg(\sum_{i=1}^{n}\sum_{j=1}^{n}F_{ij} \bigg)^2  \bigg] \bigg| \nonumber \\
    & \leq \frac{4}{n^3} \bigg|\sum_{i=1}^n \bigg[\bigg(\sum_{j=1}^{n}\tilde{H}_{ij} - \sum_{j=1}^{n}\tilde{F}_{ij} \bigg)\bigg(\sum_{j=1}^{n}\tilde{H}_{ij} + \sum_{j=1}^{n}\tilde{F}_{ij} \bigg) - \bigg(\sum_{j=1}^{n}H_{ij} - \sum_{j=1}^{n}F_{ij} \bigg)\bigg(\sum_{j=1}^{n}H_{ij} + \sum_{j=1}^{n}F_{ij} \bigg) \bigg]  \bigg| \nonumber \\
    & \quad + \frac{4}{n^4} \bigg| \bigg(\sum_{ij}\tilde{H}_{ij} - \sum_{ij}\tilde{F}_{ij} \bigg)\bigg(\sum_{ij}\tilde{H}_{ij} + \sum_{ij}\tilde{F}_{ij} \bigg) - \bigg(\sum_{ij}H_{ij} - \sum_{ij}F_{ij} \bigg)\bigg(\sum_{ij}H_{ij} + \sum_{ij}F_{ij} \bigg) \bigg| \nonumber \\
    & \leq \frac{4}{n^3} \bigg| \bigg(\sum_{j=1}^{n}\tilde{H}_{1j} - \sum_{j=1}^{n}\tilde{F}_{1j} \bigg)\bigg(\sum_{j=1}^{n}\tilde{H}_{1j} + \sum_{j=1}^{n}\tilde{F}_{1j} \bigg) + \sum_{i > 1} \bigg(\tilde{H}_{i1} - \tilde{F}_{i1} \bigg)\bigg(\sum_{j=1}^{n}\tilde{H}_{ij} + \sum_{j=1}^{n}\tilde{F}_{ij} \bigg)  \nonumber \\
    & \quad - \bigg(\sum_{j=1}^{n}H_{1j} - \sum_{j=1}^{n}F_{1j} \bigg)\bigg(\sum_{j=1}^{n}H_{1j} + \sum_{j=1}^{n}F_{1j} \bigg) - \sum_{i > 1} \bigg(H_{i1} - F_{i1} \bigg)\bigg(\sum_{j=1}^{n}H_{ij} + \sum_{j=1}^{n}F_{ij} \bigg) \bigg| \nonumber \\
    & \quad + \frac{4}{n^4} \bigg| \sum_{ij}\tilde{H}_{ij} - \sum_{ij}\tilde{F}_{ij} \bigg| \cdot \bigg| \sum_{ij}\tilde{H}_{ij} + \sum_{ij}\tilde{F}_{ij} - \bigg( \sum_{ij}H_{ij} + \sum_{ij}F_{ij} \bigg) \bigg| \nonumber \\
    & \quad + \frac{4}{n^4} \bigg| \sum_{ij}H_{ij} + \sum_{ij}F_{ij} \bigg| \cdot \bigg| \sum_{ij}\tilde{H}_{ij} - \sum_{ij}\tilde{F}_{ij} - \bigg( \sum_{ij}H_{ij} - \sum_{ij}F_{ij} \bigg) \bigg| \nonumber \\
    & \leq \frac{4}{n^3} \bigg(\bigg|  \sum_{j=1}^{n}\tilde{H}_{1j} - \sum_{j=1}^{n}\tilde{F}_{1j} \bigg| \cdot \bigg| \sum_{j=1}^{n}\tilde{H}_{1j} + \sum_{j=1}^{n}\tilde{F}_{1j} - \bigg(\sum_{j=1}^{n}H_{1j} + \sum_{j=1}^{n}F_{1j} \bigg) \bigg| \nonumber \\
    & \quad + \bigg| \sum_{j=1}^{n}\tilde{H}_{1j} + \sum_{j=1}^{n}\tilde{F}_{1j} \bigg| \cdot \bigg| \sum_{j=1}^{n}\tilde{H}_{1j} - \sum_{j=1}^{n}\tilde{F}_{1j}  - \bigg( \sum_{j=1}^{n}H_{1j} - \sum_{j=1}^{n}F_{1j} \bigg) \bigg| \bigg) \nonumber \\
    &\quad + \frac{4}{n^3} \sum_{i > 1} \bigg( \bigg| \tilde{H}_{i1} - \tilde{F}_{i1} \bigg| \cdot \bigg| \sum_{j=1}^{n}\tilde{H}_{ij} + \sum_{j=1}^{n}\tilde{F}_{ij} - \bigg( \sum_{j=1}^{n}H_{ij} + \sum_{j=1}^{n}F_{ij} \bigg)  \bigg|  \nonumber \\
    & \quad + \bigg| \sum_{j=1}^{n}H_{ij} + \sum_{j=1}^{n}F_{ij} \bigg| \cdot \bigg|\tilde{H}_{i1} - \tilde{F}_{i1} -  \bigg(H_{i1} - F_{i1} \bigg) \bigg| \bigg) \nonumber \\
    & \quad + \frac{4}{n^4} \cdot 2(2n -1) \cdot 4\nu \cdot (n^2 \cdot 4 L_2 \epsilon \sqrt{d} + n^2 \cdot 4 L_2 \epsilon \sqrt{d} ) \nonumber \\
    & \quad + \frac{4}{n^4} \cdot (n^2 \cdot 4\nu + n^2 \cdot 4 \nu) \cdot ((2n-1) \cdot 4 L_2 \epsilon \sqrt{d} + (2n-1) \cdot 4 L_2 \epsilon \sqrt{d} ) \nonumber \\
    & \leq \frac{4}{n^3} \cdot (n \cdot 4 \nu + n \cdot 4 \nu) \cdot ( n \cdot 4L_2\epsilon \sqrt{d} + n \cdot 4L_2\epsilon \sqrt{d}) + \frac{4}{n^3}\cdot(n \cdot 4 \nu + n \cdot 4 \nu) \cdot (n \cdot4L_2\epsilon \sqrt{d} + n \cdot 4L_2\epsilon \sqrt{d}) \nonumber \\
    & \quad + \frac{4}{n^3} \cdot (n-1) \cdot (8 \nu \cdot ( n \cdot 4L_2\epsilon \sqrt{d}  + n \cdot 4L_2\epsilon \sqrt{d}) + (n \cdot 4 \nu + n \cdot 4 \nu) \cdot (4L_2\epsilon \sqrt{d}+ 4L_2\epsilon \sqrt{d} ) ) \nonumber \\
    & \quad + \frac{512(2n-1)}{n^2} \nu L_2 \epsilon \sqrt{d}  \nonumber \\
    &  \leq \frac{1024(2n-1)}{n^2}\nu L_2 \epsilon \sqrt{d} \leq \frac{2048\nu L_2\epsilon \sqrt{d}}{n}. \nonumber
\end{align}
Using McDiarmid's inequality for each $\Delta(\theta_i)$ and a union bound, we then obtain that with probability at least $1-\delta$,
\begin{align}
    \max_{i \in \{1,2,\ldots, P \}}\Delta(\theta) \leq \frac{2048\nu L_2 \epsilon \sqrt{d}}{ \sqrt{2n}}\sqrt{\log \frac{2P}{\delta}} \leq \frac{1024\nu L_2 \epsilon \sqrt{d}}{\sqrt{n}}\sqrt{2\log \frac{2}{\delta} + 2 \kappa \log \frac{4\R_{\Theta}}{q}}. \nonumber
\end{align}

According to Lemma 19 in~\citet{liu2020learning}, for any two $\theta, \theta' \in \Theta$, we have
\begin{gather}
    | \hat{\sigma}^2_{\cH_1, \lambda}(S_{\bP}, S_{\bQ};k_{\theta}) - \hat{\sigma}^2_{\cH_1, \lambda}(S_{\bP}, S_{\bQ};k_{\theta'}) |
    \leq 256L_1 \nu \| \theta - \theta'\| \leq 256L_1 \nu q. \nonumber
\end{gather}
Similarly, $| \hat{\sigma}^2_{\cH_1, \lambda}(S_{\bP}, \tilde{S}_{\bQ};k_{\theta}) - \hat{\sigma}^2_{\cH_1, \lambda}(S_{\bP}, \tilde{S}_{\bQ};k_{\theta'})| \leq 256 L_1 \nu q $.

Combining these two results, we know that with probability at least $1 - \delta$,
\begin{align}
    \sup_{\theta} |\Delta(\theta)| \leq  \max_{i \in \{1,2,\ldots, P \}}\Delta(\theta) + 512L_1 \nu q \leq
    \frac{1024 \nu L_2 \epsilon \sqrt{d}}{\sqrt{n}}\sqrt{2\log \frac{2}{\delta} + 2 \kappa \log \frac{4\R_{\Theta}}{q}} + 512L_1 \nu q. \nonumber
\end{align}

Since the adversary perturbs benign test pairs, we let $n = n_\te$ and $q = \frac{1}{\sqrt{n_\te}}$, thus yielding the desired results.
\end{proof}

\subsection{Proof of Lemma~\ref{lemma:adv_benign_tp}}
\label{appendix:proof_adv_benign_tp}

\paragraph{Lemma~\ref{lemma:adv_benign_tp} (Restated).}
Under Assumptions~\ref{assump:param} to~\ref{assump:lipschitz}, we use $n_{\tr}$ samples to train a kernel $k_{\theta}$ parameterized with $\theta$ and $n_{\te}$ samples to run a test of significance level $\alpha$. Given adversarial budget $\epsilon \geq 0$, the benign pair $(S_{\bP}, S_{\bQ})$ and the corresponding adversarial pair $(S_{\bP}, \tilde{S}_{\bQ})$ where $\tilde{S}_{\bQ} \in \epsball[S_{\bQ}]$,  with the probability at least $1 - \delta$, we have
 \begin{align}
     & \sup_{\theta \in \bar{\Theta}_s} |\hat{\cF}(S_{\bP}, \tilde{S}_\bQ;k_{\theta}) - \hat{\cF}(S_{\bP}, S_\bQ;k_{\theta})| \nonumber \\
     & \leq \frac{L_2 \epsilon \sqrt{d}}{\sqrt{n_\te}}\bigg[ \frac{8 }{s } + \frac{2048 \nu^2 }{s^3} \bigg]\sqrt{2\log \frac{2}{\delta} + 2 \kappa \log (4\R_{\Theta} \sqrt{n_\te}) } + \bigg[  \frac{8 L_1}{s \sqrt{n_\te}} + \frac{1024 L_1 \nu}{s^3 \sqrt{n_\te}}\bigg]
     := \tilde{\xi}, \nonumber 
 \end{align}
 and by treating $\nu$ as a constant, we have
 \begin{align}
     \sup_{\theta \in \bar{\Theta}_s } |\hat{\cF}(S_{\bP}, \tilde{S}_\bQ;k_{\theta}) - \hat{\cF}(S_{\bP}, S_\bQ;k_{\theta})| =
     \cO \bigg(\frac{\epsilon L_2 \sqrt{d \big(\log \frac{1}{\delta} + \kappa \log (\R_{\Theta} \sqrt{n_\te})\big) } + L_1 }{s \sqrt{n_\te}} \bigg). \nonumber
 \end{align}

\begin{proof}[Proof of Lemma~\ref{lemma:adv_benign_tp}] 
Using Proposition~\ref{proposition:mmd_shift} and~\ref{proposition:sigma_shift}, we have
    \begin{align}
    & \sup_{\theta \in \bar{\Theta}_s } |\hat{\cF}(S_{\bP}, \tilde{S}_\bQ;k_{\theta}) - \hat{\cF}(S_{\bP}, S_\bQ;k_{\theta})| \nonumber \\
    & = \sup_{\theta \in \bar{\Theta}_s } |\frac{{\widehat{\MMD}}^{2}(S_{\bP}, \tilde{S}_{\bQ};k_{\theta})}{\hat{\sigma}_{\cH_1, \lambda}(S_{\bP}, \tilde{S}_{\bQ};k_{\theta})}
    - \frac{{\widehat{\MMD}}^{2}(S_{\bP}, S_{\bQ};k_{\theta})}{\hat{\sigma}_{\cH_1, \lambda}(S_{\bP}, S_{\bQ};k_{\theta})}| \nonumber \\
    & \le \bigg| \frac{{\widehat{\MMD}}^{2}(S_{\bP}, \tilde{S}_{\bQ};k_{\theta})}{\hat{\sigma}_{\cH_1, \lambda}(S_{\bP}, \tilde{S}_{\bQ};k_{\theta})}
    - \frac{{\widehat{\MMD}}^{2}(S_{\bP}, S_{\bQ};k_{\theta})}{\hat{\sigma}_{\cH_1, \lambda}(S_{\bP}, \tilde{S}_{\bQ};k_{\theta})} \bigg| + \bigg|\frac{{\widehat{\MMD}}^{2}(S_{\bP}, S_{\bQ};k_{\theta})}{\hat{\sigma}_{\cH_1, \lambda}(S_{\bP}, \tilde{S}_{\bQ};k_{\theta})}
    - \frac{{\widehat{\MMD}}^{2}(S_{\bP}, S_{\bQ};k_{\theta})}{\hat{\sigma}_{\cH_1, \lambda}(S_{\bP}, S_{\bQ};k_{\theta})} \bigg| \nonumber \\
    & = \frac{1}{|\hat{\sigma}_{\cH_1, \lambda}(S_{\bP}, \tilde{S}_{\bQ};k_{\theta})|} \cdot |{\widehat{\MMD}}^{2}(S_{\bP}, \tilde{S}_{\bQ};k_{\theta}) - {\widehat{\MMD}}^{2}(S_{\bP}, S_{\bQ};k_{\theta})| \nonumber \\
    & \quad + \frac{|{\widehat{\MMD}}^{2}(S_{\bP}, S_{\bQ};k_{\theta})| \cdot |\hat{\sigma}^2_{\cH_1, \lambda}(S_{\bP}, \tilde{S}_{\bQ};k_{\theta}) - \hat{\sigma}^2_{\cH_1, \lambda}(S_{\bP}, S_{\bQ};k_{\theta})|}{ |\hat{\sigma}_{\cH_1, \lambda}(S_{\bP}, S_{\bQ};k_{\theta})| \cdot |\hat{\sigma}_{\cH_1, \lambda}(S_{\bP}, \tilde{S}_{\bQ};k_{\theta})| \cdot |(\hat{\sigma}_{\cH_1, \lambda}(S_{\bP}, \tilde{S}_{\bQ};k_{\theta}) + \hat{\sigma}_{\cH_1, \lambda}(S_{\bP}, S_{\bQ};k_{\theta}) )|}  \nonumber \\
    & \le \frac{1}{s} |\widehat{\MMD}^{2}(S_{\bP}, \tilde{S}_{\bQ};k_{\theta}) - \widehat{\MMD}^{2}(S_{\bP}, S_{\bQ};k_{\theta}) | + \frac{4\nu}{2s^3} |\hat{\sigma}^2_{\cH_1, \lambda}(S_{\bP}, \tilde{S}_{\bQ};k_{\theta}) - \hat{\sigma}^2_{\cH_1, \lambda}(S_{\bP}, S_{\bQ};k_{\theta})| \nonumber \\
    & = \frac{1}{s} (\frac{8 L_2 \epsilon \sqrt{d}}{\sqrt{n_{\te}}}  \sqrt{2 \log \frac{2}{\delta} + 2\kappa\log (4\R_{\Theta} \sqrt{n_{\te}})} + \frac{8 L_1}{\sqrt{n_{\te}}}) + \frac{2 \nu}{s^3} (\frac{1024 \nu L_2 \epsilon \sqrt{d}}{\sqrt{n_{\te}}}\sqrt{2\log \frac{2}{\delta} + 2 \kappa \log (4\R_{\Theta} \sqrt{n_\te}) } + \frac{ 512L_1 \nu}{ \sqrt{n_\te}}) \nonumber  \\
    & = \bigg[ \frac{8 L_2 \epsilon \sqrt{d}}{s \sqrt{n_\te}} + \frac{2048 \nu^2 L_2 \epsilon \sqrt{d}}{s^3 \sqrt{n_\te}}\bigg]\sqrt{2\log \frac{2}{\delta} + 2 \kappa \log (4\R_{\Theta} \sqrt{n_\te}) } + \bigg[  \frac{8 L_1}{s \sqrt{n_\te}} + \frac{1024 L_1 \nu}{s^3 \sqrt{n_\te}}\bigg]. \nonumber
\end{align}
\end{proof}

\subsection{Proof of Theorem~\ref{theory:tp_attack}}
\label{appendix:proof_tp_attack}

Before providing the proof of Theorem~\ref{theory:tp_attack}, we need the following lemma.
We let $\cF(k_{\theta})$ refer to $\cF(S_{\bP}, S_{\bQ}; k_{\theta})$, and analogously $\hat{\cF}(k_{\theta})$ refer to $\hat{\cF}(S_{\bP}, S_{\bQ}; k_{\theta})$, for simplicity.
\begin{lemma}[\citet{liu2020learning}]
\label{lemma:benign_tp}
  Under Assumptions~\ref{assump:param} to~\ref{assump:lipschitz}, we use $n_{\tr}$ samples to train a kernel $k_{\theta}$ parameterized with $\theta$ and $n_{\te}$ samples to run a test of significance level $\alpha$. 
  With probability at least $1 - \delta$, we have
 \begin{align}
     & \sup_{\theta \in \bar{\Theta}_s} |\hat{\cF}(S_\bP, S_\bQ; k_{\theta}) - \cF(S_\bP, S_\bQ; k_{\theta})| \nonumber \\
     & \leq \frac{2\nu}{s^3}\lambda + \frac{1}{\sqrt{n_\tr}}\bigg[ \frac{8\nu}{s } +  \frac{1792\nu}{s^2s} \bigg] \sqrt{2\log \frac{2}{\delta} + 2 \kappa \log (4\R_{\Theta}\sqrt{n_{\tr}})} + \bigg[\frac{8}{s\sqrt{n_{\tr}}} + \frac{2048 \nu^2}{\sqrt{n_{\tr}}s^2s} \bigg] L_1 + \frac{4608 \nu^3}{s^2 n_{\tr} s} := \xi. \nonumber 
 \end{align}
\end{lemma}

Then, we provide the proof of Theorem~\ref{theory:tp_attack}.

\paragraph{Theorem~\ref{theory:tp_attack} (Restated).}
In the setup of Proposition~\ref{proposition:mmd_shift}, given $\hat{\theta}_{n_{\tr}} = \arg\max_{\theta \in \bar{\Theta}_s} \hat{\cF}(k_{\theta})$, $r^{(n_{\te})}$ denoting the rejection threshold, $\cF^{*} = \sup_{\theta \in \bar{\Theta}_s} \cF(k_\theta)$, and constants $ C_1, C_2, C_3$ depending on $\nu, L_1 , \lambda, s, \R_{\Theta}$ and $\kappa$, with probability at least $1-\delta$, the test under adversarial attack has power
\begin{align}
    \Pr(n_{\te}{\widehat{\MMD}}^{2}(S_\bP, \tilde{S}_\bQ; k_{\hat{\theta}_{n_{\tr}}}) > r^{(n_{\te})})  \geq \Phi \bigg[ \sqrt{n_{\te}} \bigg( \cF^* - \frac{C_1}{\sqrt{n_{\tr}}}\sqrt{\log \frac{\sqrt{n_{\tr}}}{\delta}} -  \frac{C_2 L_2 \epsilon \sqrt{d}}{\sqrt{n_{\te}}} \sqrt{\log \frac{\sqrt{n_{\te}}}{\delta}} \bigg) -  C_3 \sqrt{\log \frac{1}{\alpha}} \bigg]. \nonumber
    \end{align}

\begin{proof}[Proof of Theorem~\ref{theory:tp_attack}]

Letting $\theta^* = \argmax \cF(k_{\theta})$, we know that $\hat{\cF}(S_\bP, S_\bQ; k_{\hat{\theta}_{n_{\tr}}}) \geq  \hat{\cF}(S_\bP, S_\bQ; k_{\theta^*})$ because $\hat{\theta}_{n_{\tr}}$ maximizes $\hat{\cF}$. 
Using Lemma~\ref{lemma:adv_benign_tp} and~\ref{lemma:benign_tp} 
, in the adversarial setting, we can obtain
\begin{align}
\label{eq:tc_gap}
    \cF(S_\bP, \tilde{S}_\bQ; k_{\hat{\theta}_{n_{\tr}}}) 
    & \geq \hat{\cF}(S_\bP, \tilde{S}_\bQ; k_{\hat{\theta}_{n_{\tr}}}) - \xi 
    \geq (\hat{\cF}(S_\bP, S_\bQ; k_{\hat{\theta}_{n_{\tr}}}) - \tilde{\xi}) - \xi 
    \geq \hat{\cF}(S_\bP, S_\bQ; k_{\theta^*}) - \xi -\tilde{\xi} \nonumber \\
    &\geq (\cF(S_\bP, S_\bQ; k_{\theta^*}) - \xi) - \xi - \tilde{\xi}
    = \cF^* - 2 \xi - \tilde{\xi}. 
\end{align}

Corollary 11 of \citet{gretton2012kernel} implies that $r^{(n_{\te})} \leq 4 \nu \sqrt{\log (\alpha^{-1}) n_{\te}}$ no matter the choice of $\theta$.
According to Theorem~\ref{theory:asymptotics} and Eq.~\eqref{eq:tc_gap}, with probability at least $1 - \delta$, the test in adversarial settings has power

\begin{align}
    &\Pr\bigg[ n_{\te} \widehat{\MMD}^2(S_\bP, \tilde{S}_\bQ;k_{\hat{\theta}_{n_\tr}}) > r^{(n_{\te})} \bigg] \nonumber \\
    & = \Pr \bigg[ n_{\tr} \frac{\widehat{\MMD}^2(S_\bP, \tilde{S}_\bQ;k_{\hat{\theta}_{n_\tr}}) - \MMD^2(S_\bP, \tilde{S}_\bQ; k_{\hat{\theta}_{n_\tr}})}{\sigma_{\cH_1}(S_\bP, \tilde{S}_\bQ; k_{\hat{\theta}_{n_\tr}})} > \frac{r^{(n_{\te})}}{\sqrt{n_\te} \sigma_{\cH_1}(S_\bP, \tilde{S}_\bQ; k_{\hat{\theta}_{n_\tr}})} - \frac{\sqrt{n_\te} \MMD^2(S_\bP, \tilde{S}_\bQ; k_{\hat{\theta}_{n_\tr}})}{\sigma_{\cH_1}(S_\bP, \tilde{S}_\bQ; k_{\hat{\theta}_{n_\tr}})} \bigg] \nonumber \\
    & \rightarrow \Phi \bigg[ \sqrt{n_\te} \cF(S_\bP, \tilde{S}_\bQ; k_{\hat{\theta}_{n_\te}}) - \frac{r^{(n_{\te})}}{\sqrt{n_\te}\sigma_{\cH_1}(S_\bP, \tilde{S}_\bQ; k_{\hat{\theta}_{n_\tr}})} \bigg] \nonumber \\
    & \geq \Phi \bigg[ \sqrt{n_\te} (\cF^* - 2\xi - \tilde{\xi}) - \frac{r^{(n_{\te})}}{s \sqrt{n_\te}} \bigg] \nonumber \\
    & \geq \Phi \bigg[ \sqrt{n_{\te}} \bigg( \cF^* -  \frac{C_1}{\sqrt{n_{\tr}}}\sqrt{\log \frac{\sqrt{n_{\tr}}}{\delta}} - \frac{C_2 L_2 \epsilon \sqrt{d}}{\sqrt{n_{\te}}} \sqrt{\log \frac{\sqrt{n_{\te}}}{\delta}} \bigg) - C_3 \sqrt{\log \frac{1}{\alpha}} \bigg], \nonumber
\end{align}
where $C_1, C_2, C_3$ are constants depending on $\nu, L_1, \kappa, \R_{\Theta}, \lambda$ and $s$.
\end{proof}

\section{Related Works}
\label{sec:related_work}

In this section, we discuss the differences between our work and the related studies. 
\paragraph{Two-sample tests.}
TST is a premier statistical method to judge whether two sets of data come from the same distribution. Classical TSTs such as \textit{t}-test and Kolmogorov-Smirnov test require strong assumptions on the distributions being studied and are only efficient when applied to one-dimensional data. Non-parametric TSTs, relaxing the distributional assumptions and being able to handling complex distributions, have been applied to a wide of real-world domains~\cite{gretton2009fast, sugiyama2011least, gretton2012kernel,sutherland2016generative,chen2017new, ghoshdastidar2017two, li2018fully, kirchler2020two, chwialkowski2015fast,jitkrittum2016interpretable,lopez2016revisiting,cheng2019classification,liu2020learning,liu2021meta}.
These tests have also allowed applications in various machine learning problems such as domain adaptation, covariate shift, label-noise learning, generative modeling, fairness and causal discovery~\cite{binkowski2018demystifying,zhang2020one,DBLP:conf/nips/FangL0S20,gong2016domain,fang2020open,liu2019butterfly,DBLP:conf/ijcai/ZhangLFY0020, liu2020multi,stojanov2019data,lopez2016revisiting, oneto2020exploiting}. However, people rarely doubt the reliability of non-parametric TSTs. In other words, adversarial robustness of non-parametric TSTs is barely studied. In this paper, we leverage our proposed adversarial attack to disclose the failure mode of non-parametric TSTs and propose an effective strategy to make TSTs reliable in analyzing critical data.

\paragraph{Robust hypothesis tests.}
Previous robust hypothesis tests are composite tests where the null and the alternative hypotheses include a family of distributions, to obtain the reliable estimation
of the underlying distributions when there exists outliers in training dataset. These robust tests introduce various uncertainty sets for the distributions under the null and the alternative hypotheses such as $\epsilon$-contamination sets~\cite{huber2004robust} and sets centered around the empirical distribution defined via Kullback-Leibler divergence~\cite{levy2008robust,gul2017minimax} or Wasserstein metric~\cite{GaoX0X18, xie2021robust}. In comparison, our study discloses a premier hypothesis testing method (i.e., non-parametric TSTs) is non-robust against adversarial attacks during the testing procedure. Further, we develop a novel defense---robust deep kernels for TSTs, to enhance adversarial robustness of non-parametric TSTs at the testing time.

\paragraph{Adversarial attacks and defenses.} There is a bunch of studies on adversarial attacks~\cite{szegedy,Goodfellow14_Adversarial_examples,moosavi2016deepfool,papernot2016towards,Carlini017_CW,chen2018ead,ilyas2018black,Athalye_ICML_18_Obfuscated_Gradients,cheng2018queryefficient, xiao2018spatially,zheng2019distributionally,wong2019wasserstein,mopuri2018generalizable,Alaifari19_iclr_deformation,sriramanan2020guided,Cheng2020Sign-OPT:,chen2020hopskipjumpattack,rahmati2020geoda,yan2020policy,croce2020reliable,wu2020adversarial, Wu2020Skip,andriushchenko2020square,croce2020robustbench, yu2021lafeat,yao2021automated,hendrycks2021natural,kanth2021learning} and defenses~\cite{Madry_adversarial_training,Cai_CAT,YanGZ18,Wang_Xingjun_MA_FOSC_DAT,song2018improving,Tramer_iclr_18,Eric_Wong_provable_defence_convex_polytope,Ali_NIPS19_adversarial_training_for_free,Pang_ICML_19_AT_Ensemble,carmon2019unlabeled,wang2020improving_MART,ding2020mma,wu2020adversarial,dong2020adversarially,wong2020fast_zico_kolter, sehwag2020hydra,Lu_yiping_NIPS19_yopo,local_linearilization,Zhang_trades,zhang2020fat,zhang2021geometryaware,sriramanan2020guided,sriramanan2021towards, robey2021adversarial, DBLP:conf/icml/ZouFG21, kim2021distilling,wang2021probabilistic,sarkar2021adversarial,pang2021bag,chen2021robust,erdemir2021adversarial,gowal2021improving,rebuffi2021data} in the parametric settings, especially focusing on DNNs. On the other hand, studies on robustness of non-parametric classifiers (e.g., nearest neighbors, decision trees, random forests and kernel classifiers) are gaining a growing attention~\cite{amsaleg2017vulnerability,DBLP:conf/nips/HeinA17,wang2018analyzing,DBLP:conf/icml/ChenZBH19,sitawarin2019robustness,yang2019adversarial,yang2020robustness,bhattacharjee2020non,bhattacharjee2021consistent} as well. In contrast, our study focuses on adversarial robustness of non-parametric TSTs, which belongs to the field of hypothesis test rather than classification problems.

\paragraph{Statistical adversarial data detection.} Non-parametric TSTs have been applied to judge if upcoming data contains adversarial data that is statistically different from benign data distribution~\cite{metzen2017detecting,feinman2017detecting,grosse2017statistical,gao2021maximum}. These works focus on utilizing statistical methods (e.g., TSTs) to distinguish adversarial data against DNNs from benign data. Compared to these works, our work investigates TST itself. We disclose the adversarial vulnerabilities of non-parametric TSTs through adversarial attacks and further propose effective defensive strategies to make non-parametric TSTs reliable.

\section{Non-Parametric Two-Sample Tests}
\label{appendix:tst_intro}
We provide an introduction to the typical non-parametric TSTs in this section.

\subsection{Test Statistics}
\paragraph{C2ST-S~\cite{lopez2016revisiting}} Classifier-based two-sample test (C2ST) utilizes a classifier
$f : \cX \rightarrow \bR$ that outputs the classification probabilities. C2ST trains $f$ via maximizing the classification accuracy, and then makes judgements on the test pairs. C2ST-S is based on the sign of classification probabilities. The test statistic of C2ST-S proposed in~\citet{lopez2016revisiting} is 
\begin{align}
\label{eq:ts_c2sts}
\cD^{(\rS)}(S_{\bP}, S_\bQ) = \frac{1}{2n} \sum_{x_i \in S_\bP} \mathbbm{1}(f(x_i)>0) + \frac{1}{2n} \sum_{y_i \in S_\bQ} \mathbbm{1}(f(y_i)<0).
\end{align} 
Further, ~\citet{liu2020learning} pointed out that $\cD^{(\rS)}(S_{\bP}, S_\bQ)$ is equivalent to $\widehat{\MMD}^2(S_{\bP}, S_\bQ;k^{(\rS)})$.

\paragraph{C2ST-L~\cite{cheng2019classification}}
C2ST-L utilizes the classification confidence given by $f$ instead of only accessing the sign of $f$'s output.
Letting $f : \cX \rightarrow \bR$ be a classifier that outputs classification probabilities, the test statistic of C2ST-L proposed in~\citet{cheng2019classification} is
\begin{align}
\label{eq:ts_s2stl}
\cD^{(\rL)}(S_{\bP}, S_\bQ) = \frac{1}{n} \sum_{x_i \in S_\bP} f(x_i) - \frac{1}{n} \sum_{y_i \in S_\bQ} f(y_i).
\end{align}
Similar to C2ST-S, ~\citet{liu2020learning} also pointed out that $\cD^{(\rL)}(S_{\bP}, S_\bQ)$ is equivalent to $\widehat{\MMD}^2(S_{\bP}, S_\bQ;k^{(\rL)})$.

\paragraph{ME~\cite{chwialkowski2015fast,jitkrittum2016interpretable}.}
Given a positive definite kernel $k: \cX \times \cX \rightarrow \bR$ and a set of $G$ test locations $\mathcal{V} = \{ v_i\}_{i=1}^G$, the test statistic of ME is 
\begin{align}
\label{eq:ts_ME}
\cD^{(\ME)}(S_\bP, S_\bQ) = n \bar{z}_n^{\top}S_n^{-1}\bar{z}_n,
\end{align}
where
$\bar{z}_n=\frac{1}{n}\sum_{i=1}^n z_i$, $S_n = \frac{1}{n-1}\sum_{i=1}^n (z_i - \bar{z}_n)(z_i - \bar{z}_n)^{\top}$, and $z_i = (k(\bmx_i, v_j) - k(\bmy_i, v_j))_{j=1}^G \in \bR^G$.

\paragraph{SCF~\cite{chwialkowski2015fast,jitkrittum2016interpretable}.}
Given a positive definite kernel $k: \cX \times \cX \rightarrow \bR$ and a set of $G$ test locations $\mathcal{V} = \{ v_i\}_{i=1}^G $, the test statistic of SCF is 
\begin{align}
\label{eq:ts_SCF}
\cD^{(\SCF)}(S_\bP, S_\bQ) = n \bar{z}_n^{\top}S_n^{-1}\bar{z}_n,
\end{align}
where
$\bar{z}_n=\frac{1}{n}\sum_{i=1}^n z_i$, $S_n = \frac{1}{n-1}\sum_{i=1}^n (z_i - \bar{z}_n)(z_i - \bar{z}_n)^{\top}$, 
and $z_i = [ \hat{h}(\bmx_i)\sin{(\bmx_i^\top v_j)} - \hat{h} (\bmy_i) \sin{(\bmy_i^\top v_j)},  \hat{h}(\bmx_i)\cos{(\bmx_i^\top v_j)} - \hat{h} (\bmy_i) \cos{(\bmy_i^\top v_j)}]_{j=1}^G$. $\hat{h}(\bmx) = \int_{\bR^d} \exp (-iux)l(u)d u$ 
is the Fourier transform of $l(\bmx)$, 
and $h: \bR^d \rightarrow \bR $ is an analytic translation-invariant kernel.

\subsection{Test Criterion}
\label{appendix:extra_test_criterion}
For C2ST-S and C2ST-L, \citet{lopez2016revisiting} and \citet{cheng2019classification} proposed to maximize $f$'s classification accuracy, but it cannot directly maximize the test power~\cite{liu2020learning}. In this paper, therefore, we take $\hat{\cF}^{(\rS)}(S_\bP,S_\bQ) = \frac{\widehat{\MMD}^2(S_\bP, S_\bQ; k^{(\rS)})}{\hat{\sigma}_{\cH_1, \lambda} (S_\bP, S_\bQ; k^{(\rS)})} $ and $\hat{\cF}^{(\rL)}(S_\bP,S_\bQ) = \frac{\widehat{\MMD}^2(S_\bP, S_\bQ; k^{(\rL)})}{\hat{\sigma}_{\cH_1, \lambda} (S_\bP, S_\bQ; k^{(\rL)})} $ as the test criterion for C2ST-S and C2ST-L, respectively. To make $\hat{\cF}^{(\rS)}$ differentiable, we modify the kernel for C2ST-S as follows:
\begin{align}
    k^{(\rS)}(x, y) = \frac{1}{16} ( \frac{f(x)}{|f(x)|} + 1) ( \frac{f(y)}{|f(y)|} + 1). 
\end{align}

For ME and SCF tests, ~\citet{chwialkowski2015fast} and~\citet{jitkrittum2016interpretable} theoretically pointed out that maximizing $\cD^{(\ME)}(S_\bP, S_\bQ)$ and $\cD^{(\SCF)}(S_\bP, S_\bQ)$ can maximize the test power of ME and SCF, respectively. Therefore, $\hat{\cF}^{(\ME)}(\cdot, \cdot) = \cD^{(\ME)}(\cdot, \cdot)$ and $\hat{\cF}^{(\SCF)}(\cdot, \cdot) = \cD^{(\SCF)}(\cdot, \cdot)$.

\section{Experimental Details and Results}

\subsection{Datasets}
\label{appendix:dataset_stats}
In this section, we introduce the distribution $\bP$ and $\bQ$ of each dataset.

\paragraph{Blob.} 
Blob is often used to validate two-sample test methods~\cite{gretton2012kernel,jitkrittum2016interpretable,sutherland2016generative}. We show the specifications of $\bP$ and $\bQ$ of Blob in Table~\ref{tab:syn_dataset}.

\paragraph{High-dimensional Gaussian mixture.}
High-dimensional Gaussian mixture (HDGM) was utilized as a benchmark dataset in~\citet{liu2020learning}. HDGM can be regarded as high-dimensional Blob which contains two modes with the same variance and different covariance. We show the specifications of $\bP$ and $\bQ$ of HDGM in Table~\ref{tab:syn_dataset}. 

We set $d = 10$ for experiments on HDGM in Section~\ref{sec:attack_power} and~\ref{sec:mmd-rod-tp}. In section~\ref{sec:ablation}, we conduct experiments on HDGM with different $d \in \{ 5,10, 15,20,25\}$. In practice, the scale of data from HDGM is roughly betwen $-4.37$ and $4.70$. The adversarial budget we set in the experiments ($\epsilon=0.05$) on HDGM is small enough.

\begin{table}[h!]
\centering
\caption{Specifications of $\bP$ and $\bQ$ of synthetic datasets. $\mu_1^b=[0,0], \mu_2^b=[0,1], \mu_3^b=[0,2], ..., \mu_8^b=[2,1], \mu_9^b=[2,2]$. $\mu_1^h= \mathbf{0}_{d}, \mu_2^h= 0.5 \times \mathbf{1}_{d}$ where $\mathbf{1}_d$ is an identity matrix with size $d$. $\Delta_i^b=-0.02-0.002 \times (i-1)$ if $i < 5$ and $\Delta_i^b = 0.02 + 0.002 \times (i-6)$ if $i>5$. if $i=5$, $\Delta_i^b=0.$ $\Delta_1^h$ and $\Delta_2^h$ are set to 0.5 and -0.5, respectively.}
\label{tab:syn_dataset}
\begin{tabular}{c|cc}
\hline
Datasets & $\bP$ & $\bQ$ \\ \hline
Blob & $\sum_{i=1}^9 \frac{1}{9} \cN (\mu_i^b, 0.03 \times I_2)$ & $\sum_{i=1}^9 \frac{1}{9} \cN \bigg(\mu_i^b, \begin{bmatrix} 0.03 & \Delta_i^b \\ \Delta_i^b & 0.03 \end{bmatrix} \bigg ) $ \\
HDGM & $\sum_{i=1}^2 \frac{1}{2} \cN(\mu_i^h, I_d)$ & $\sum_{i=1}^2 \frac{1}{2} \cN \bigg (\mu_i^h, \begin{bmatrix} 1 & \Delta_i^h & \mathbf{0}_{d-2} \\ \Delta_i^h & 1 & \mathbf{0}_{d-2} \\ \mathbf{0}_{d-2}^T & \mathbf{0}_{d-2}^T & I_{d-2} \end{bmatrix} \bigg )$ \\
\hline
\end{tabular}

\end{table}

\paragraph{Higgs.}
For the experiments on Higgs, we compare the jet $\Phi$-momenta distribution ($d=4$) of the background process, $\bP$, which lacks Higgs bosons, to the corresponding distribution $\bQ$ for the process that produces Higgs bosons, following~\citet{chwialkowski2015fast}. Higgs dataset can be downloaded from \href{https://archive.ics.uci. edu/ml/datasets/Higgs}{UCI Machine Learning Repository}. 
In practice, the scale of data from Higgs is betwen $-1.74$ and $1.74$. The adversarial budget we set in the experiments ($\epsilon=0.05$) on Higgs is small enough.

\paragraph{MNIST.}
For the experiments on MNIST, we compare true MNIST images drawn from MNIST dataset~\cite{lecun1998gradient} (regarded as the distribution $\bP$) to fake MNIST images generated from a pretrained deep convolutional generative adversarial network (DCGAN)~\cite{radford2015unsupervised} (regarded as the distribution $\bQ$). Samples drawn from $\bQ$ can be generated by implementing \href{https://github.com/eriklindernoren/ PyTorch-GAN/blob/master/implementations/dcgan/dcgan.py.}{dcgan.py}.

\paragraph{CIFAR-10.}
For the experiments on CIFAR-10, we compare samples drawn from the class ``cat'' (regarded as the distribution $\bP$) to samples drawn from the class ``dog'' (regarded as the distribution $\bQ$) in CIFAR-10 dataset~\cite{krizhevsky2009learning_cifar10}. CIFAR-10 dataset can be downloaded via PyTorch~\cite{paszke2019pytorch}.

\subsection{Training Settings}
\label{appendix:exp_detail}
 
We conduct all experiments on Python 3.8 (PyTorch 1.1) with NVIDIA RTX A50000 GPUs. We run MMD-D, MMD-G, C2ST-S, C2ST-L, ME and SCF using \hyperlink{https://github.com/fengliu90/DK-for-TST}{the GitHub code} provided by~\citet{liu2020learning} and implement MMD-RoD by ourselves. Following~\citet{lopez2016revisiting}, we use a deep neural network $f$ as the classifier in C2ST-S and C2ST-L, and train $f$ by minimizing cross-entropy loss. The neural network structure $\phi$ in MMD-D and MMD-RoD has the same architecture with feature extractor in $f$, i.e., $f = g \circ \phi$ where $g$ is composed of two fully-connected layers and outputs the classification probabilities. For MNIST and CIFAR-10, we normalize the raw data into the scale $[-1,1]$.

For Blob, HDGM and Higgs, $\phi$ is a five-layer fully-connected neural network. The number of neurons in hidden and output layers of $\phi$ are set to 50 for Blob, $3 \times d$ for HDGM and 20 for Higgs, where $d$ is the dimensionality of samples. For MNIST and CIFAR-10, $\phi$ is a convolutional neural network (CNN) that contains four convolutional layers and one fully-connected layer. The structure of the CNN exactly follows~\citet{liu2020learning}.

We use Adam optimizer~\cite{kingma2014adam} to optimize (1) parameters of $f$ in C2ST-S and C2ST-L, (2) parameters of $\phi$ in MMD-D and MMD-RoD and (3) kernel lengthscale in MMD-G. We set drop-out rate to zero when training C2ST-S, C2ST-L, MMD-D and MMD-RoD on all datasets. We set the number of training samples $n_{\tr}$ to 100 for Blob, 3, 000 for HDGM, 5, 000 for Higgs, 500 for MNIST and CIFAR-10. 

For ME and SCF, we follow~\cite{chwialkowski2015fast} and set $J=10$ for Higgs. For other datasets, we set $J=5$.

For C2ST-S and C2ST-L, we set batchsize to 128 for Blob, HDGM and Higgs, and 100 for MNIST and CIFAR-10. We set the number of training epochs to $9000 \times n^{te}/$batchsize for Blob, 1, 000 for HDGM and Higgs, 2, 000 for MNIST and CIFAR-10. We set learning rate to 0.001 for Blob, HDGM and Higgs, and 0.0002 for MNIST and CIFAR-10.

For MMD-D, we use full batch (i.e., all samples) to train MMD-D and MMD-RoD for Blob, HDGM and Higgs. We use mini-batch (batchsize is 100) to train MMD-D and MMD-RoD for MNIST and CIFAR-10. We set the number of training epochs to 2, 000 for Blob, HDGM, Higgs and MNIST, and 1, 000 for CIFAR-10. We set learning rate to 0.0005 for Blob and Higgs, 0.00001 for HDGM, 0.001 for MNIST and 0.0002 for CIFAR-10.

For MMD-RoD, we keep $\epsilon$ for each dataset same as that in Table~\ref{tab:attack_power} and set $T$ to 1 for all datasets. We set learning rate to 0.0005 for MNIST. Other training settings of MMD-RoD keep same as that of MMD-D.

\subsection{Testing Procedure}
\label{appendix:testing}

We use permutation test to compute p-values of MMD-D, MMD-G, C2ST-S, C2ST-L and MMD-RoD. We set $\alpha$ to 0.05 and the iteration number of permutation test to 100 for all experiments. In addition, we utilize the wild bootstrap process~\cite{chwialkowski2014wild} to resample the value of MMD for MMD-D, MMD-G and MMD-RoD since the adversarial data are probably not IID. The wild bootstrap can ensure that we obtain correct p-values in non-IID/IID scenarios~\cite{chwialkowski2014wild}.

\paragraph{Wild bootstrap process.} Following~\citet{leucht2013dependent} and~\citet{chwialkowski2014wild}, we utilize the following wild bootstrap process:
\begin{align}
    \label{eq:wild_bootstrap}
    W_t = e^{-1/l} W_{t-1} + \sqrt{1 - e^{-2/l}} \tau_t,
\end{align}
where $W_0, \tau_0, ..., \tau_t$ are independent standard normal random variables. In all experiments, we set $l = 0.5$.

We summarize the permutation test with wild bootstrap process for non-parametric TSTs based on MMD in Algorithm~\ref{alg:testing}. 

\begin{algorithm}[h!]
  \caption{Testing with $k_{\theta}$ on $S_\bP$ and $S_\bQ$ }
  \label{alg:testing}
\begin{algorithmic}[1]
  \STATE {\bfseries Input:} input pair $(S_{\bP}, S_{\bQ})$, kernel $k_{\theta}$ parameterized with $\theta$, iteration number of permutation test $n_{\mathrm{perm}}$
  \STATE {\bfseries Output:} $est$, p-value: $\frac{1}{n_{\mathrm{perm}}} \sum_{i=1}^{n_{\mathrm{perm}}} \mathbbm{1} (perm_i > est)$
  \STATE $est \leftarrow \widehat{\MMD}^2(S_{\bP}, S_{\bQ}; k_{\theta})$
  \FOR{$i=1$ \textbf{to} $n_{\mathrm{perm}}$}
  \STATE Generate $\{W^{\bP}_i\}_{i=1}^{n_{\te}}$ and $\{W^{\bQ}_i\}_{i=1}^{n_{\te}}$ using Eq.~\eqref{eq:wild_bootstrap}
  \STATE $\{\tilde{W}^{\bP}_i\}_{i=1}^{n_{\te}} \leftarrow \{W^{\bP}_i\}_{i=1}^{n_{\te}} - \frac{1}{n_{\te}} \sum_{i=1}^{n_{\te}} W_i^{\bP}$
  \STATE $\{\tilde{W}^{\bQ}_i\}_{i=1}^{n_{\te}} \leftarrow \{W^{\bQ}_i\}_{i=1}^{n_{\te}} - \frac{1}{n_{\te}} \sum_{i=1}^{n_{\te}} W_i^{\bQ}$
  \STATE $perm_i \leftarrow \frac{1}{n_{\te}(n_{\te}-1)}\sum_{i,j} H_{ij} \tilde{W}_i^{\bP} \tilde{W}_j^{\bQ}$
  \ENDFOR
\end{algorithmic}
\end{algorithm}

\subsection{Weight Set Configurations}
\label{appendix:attack_config}
Observed from the lower right panel of Figure~\ref{fig:ablation}, we empirically find that an appropriate weight set is critical to the performance of EA. We finetune the weight set by increasing the weight of the TST that is difficult to be successfully fooled. Table~\ref{tab:attack_weight} summarizes the manually-finetuned weight of MMD-D, MMD-G, C2ST-S, C2ST-L, ME and SCF for each dataset.

\begin{table}[h!]
\centering
\caption{The manually-finetuned weight set of EA for each dataset.}
\label{tab:attack_weight}
\begin{tabular}{c|c}
\hline
Datasets  & $\bW$   \\ \hline
Blob & $\{\frac{5}{29}, \frac{1}{29}, \frac{1}{29}, \frac{20}{29}, \frac{1}{29}, \frac{1}{29}\}$ \\
HDGM & $\{\frac{25}{79}, \frac{1}{79}, \frac{1}{79}, \frac{50}{79}, \frac{1}{79}, \frac{1}{79} \}$\\
Higgs & $\{\frac{3}{98}, \frac{45}{98}, \frac{4}{98}, \frac{3}{98}, \frac{40}{98}, \frac{3}{98} \}$ \\
MNIST & $\{\frac{1}{109}, \frac{45}{109}, \frac{1}{109}, \frac{1}{109}, \frac{60}{109}, \frac{1}{109}\}$ \\
CIFAR-10 & $\{\frac{1}{80}, \frac{50}{80}, \frac{4}{80}, \frac{4}{80}, \frac{20}{80}, \frac{1}{80} \}$ \\ \hline
\end{tabular}
\end{table}

\subsection{Type $\bigromanone$ Errors}
\label{appendix:type1}

The Type $\bigromanone$ error of a TST measures the probability of rejecting $\cH_0$ when $\cH_0$ is true.
If the Type $\bigromanone$ error was much higher than $\alpha$, this TST would always reject the null hypothesis, which invalidates this TST~\cite{chwialkowski2014wild}. Therefore, a reasonable Type $\bigromanone$ error of a TST should not be much higher than $\alpha$.

\paragraph{Type $\bigromanone$ Errors of six typical non-parametric TSTs.} 
We report the Type $\bigromanone$ error of typical non-parametric TSTs on each dataset in Table~\ref{tab:type1}. As for the experimental configurations, the only difference from settings in Section~\ref{sec:attack_power} is that the training pairs and test pairs are composed of samples drawn from the same distribution $\bP$. 
Table~\ref{tab:type1} shows that these six typical non-parametric TSTs have reasonable Type $\bigromanone$ errors in benign settings.

\begin{table}[h!]
\centering
\caption{We report the Type $\bigromanone$ error of six typical non-parametric TSTs ($\alpha=0.05$) on five benchmark datasets.}
\label{tab:type1}
\begin{tabular}{cccccccc}
\hline
Datasets  & $n_{\te}$ & MMD-D & MMD-G & C2ST-S & C2ST-L & ME & SCF   \\ \hline
Blob & 100 & 0.056\scriptsize{$\pm$0.000} & 0.056\scriptsize{$\pm$0.000} & 0.049\scriptsize{$\pm$0.000} & 0.051\scriptsize{$\pm$0.000} & 0.051\scriptsize{$\pm$0.000} & 0.042\scriptsize{$\pm$0.000} \\
HDGM & 3000 & 0.057\scriptsize{$\pm$0.000} & 0.048\scriptsize{$\pm$0.000} & 0.056\scriptsize{$\pm$0.000}  & 0.040\scriptsize{$\pm$0.000} & 0.050\scriptsize{$\pm$0.000} & 0.041\scriptsize{$\pm$0.000}\\
Higgs & 5000 & 0.058\scriptsize{$\pm$0.000} & 0.043\scriptsize{$\pm$0.000} & 0.040\scriptsize{$\pm$0.001} & 0.045\scriptsize{$\pm$0.001} & 0.043\scriptsize{$\pm$0.000} & 0.029\scriptsize{$\pm$0.000} \\
MNIST & 500 & 0.026\scriptsize{$\pm$0.000} & 0.009\scriptsize{$\pm$0.000} & 0.030\scriptsize{$\pm$0.000} & 0.038\scriptsize{$\pm$0.000} & 0.026\scriptsize{$\pm$0.000} & 0.010\scriptsize{$\pm$0.000} \\
CIFAR-10 &  500 &  0.032\scriptsize{$\pm$0.000} & 0.001\scriptsize{$\pm$0.000} & 0.000\scriptsize{$\pm$0.000} & 0.003\scriptsize{$\pm$0.000} & 0.001\scriptsize{$\pm$0.000} & 0.000\scriptsize{$\pm$0.000} \\
\hline
\end{tabular}
\end{table}

\paragraph{Type $\bigromanone$ error of MMD-RoD.}
We report the Type $\bigromanone$ error of MMD-RoD in Table~\ref{tab:mmd-rod-type1}. 
The training pairs and test pairs are composed of samples drawn from the same distribution $\bP$. The training settings and testing procedure of MMD-RoD exactly follow Section~\ref{sec:mmd-rod-tp}.
Table~\ref{tab:mmd-rod-type1} shows that the Type $\bigromanone$ error of MMD-RoD maintains reasonable in benign settings.

\begin{table*}[h!]
\centering
\caption{The Type $\bigromanone$ error of MMD-RoD.}
\label{tab:mmd-rod-type1}
\begin{tabular}{ccccc}
\hline
Blob & HDGM & Higgs & MNIST & CIFAR-10 \\ \hline
0.049\scriptsize{$\pm$0.004} & 0.056\scriptsize{$\pm$0.000} & 0.030\scriptsize{$\pm$0.001} & 0.002\scriptsize{$\pm$0.000} & 0.000\scriptsize{$\pm$0.000} \\ 
\hline
\end{tabular}
\end{table*}

\subsection{Transferability between Different Types of Non-Parametric TSTs}
\label{appendix:extensive_exp}

We report the test power of non-parametric TSTs under the adversarial attack against a certain type of TSTs on MNIST in Figure~\ref{fig:MNIST_one_for_all}. 
The experimental settings are kept the same as in Section~\ref{sec:attack_power} except $\bW$. We set $w^{(\cJ)}=1$ for the attack implemented on benign test pairs in each row of Figure~\ref{fig:MNIST_one_for_all} where $\cJ$ is the target non-parametric TST (corresponding to the ordinate). Figure~\ref{fig:MNIST_one_for_all} shows that a specific attack against a certain type of TST sometimes can fool other types of TSTs. 

Therefore, an ensemble of TSTs is sometimes effective against a specific attack against a certain type of TST.
For example, an ensemble of C2ST-S and C2ST-L could still be vulnerable against the attack against C2ST-S since the test power of C2ST-S and C2ST-L are simultaneously degraded under the attack against C2ST-S (see the third row of Figure~\ref{fig:MNIST_one_for_all}). However, an ensemble of those six typical non-parametric TSTs can defend the attack against C2ST-S since MMD-D, MMD-G and ME all have a high test power under the attack against C2ST-S (see the third row of Figure~\ref{fig:MNIST_one_for_all}).

However, an ensemble of TSTs is no longer an effective defense under EA. Compared to the attack against a particular type of TST, our proposed EA that jointly minimizes a weighted sum of different test criteria can significantly degrades the test power of different TSTs simultaneously (empirically validated in Section~\ref{sec:attack_power}).

In addition, we further show the test power of non-parametric TSTs under EA against a TST ensemble composed by leaving one TST (corresponding to the ordinate) out of Ensemble on MNIST in Figure~\ref{fig:MNIST_leave_one_out}. The experimental settings follow Section~\ref{sec:attack_power} except $\bW$. In each row of Figure~\ref{fig:MNIST_leave_one_out}, we set the weight of the TST (corresponding to the ordinate) that is needed to be left out to 0; we then normalize the weights of leftover TSTs in Ensemble to $[0,1]$ according to the original weight set summarized in Table~\ref{tab:attack_weight}, so that the weight sum is 1. Figure~\ref{fig:MNIST_leave_one_out} demonstrates that attacks against an ensemble of TSTs sometimes can successfully fool TSTs that are not included in the attack ensemble.

All in all, Figure~\ref{fig:MNIST_transferability} validates that our proposed EA has transferability between different types of non-parametric TSTs, and it further validates that existing non-parametric TSTs lack adversarial robustness. 

\begin{figure*}[h!]
\centering
\subfigure[]{
\begin{minipage}[t]{0.49\linewidth}
\centering
\includegraphics[width=\linewidth]{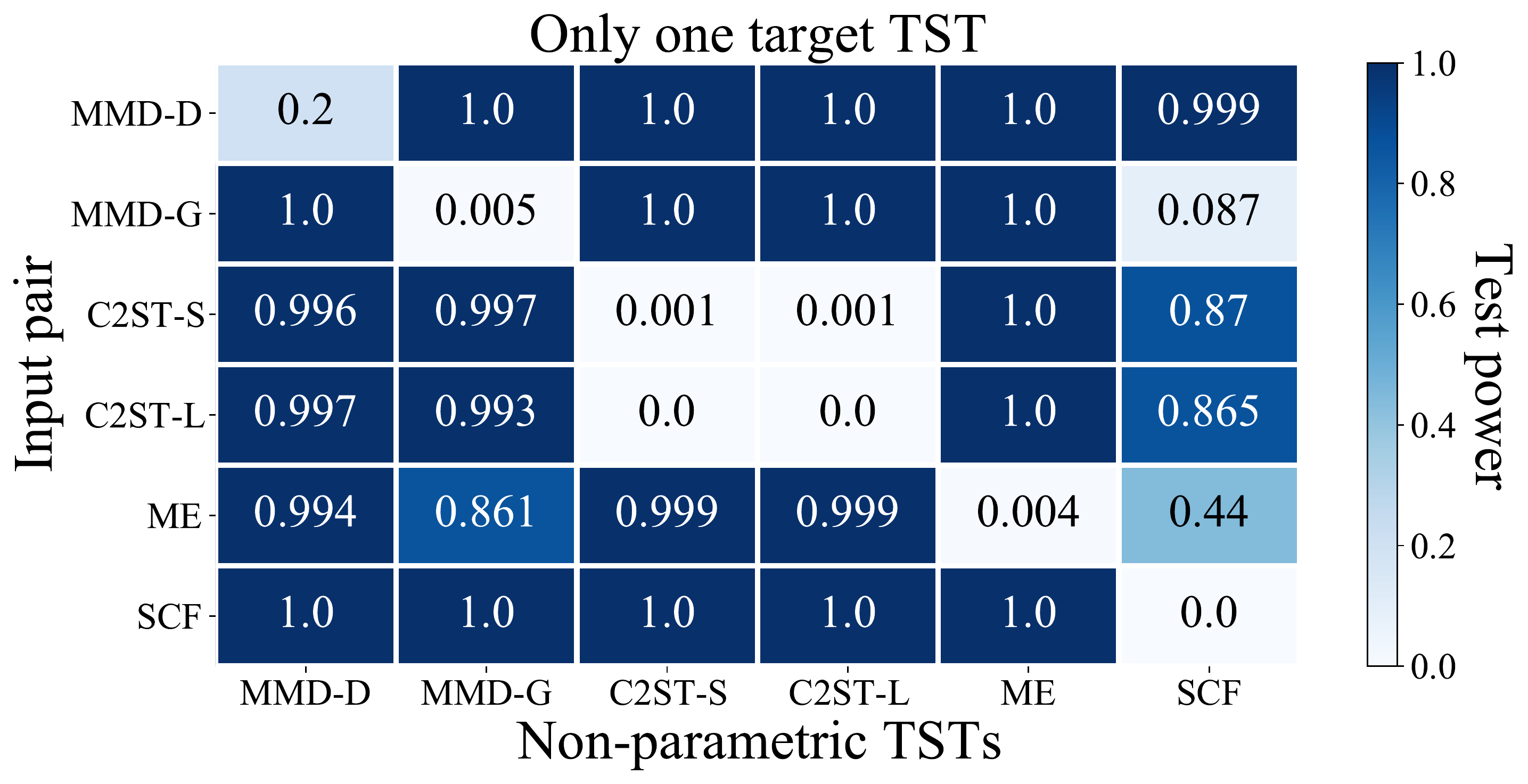}
\vskip -1.5in
\label{fig:MNIST_one_for_all}
\end{minipage}%
}%
\subfigure[]{
\begin{minipage}[t]{0.49\linewidth}
\centering
\includegraphics[width=\linewidth]{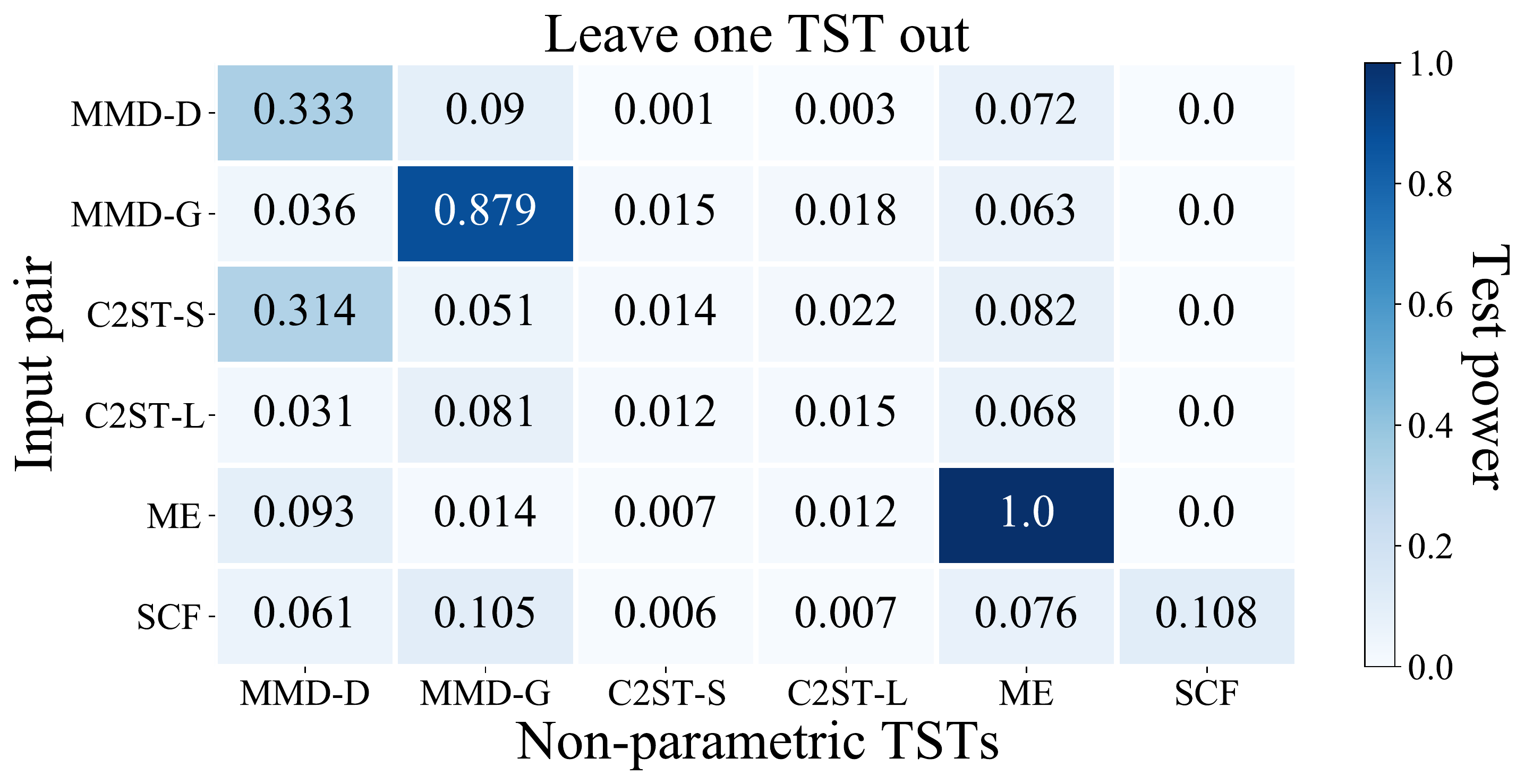}
\vskip -1.5in
\label{fig:MNIST_leave_one_out}
\end{minipage}%
}%
\centering
\vspace{-3mm}
\caption{In Figure~\ref{fig:MNIST_one_for_all}, each number represents the test power of the non-parametric TST corresponding to its abscissa on adversarial test pairs generated by the attack against the target TST corresponding to its ordinate. 
In Figure~\ref{fig:MNIST_leave_one_out}, each number represents the test power of the non-parametric TST corresponding to the abscissa on adversarial test pairs generated by the attack against Ensemble except the TST corresponding to its ordinate.}
\label{fig:MNIST_transferability}
\end{figure*}

\subsection{Discussions about the Situation When $d$ is Larger}
\label{appendix:d_reason}
In this section, we discuss the reason for the phenomenon where the test power of Ensemble under EA does not continue to decrease with larger $d$ (e.g., $d > 15$), which is shown in the upper right of Figure~\ref{fig:ablation}. We demonstrate the test power of each particular non-parametric TST and Ensemble under EA with different $d$ in Table~\ref{tab:ablation_d}. Table~\ref{tab:ablation_d} shows that, with the increasing of $d$, the test power of most TSTs (e.g., MMD-D, MMD-G) becomes lower. However, the ME test seems to be difficult to be successfully fooled with larger $d$, especially $d > 15$. We believe that upweighting the test criterion of ME (i.e., enlarging $w^{(ME)}$) during conducting EA on HDGM with larger $d$ could make EA further hurt the test power of ME and Ensemble.

\begin{table}[h!]
\centering
\caption{We report the average test power of six typical non-parametric TSTs ($\alpha=0.05$) as well as Ensemble under EA on HDGM with different $d \in \{5,10,15,20,25\}$.}
\label{tab:ablation_d}
\begin{tabular}{c|ccccccc }
\hline
$d$ & MMD-D & MMD-G & C2ST-S & C2ST-L & ME & SCF & Ensemble   \\ \hline
5 & 0.289\scriptsize{$\pm$0.019} & 0.613\scriptsize{$\pm$0.029} & 0.123\scriptsize{$\pm$0.017} & 0.597\scriptsize{$\pm$0.137} & 0.885\scriptsize{$\pm$0.080} & 0.297\scriptsize{$\pm$0.003} & 0.983\scriptsize{$\pm$0.023} \\
10 & 0.259\scriptsize{$\pm$0.009} & 0.081\scriptsize{$\pm$0.003} & 0.105\scriptsize{$\pm$0.000} & 0.090\scriptsize{$\pm$0.000} & 0.500\scriptsize{$\pm$0.025} & 0.006\scriptsize{$\pm$0.000} & 0.734\scriptsize{$\pm$0.078} \\
15 & 0.094\scriptsize{$\pm$0.002} & 0.063\scriptsize{$\pm$0.000} & 0.079\scriptsize{$\pm$0.000} & 0.086\scriptsize{$\pm$0.000} & 0.655\scriptsize{$\pm$0.000} & 0.003\scriptsize{$\pm$0.000} & 0.665\scriptsize{$\pm$0.093} \\
20 & 0.008\scriptsize{$\pm$0.000} & 0.014\scriptsize{$\pm$0.000} & 0.067\scriptsize{$\pm$0.000} & 0.051\scriptsize{$\pm$0.000} & 0.696\scriptsize{$\pm$0.000} & 0.006\scriptsize{$\pm$0.000} & 0.765\scriptsize{$\pm$0.051} \\
25 & 0.000\scriptsize{$\pm$0.000} & 0.000\scriptsize{$\pm$0.000} & 0.009\scriptsize{$\pm$0.000} & 0.000\scriptsize{$\pm$0.000} & 0.762\scriptsize{$\pm$0.000} & 0.000\scriptsize{$\pm$0.000} & 0.707\scriptsize{$\pm$0.081} \\ \hline
\end{tabular}
\end{table}

\subsection{Extensive Experiments about Adversarially Learning Kernels for TSTs}
Here, we study a different adversarial learning objective for obtaining robust kernels that minimizes a weighted sum of benign and adversarial loss~\cite{Goodfellow14_Adversarial_examples, Zhang_trades}, which is formulated as follows.
\begin{align}
\label{eq:rob_obj_addbenign}
    \hat{\theta} \approx \argmax_{\theta} ( \beta \cdot \hat{\cF}(S_{\bP}, S_{\bQ}; k_{\theta}) +  (1 - \beta) \cdot \hat{\cF}(S_{\bP}, \tilde{S}_{\bQ}; k_{\theta})),
\end{align}
where the adversarial set $\tilde{S}_{\bQ}$ is generated using Eq.~\eqref{eq:attack} and $0 \le \beta \le 1$ is a constant. Note that Eq.~\eqref{eq:rob_obj} is a special case of Eq.~\eqref{eq:rob_obj_addbenign} when we set $\beta = 0$.

We call non-parametric TSTs with robust deep kernels obtained by Eq.~\eqref{eq:rob_obj_addbenign} as ``MMD-RoD$^{*}$''. The training algorithm of MMD-RoD$^{*}$ is almost same as Algorithm~\ref{alg:rob_kernel} expect that the Line 6 in Algorithm~\ref{alg:rob_kernel} is replaced with $\theta \gets \theta + \eta \nabla_{\theta}( \beta \cdot \hat{\cF}(S_{\bP}, S_{\bQ}; k_{\theta}) +  (1 - \beta) \cdot \hat{\cF}(S_{\bP}, \tilde{S}_{\bQ}; k_{\theta}))$.

We conduct experiments to evaluate the adversarial robustness of MMD-RoD$^{*}$. We set $\beta = 0.5$ and denote the ensemble of six typical non-parametric TSTs and  MMD-RoD$^{*}$ as ``Ensemble$^{*}$''. Other settings of training, attack and testing procedure exactly follow Section~\ref{sec:mmd-rod-tp}. We report the test power of MMD-RoD$^{*}$ and Ensemble$^{*}$ in Table~\ref{tab:mmd-rod-addbenign}.

\begin{table}[h!]
\centering
\caption{Test power of MMD-RoD$^{*}$ and Ensemble$^{*}$.}
\label{tab:mmd-rod-addbenign}
\begin{tabular}{c|c|ccccc}
\hline
 & EA & Blob & HDGM & Higgs & MNIST & CIFAR-10 \\ \hline
\multirow{2}{*}{MMD-RoD$^{*}$} & $\times$ & 1.00\scriptsize{$\pm$0.04} & 1.00\scriptsize{$\pm$0.02} & 0.52\scriptsize{$\pm$0.00} & \textbf{1.00}\scriptsize{$\pm$0.12} & \textbf{1.00}\scriptsize{$\pm$0.00} \\ 
 & $\surd$ & 0.13\scriptsize{$\pm$0.06} & 0.01\scriptsize{$\pm$0.00} & 0.19\scriptsize{$\pm$0.02} & \textbf{0.86}\scriptsize{$\pm$0.00} & \textbf{0.84}\scriptsize{$\pm$0.01} \\ \hline
\multirow{2}{*}{Ensemble$^{*}$} & $\times$ & 1.00\scriptsize{$\pm$0.00} & 1.00\scriptsize{$\pm$0.00} & 1.00\scriptsize{$\pm$0.00} & 1.00\scriptsize{$\pm$0.00} & 1.00\scriptsize{$\pm$0.00} \\ 
& $\surd$ & 0.85\scriptsize{$\pm$0.01} & 0.74\scriptsize{$\pm$0.02} & 0.54\scriptsize{$\pm$0.04} & \textbf{0.89}\scriptsize{$\pm$0.00} & \textbf{0.88}\scriptsize{$\pm$0.00} \\ 
\hline
\end{tabular}
\end{table}

Compared to MMD-RoD (in Table~\ref{tab:mmd-rod}), we find MMD-RoD$^{*}$ that incorporates benign training pairs into adversarially learning kernels improves the test power in benign settings (especially on HDGM), but obtains the lower test power in adversarial settings among all datasets. Therefore, we recommend utilizing only adversarial training pairs for adversarially learning deep kernels.

\section{Description of Attackers against Non-Parametric TSTs}
\label{appendix:attacker}
In this section, we provide a detailed description of the attacker against non-parametric TSTs from four perspectives ``goal, knowledge, capability, strategy''~\cite{biggio2018wild}.

\begin{itemize}
\item \textbf{Goal.} The attacker aims to make a target non-parametric TST incorrectly judge two sets of data are drawn from the same distribution during the test procedure, when in reality these two sets of data are drawn from different distributions. 

\item \textbf{Knowledge.} Depending on the assumptions made on the attacker's knowledge, we have different attack scenarios. 
\begin{itemize}
    \item Perfect-knowledge white-box attacks. The attacker is assumed to know everything about the target non-parametric TST, such as the target non-parametric TST's test criterion function and kernel parameters. 
    \item Limited-knowledge gray-box attacks. The attacker has part of the target non-parametric TST's knowledge. For example, the attacker knows the target non-parametric TST's test criterion function, but does not know its kernel parameters and training data. 
    \item Zero-knowledge black-box attacks. The attacker does not have any knowledge about the target non-parametric TST. The attacker can only query the non-parametric TST in a black-box manner and then obtain the judgement on the test pairs. 
\end{itemize}

\item \textbf{Capability.} The attacker can only manipulate test data, and the malicious perturbations should be human-imperceptible. 

\item \textbf{Strategy.} The attacker searches for adversarial sets via minimizing the target non-parametric TST's test criterion under data manipulation constraints. 

\end{itemize}


\end{document}